\documentclass{article}
 \usepackage{log_2022}

\usepackage{booktabs}
\usepackage{multirow}
\usepackage{amsfonts}
\usepackage{graphicx}
\usepackage{duckuments}
\usepackage{multicol}
\usepackage{subcaption}
\usepackage{enumitem}
\usepackage{calc}
\usepackage{bbm}
\usepackage{algorithm}
\usepackage{amsmath, amssymb}
\usepackage{algorithm}
\usepackage{algpseudocode}
\usepackage[nonumberlist]{glossaries}
\usepackage{tabularx}
\usepackage{graphicx,subcaption}
\usepackage[export]{adjustbox}
\usepackage{pdflscape}
\usepackage{yhmath}
\usepackage{wrapfig}
\usepackage{geometry}
\usepackage{titlesec}
\usepackage{xcolor}
\usepackage{pifont}
\usepackage{tikz-cd}
\usepackage{colortbl}
\usepackage{amsmath}

\usepackage{amsthm}
\usepackage{thmtools}
\usepackage{etoolbox}
\usepackage{fancyhdr}
\usepackage{mathtools}
\mathtoolsset{showonlyrefs}

\newbool{inappendix}
\boolfalse{inappendix}
\pretocmd{\appendix}{\booltrue{inappendix}}{}{}

\newcommand{\showsimplified}[1]{\ifbool{inappendix}{}{#1}}
\newcommand{\showfull}[1]{\ifbool{inappendix}{#1}{}}

\definecolor{OurColor}{RGB}{232, 245, 233}
\definecolor{HomophilyColor}{RGB}{224, 247, 248} 
\definecolor{BottleneckColor}{RGB}{242, 242, 242}

\definecolor{sectioncolor}{RGB}{0,0,90}

\definecolor{textblue}{HTML}{1E90FF}
\definecolor{textgreen}{HTML}{2E8B57}
\definecolor{textyellow}{HTML}{B8860B}
\definecolor{textorange}{HTML}{FF8C00}
\definecolor{textred}{HTML}{DC143C}

\definecolor{HomophilyColor}{gray}{0.92}
\definecolor{BottleneckColor}{gray}{0.96}

\makeglossaries

\newcommand{\greencheck}{\textcolor{green!70!black}{\ding{51}}}
\newcommand{\redcross}{\textcolor{red}{\ding{55}}}

\newcommand{\sigsenexp}[3][]{S^{(#2)}_{#1,#3}}
\newcommand{\noisesenexp}[3][]{N^{(#2)}_{#1,#3}}
\newcommand{\globsenexp}[3][]{T^{(#2)}_{#1,#3}}

\usepackage{xcolor}
\definecolor{myblue}{HTML}{4169E1}
\definecolor{myred}{HTML}{B22222}
\definecolor{mygreen}{HTML}{228B22}
\definecolor{mypurple}{HTML}{8A2BE2}

\usepackage[numbers]{natbib}

\usepackage{ifthen}

\makeatletter
\newcommand{\citen}[2][]{
  [\citenum{#2}
    \ifthenelse{\equal{#1}{}}{}{, #1}
  ]
}
\makeatother
\usepackage{commands}

\title[Class-bottlenecks restrict the signal-to-noise ratio in message passing]{Limits of message passing for node classification:\\ How class-bottlenecks restrict signal-to-noise ratio}

\author[J. Rubin et al.]{
Jonathan Rubin\\
\institute{Department of Mathematics\\\small Department of Computing \\\small UKRI Centre for Doctoral Training in AI for Healthcare\\\small Imperial College London}\\
\email{jonathan.rubin19@imperial.ac.uk}\And
Sahil Loomba\\
\institute{Department of Mathematics\\\small Imperial College London}\\
\institute{MIT Institute for Data, Systems, and Society}\\
\email{sloomba@mit.edu}\And
Nick S. Jones\\
\institute{Department of Mathematics\\\small I-X Centre for AI in Science\\\small EPSRC Centre for the Mathematics of Precision Healthcare\\\small Imperial College London}\\
\email{nick.jones@imperial.ac.uk}
}

\begin{document}

\maketitle
\begin{abstract}Message passing neural networks (MPNNs) are powerful models for node classification but suffer from performance limitations under heterophily (low same-class connectivity) and structural bottlenecks in the graph. We provide a unifying statistical framework exposing the relationship between heterophily and bottlenecks through the signal-to-noise ratio (SNR) of MPNN representations. The SNR decomposes model performance into feature-dependent parameters and feature-independent sensitivities. We prove that the sensitivity to class-wise signals is bounded by higher-order homophily---a generalisation of classical homophily to multi-hop neighbourhoods---and show that low higher-order homophily manifests locally as the interaction between structural bottlenecks and class labels (class-bottlenecks). Through analysis of graph ensembles, we provide a further quantitative decomposition of bottlenecking into underreaching (lack of depth implying signals cannot arrive) and oversquashing (lack of breadth implying signals arriving on fewer paths) with closed-form expressions. We prove that optimal graph structures for maximising higher-order homophily are disjoint unions of single-class and two-class-bipartite clusters. This yields BRIDGE, a graph ensemble-based rewiring algorithm that achieves near-perfect classification accuracy across all homophily regimes on synthetic benchmarks and significant improvements on real-world benchmarks, by eliminating the ``mid-homophily pitfall'' where MPNNs typically struggle, surpassing current standard rewiring techniques from the literature. Our framework, whose code we make available for public use, provides both diagnostic tools for assessing MPNN performance, and simple yet effective methods for enhancing performance through principled graph modification.

\end{abstract}

\section*{Main}

Geometric deep learning has emerged as a powerful framework for learning representations of structured data \cite{kipf2017semi, Vaswani2017Attention, wu2020comprehensive, bronstein2021geometric}, leveraging dependencies between entities to capture complex patterns \cite{battaglia2018relational, waikhom2021graph}. These dependencies often come in the form of graphs, where entities are represented by nodes and relations by edges. Message passing neural networks (MPNNs) are prominent models in this framework that operate by iteratively updating each node's representation based on its neighbours' features, propagating information across the graph to build expressive node representations \cite{gilmer2017neural, hamilton2017inductive}.

\begin{figure}[!htbp]
  \centering
  
  \begin{subfigure}[t]{\linewidth}
    \centering
    \makebox[\textwidth][c]{\includegraphics[width=1.3\linewidth]{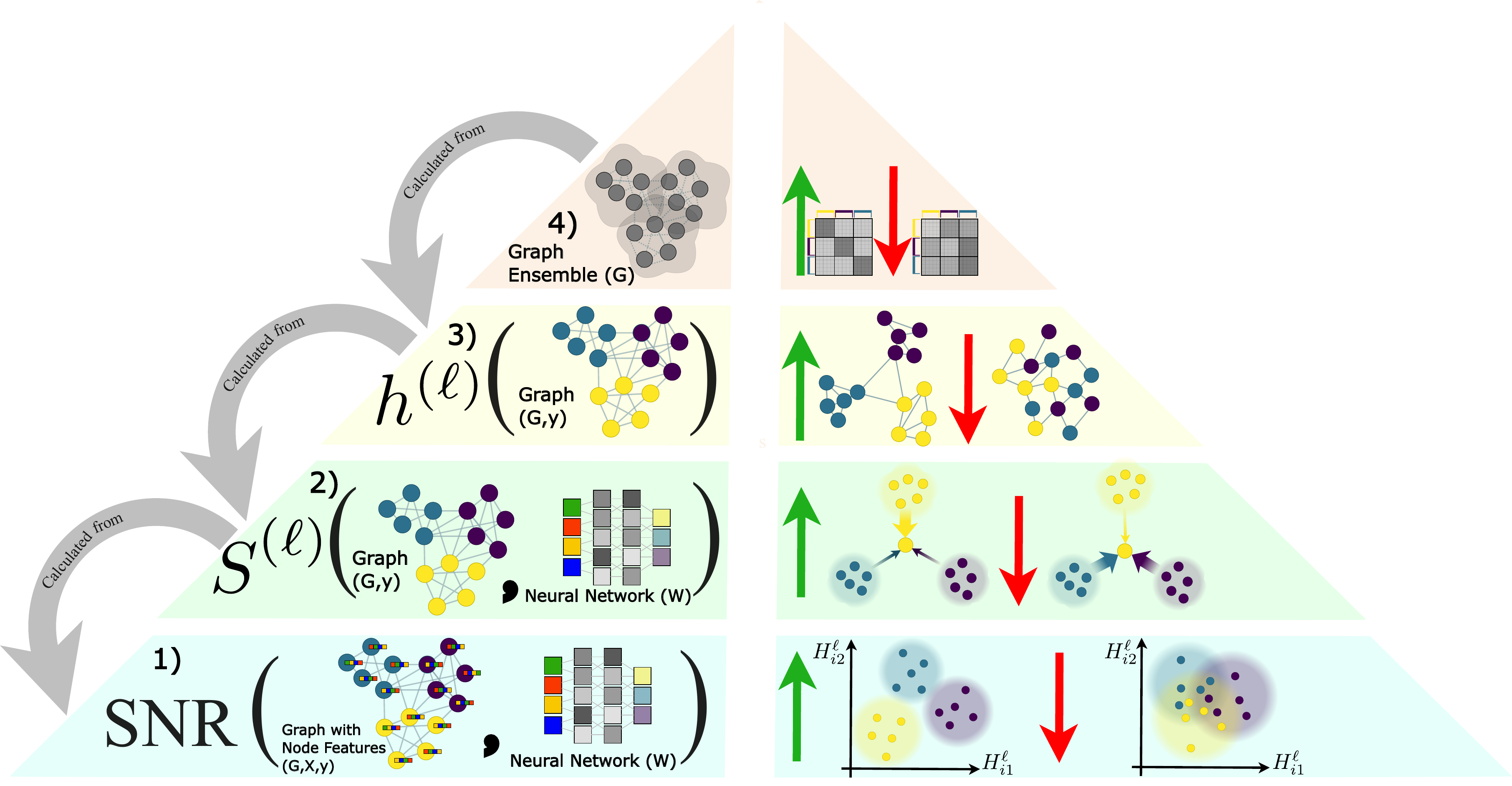}}
    \caption{Hierarchical decomposition of MPNN performance in node classification tasks.}
    \label{fig:framework_outline}
  \end{subfigure}
  
  \vspace{0.5em} 
  
  \centering
  \begin{subfigure}[t]{0.49\linewidth}
    \centering
    \includegraphics[width=\linewidth]{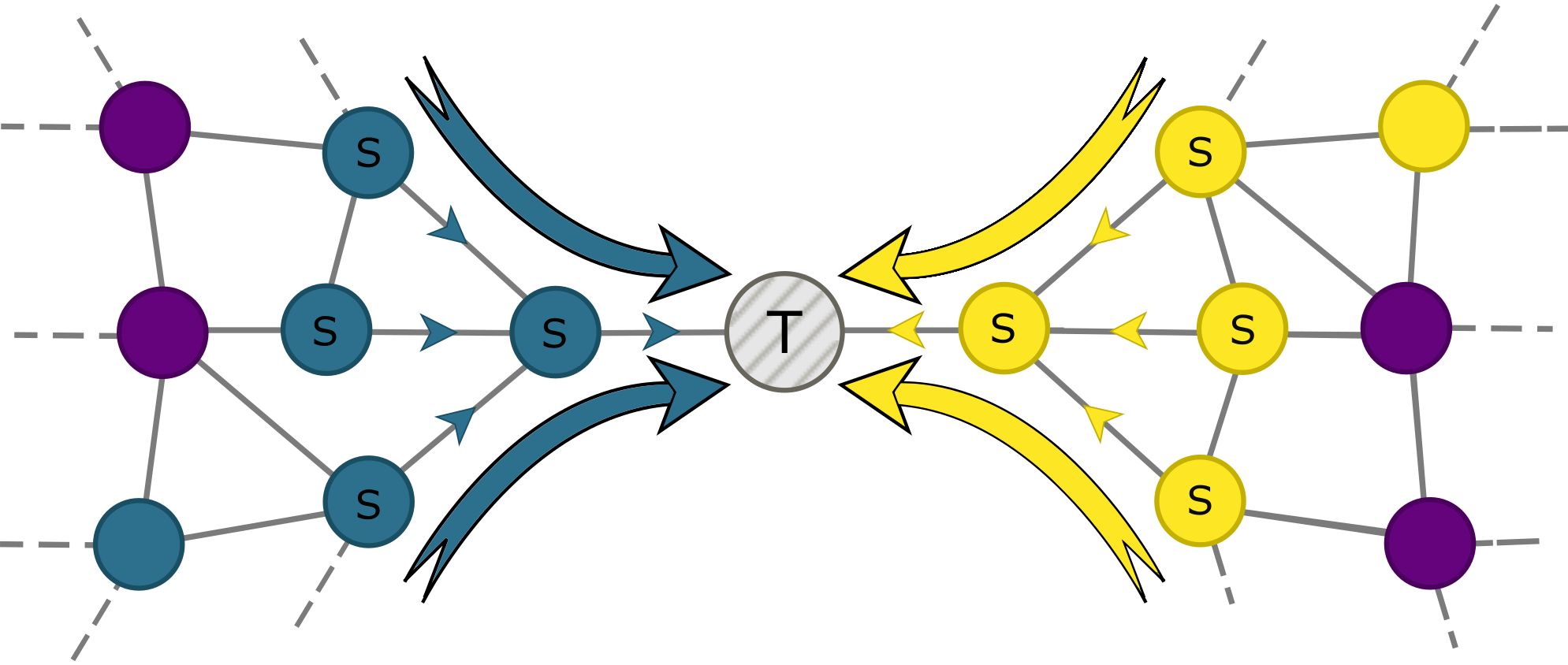}
    \caption{A ``bad'' class-bottleneck.}
    \label{fig:homophilic_bottleneck_bad}
  \end{subfigure}
  \hfill
  \begin{subfigure}[t]{0.49\linewidth}
    \centering
    \includegraphics[width=\linewidth]{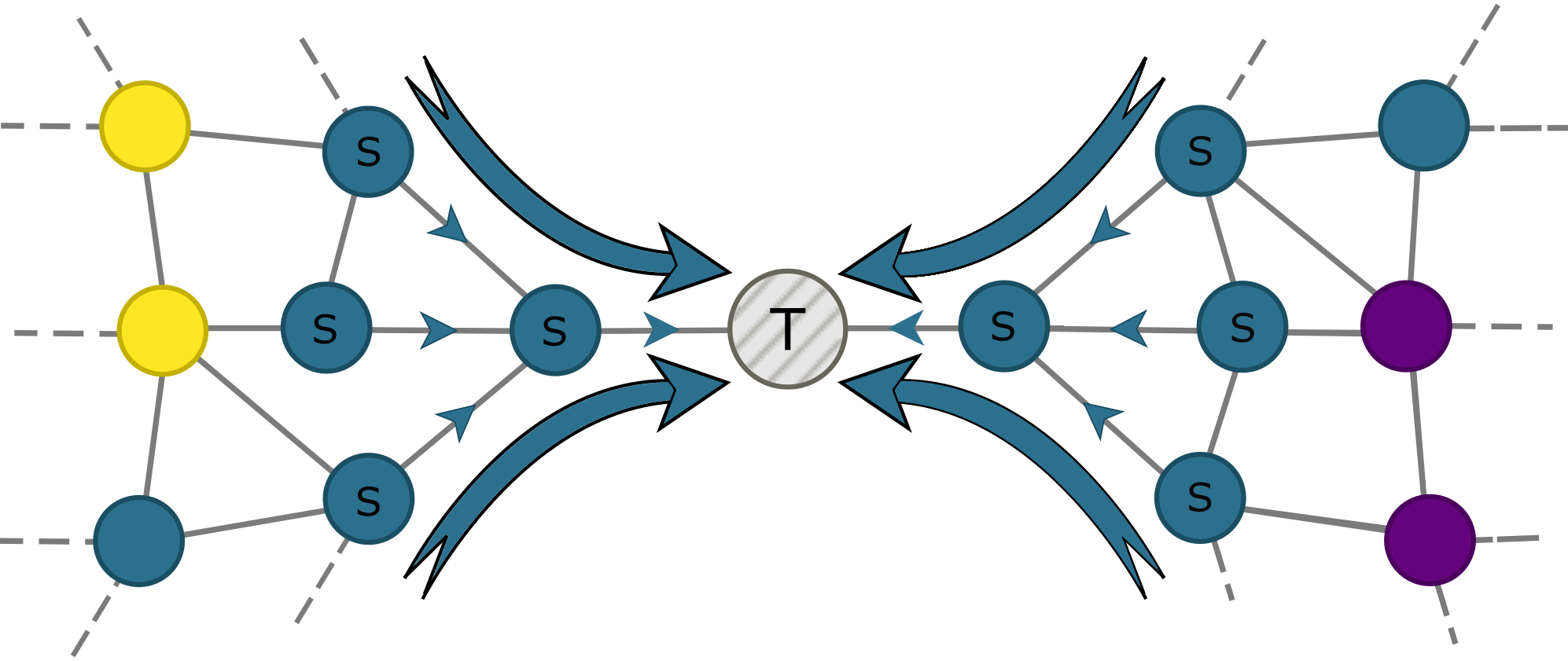}
    \caption{A ``mild'' class-bottleneck.}
    \label{fig:homophilic_bottleneck_good}
  \end{subfigure}
  
  \vspace{0.5em}

  \caption{\textbf{(a) Analysis of MPNN performance in node classification can be hierarchically decomposed; Eq.~\eqref{eq:framework-hierarchy}.}~We incrementally decouple the different factors that contribute to MPNN performance on a node classification task---the graph structure $G$, the node labels $y$, the model weights $W$, and the input features $X$. The signal-to-noise ratio (1; SNR) depends on the signal sensitivity (2; $S^{(\ell)}$), which is bounded by higher-order homophily (3; $h^{(\ell)}$; (Eq.~\eqref{eq:avg_sensitivity_bound}) that can be approximated using the expected adjacency matrix (4; $\mathbb{E}[\mathbf{A}]$) of the graph ensemble through oversquashing/underreaching analysis (Eq.~\eqref{eq:underover}, Theorem~\ref{theorem:underreaching_oversquashing}). \textbf{(b), (c) Not all structural bottlenecks are equal: the interaction between class labels and structural bottlenecks (class-bottlenecks) determines node classification performance of MPNNs.}~Both graphs in (b) and (c) depict a structural bottleneck. However, in (b) a ``bad'' bottleneck where messages from source nodes (S) of different classes interfere at the target node (T), limiting the local class-bottlenecking score $\localhomo{\ell,\ell}{\hat{\mat{A}}}{T}$ (Eq.~\eqref{eq:p_order_local_homophily}) and thus restricting signal sensitivity (Eq.~\eqref{eq:individual_signal_sensitivity_sgc}, Eq.~\eqref{eq:individual_signal_sensitivity_mpnn}). In (c), a more ``mild'' bottleneck still throttles signals coming from the source nodes, but the same-class source nodes positively reinforce the signal.}\label{fig:framework_homophilic}
\end{figure}

However, the performance of MPNNs can be substantially hindered in certain graph structures, especially for the task of node classification. Heterophilic graphs, which contain a high proportion of edges connecting nodes of different classes, pose a challenge as they limit the aggregation of class-specific information. Homophily, the tendency of nodes within the same class to preferentially connect to one another, thus plays a crucial role in determining MPNN performance \cite{zhu2020homophily, luan2022revisiting, zheng2022graph, ma2021homophily,gong2024survey,luan2024heterophilicgraphlearninghandbook}.

Additionally, bottleneck structures in the graph have been shown to impede the flow of information as a result of underreaching---where information from distant nodes fails to propagate through the network---and oversquashing---where information from multiple source nodes is compressed into a fixed-size vector---leading to loss of important signals \cite{alon2021bottleneck, topping2021understanding, black2023understanding}. Our work provides a unified framework to analyse these phenomena and assess their impact on node classification performance.

Prior work has investigated these behaviours in isolation, focusing on specific failure modes and proposing tailored architectures to mitigate them. For example, Di  Giovanni et al. \cite{digiovanni2023does} analyse how poor MPNN sensitivity as a result of bottlenecks, measured through the Jacobian of the MPNN, restrict their expressive power, whilst Novak et al. \cite{novak2018sensitivity} show that neural network sensitivity, measured using the mean Jacobian norm, reduces generalisation power. On the other hand, Zhu et al. and Luan et al. \cite{zhu2020homophily,luan2024graphneuralnetworkshelp} investigate the impact of homophily on intra-class and inter-class node distinguishability and empirically study when graph-aware models outperform graph-agnostic models.

These varying viewpoints present different and sometimes conflicting implications for MPNN sensitivity and homophily. For instance, following Di  Giovanni et al. \cite{digiovanni2023does}, graphs with strong bottlenecks lead to less expressive MPNN models, however these same structures would result in MPNNs with lower sensitivity and thus Novak et al. \cite{novak2018sensitivity} suggest they would exhibit better generalisation. This contrast highlights the need to distinguish between different \emph{types} of sensitivity in MPNNs. Additionally, a graph with a strong community structure will be highly bottlenecked at the intersection of the communities; yet, if those communities align with node classes the graph would be highly homophilic and, by Luan et al. \cite{luan2024graphneuralnetworkshelp}, the MPNN would distinguish node classes more effectively. 
These examples show the need for a unified understanding of how graph structure affects MPNN performance in node classification, since a holistic view is crucial for designing MPNNs that can robustly learn distinct representations for node classification. In this paper, we answer the following ultimate question:

\begin{center}
\textit{What is the precise relationship between homophily and bottlenecks, and how does this relationship dictate the fundamental performance limits of MPNNs?}
\vspace{10pt}
\end{center}
Specifically, our work makes the following contributions to understanding and improving MPNNs:
\begin{enumerate}[leftmargin=*]
\item \textbf{Signal-to-noise ratio of message passing.} 
   We introduce a signal-to-noise ratio (SNR) that quantifies MPNN performance through two orthogonal components: feature-independent model sensitivity measures---$S^{(\ell)}(\cdot)$ in Figure \ref{fig:framework_outline}---that capture how MPNNs respond to input changes, and model-independent statistics that characterise input feature quality.
\item \textbf{Higher‑order homophily bounds sensitivity.} 
   We show that the average signal sensitivity is provably restricted by higher‑order homophily $h^{(\ell)}(\cdot)$; low homophily manifests locally as \emph{class-bottlenecks}, depicted in Figure \ref{fig:homophilic_bottleneck_bad}, that throttle class‑specific information regardless of architecture. 
\item \textbf{Bottlenecks decompose into underreaching and oversquashing.} 
   Assuming a graph ensemble, we quantitatively decompose higher‑order homophily into \emph{underreaching} (lack of depth for distant signals) and \emph{oversquashing} (lack of breadth for signals arriving on too few paths)---whose joint contributions to bottlenecking are only heuristically explained in the literature---and provide closed‑form expressions for both effects. 
\item \textbf{Optimal structure and principled rewiring.}
We prove that the graph structures that maximise higher-order homophily are disjoint unions of single-class and two-class-bipartite clusters.  This theoretical result yields BRIDGE, Block Resampling from Inference-Derived Graph Ensembles, an iterative edge-resampling algorithm that reshapes the graph structure toward this optimum.
\end{enumerate}

Our framework builds a simple hierarchical view of how different factors affect MPNN performance in node classification, visualised in Figure \ref{fig:framework_outline}. We incrementally decouple the different factors by focusing on a central quantity at each level as given by Eq.~\eqref{eq:framework-hierarchy}. The ultimate measure, the signal-to-noise ratio (SNR), depends on the complete setup: the graph structure $G$, the node labels $y$, the model weights $W$, and the input features $X$. The  SNR is shown to be a direct function of the model's ``signal sensitivity'' $S^{(\ell)}$, which captures how the model $W$ processes label-relevant signals $y$ on the given graph $G$. This sensitivity, in turn, is bounded by the graph's higher-order homophily $h^{(\ell)}$, a structural property capturing multi-hop class-wise connectivity that depends only on the graph structure $G$ and the true class labels $y$. Finally, this higher-order homophily can be approximated by analysing the properties of the underlying graph ensemble, represented by the expected adjacency matrix $\mathbb{E}[\mathbf{A}]$, inferred from the given instance of the graph $G$.

\begin{equation}
\label{eq:framework-hierarchy}
\begin{tikzcd}[
  column sep=1.8cm,
  row sep=0cm,
  every arrow/.append style={thick},
  nodes={font=\sffamily}
]
 \mathbb{E}[\mathbf{A}]
  \arrow[r, -{Stealth[length=3mm]},
         "{\small\text{approximates}}",
         "{\scriptsize \mathrm{Eq.}~\eqref{eq:expected_l_homophily_approx}}"'] &
 {h}^{(\ell)}
  \arrow[r, -{Stealth[length=3mm]},
         "{\small\text{bounds}}",
         "{\scriptsize \mathrm{Eq.}~\eqref{eq:avg_sensitivity_bound}}"'] &
 {S}^{(\ell)}
  \arrow[r, -{Stealth[length=3mm]},
         "{\small\text{controls}}",
         "{\scriptsize \mathrm{Eq.}~\eqref{eq:snrsensitivity}}"'] &
 \mathrm{SNR}
\\
 |[font=\rm]|{\textit{(G)}} &
 |[font=\rm]|{\textit{(G, y)}} &
 |[font=\rm]|{\textit{(G, y, W)}} &
 |[font=\rm]|{\textit{(G, y, W, X)}}
\end{tikzcd}
\end{equation}

Through extensive experiments on standard benchmark synthetic graphs and real-world graph datasets, we validate our theoretical analysis and demonstrate the practical utility of our framework. Overall, our work offers a deeper understanding of the mechanisms driving MPNN performance and provides guiding principles for model design. Our results pave the way for a more statistically grounded analysis of MPNNs, unlocking their potential for a wider range of applications. Code for all SNR calculations as well as the BRIDGE algorithm is available at: \href{https://github.com/jr419/BRIDGE}{\texttt{https://github.com/jr419/BRIDGE}}, where we provide additional documentation on how to use it.

\begin{table}[htbp]
  \centering
  \makebox[\textwidth]{
  \begin{tabular}{l*{4}{c}}
    \toprule
    \multirow{2}{*}{\textbf{Paper}}
      & \multicolumn{2}{c}{\textbf{Homophily}}
      & \multicolumn{2}{c}{\textbf{Bottlenecks}} \\
    \cmidrule(lr){2-3}\cmidrule(lr){4-5}
      & \textbf{First-order}
      & \textbf{Higher-order}
      & \textbf{Oversquashing}
      & \textbf{Underreaching} \\
    \midrule
    \rowcolor{OurColor}
    Our Paper & \greencheck & \greencheck & \greencheck & \greencheck \\

    \rowcolor{HomophilyColor}
    Zhu et al. \cite{zhu2020homophily} & \greencheck & \redcross & \redcross & \redcross \\
    \rowcolor{HomophilyColor}
    Luan et al. \cite{luan2022revisiting} & \greencheck & \redcross & \redcross & \redcross \\
    \rowcolor{HomophilyColor}
    Rossi et al. \cite{rossi2023edgedirectionalityimproveslearning} & \greencheck & \greencheck & \redcross & \redcross \\
    \rowcolor{HomophilyColor}
    Ma et al. \cite{ma2021homophily} & \greencheck & \redcross & \redcross & \redcross \\
    \rowcolor{HomophilyColor}
    Luan et al. \cite{luan2024graphneuralnetworkshelp} & \greencheck & \redcross & \redcross & \redcross \\

    \rowcolor{BottleneckColor}
    Alon and Yahav \cite{alon2021bottleneck} & \redcross & \redcross & \greencheck & \greencheck \\
    \rowcolor{BottleneckColor}
    Topping et al. \cite{topping2021understanding} & \redcross & \redcross & \greencheck & \redcross \\
    \rowcolor{BottleneckColor}
    Black et al. \cite{black2023understanding} & \redcross & \redcross & \greencheck & \redcross \\
    \rowcolor{BottleneckColor}
    Di  Giovanni et al. \cite{digiovanni2023does} & \redcross & \redcross & \greencheck & \redcross \\

    \bottomrule
  \end{tabular}
  }
  \caption{\textbf{Comparison of various aspects of node classification performance of MPNNs considered in the literature.}  
  Row shading differentiates the homophily and bottleneck literatures.}
  \label{tab:theory-comparison}
\end{table}

\subsection*{Problem setup}
We consider semi-supervised node classification on an attributed graph $G = (V, E)$ with node set $V:=[n]$ consisting of $n$ nodes and possibly directed edge set $E:=\{(i,j)\in V^2\,:\,i\text{ and }j\text{ are directly connected}\}$, encoded by the adjacency matrix $\mat{A} \in \{0,1\}^{n\times n}$, with feature matrix $\mat{X} \in \mathbb{R}^{n\times d_{\mathrm{in}}}$. Each node $i$ belongs to a class $y_i \in [k]$, and the objective is to learn discriminative node representations $\mat{H}_i \in \mathbb{R}^{d_{\mathrm{out}}}$ that enable accurate class predictions.

\paragraph{Message passing neural networks.}

MPNNs learn node representations by iteratively aggregating and transforming feature information from each node's local neighbourhood \cite{gilmer2017neural, bronstein2021geometric, battaglia2018relational}. Formally, an MPNN computes the representation of node $i$ at layer $\ell+1$ as:

\begin{align}
  \mat{H}_i^{(\ell+1)} := U_\ell\left(\mat{H}_i^{(\ell)}, \sum_{j \in V} \hat{A}_{ij} M_\ell\left(\mat{H}_i^{(\ell)}, \mat{H}_j^{(\ell)}\right)\right) \label{eq:mpnn}
\end{align}

where $\hat{\mat{A}} \in \mathbb{R}^{n \times n}$ is a graph shift operator, typically a normalised version of the adjacency matrix, $U_\ell$ and $M_\ell$ are learnable transformations. Initialised with $\mat{H}_i^{(0)} = \mat{X}_i$, stacked layers ($\ell = 0,\dots,L-1$) sequentially integrate multi-hop dependencies, with final representations $\mat{H}_i^{(L)}$ fed to a softmax classifier for class prediction.

\paragraph{Feature distribution.}
Consider a reparameterisation of node $i$’s feature vector in terms of its class-wise mean vector $\vect{\mu}$, global shift $\vect{\gamma}$ and corresponding residual or ``noise'' vector $\vect{\epsilon}$, akin to the reparameterisation used in variational autoencoders to learn latent data distributions in a differentiable manner \cite{kingma2022autoencoding}:
\begin{equation}\label{eq:featuredecomp}
  \mat{X}_j = \underbrace{\vect{\mu}_{y_j}}_{\text{class signal}} + \underbrace{\vect{\gamma}}_{\text{global shift}} + \underbrace{\vect{\epsilon}_j}_{\text{node noise}}.
\end{equation}
We make reasonable assumptions on these three terms that encompass most existing feature distributions in the literature (such as the CSBM model \cite{ma2021homophily}): $\vect{\mu}_{y_j}$ represents class-specific signals i.e. $\mathbb{E}[\vect{\mu}_{c}] = \mathbb{E}[\mat{X}_j\,|\,y_j=c]$, $\vect{\gamma}$ captures zero-mean global variations, and $\vect{\epsilon}_j$ denotes IID zero-mean node-level noise. The feature-covariance structure is characterized by signal covariance $\var{\vect{\mu}} :=\mat{\Sigma}$, global shift covariance $\var{\vect{\gamma}} :=\mat{\Phi}$, and noise covariance $\var{\vect{\epsilon}_i} :=\mat{\Psi}$. In other words, for nodes $j$ and $k$, their feature covariance satisfies:

\begin{align}\label{eq:total_covariance}
\mathrm{Cov}(\mathbf{X}_j,\mathbf{X}_k)
=
\begin{cases}
\underbrace{\begin{bmatrix}
\Phi_{11} & \cdots & \Phi_{1d} \\
\vdots & \ddots & \vdots \\
\Phi_{d1} & \cdots & \Phi_{dd}
\end{bmatrix}}_{\mathbf{\Phi}},
&
y_j \neq y_k, \\[6pt]
\underbrace{\begin{bmatrix}
\Sigma_{11} & \cdots & \Sigma_{1d} \\
\vdots & \ddots & \vdots \\
\Sigma_{d1} & \cdots & \Sigma_{dd}
\end{bmatrix}}_{\mathbf{\Sigma}}
+
\underbrace{\begin{bmatrix}
\Phi_{11} & \cdots & \Phi_{1d} \\
\vdots & \ddots & \vdots \\
\Phi_{d1} & \cdots & \Phi_{dd}
\end{bmatrix}}_{\mathbf{\Phi}},
&
y_j = y_k,\ j \neq k, \\[6pt]
\underbrace{\begin{bmatrix}
\Sigma_{11} & \cdots & \Sigma_{1d} \\
\vdots & \ddots & \vdots \\
\Sigma_{d1} & \cdots & \Sigma_{dd}
\end{bmatrix}}_{\mathbf{\Sigma}}
+
\underbrace{\begin{bmatrix}
\Phi_{11} & \cdots & \Phi_{1d} \\
\vdots & \ddots & \vdots \\
\Phi_{d1} & \cdots & \Phi_{dd}
\end{bmatrix}}_{\mathbf{\Phi}}
+
\underbrace{\begin{bmatrix}
\Psi_{11} & \cdots & \Psi_{1d} \\
\vdots & \ddots & \vdots \\
\Psi_{d1} & \cdots & \Psi_{dd}
\end{bmatrix}}_{\mathbf{\Psi}},
&
j = k.
\end{cases}
\end{align}

Notably, we are not treating $\vect{y}$ as a random variable, but as a fixed class label set, defining the distribution over possible feature sets $\mat{X}$.
We thus explicitly separate class-driven structure from global and node-specific stochasticity. The class-wise covariance structure, $\mat{\Sigma}$, controls the degree of consistency among node features within each class, making unique aspects of each class easier or harder to discern.

\paragraph{Homophily.}
In graph-based learning, homophily refers to the tendency of similar nodes (e.g., nodes with the same class label) to be preferentially connected. This property is quantified in various ways in the literature, but most commonly using two measures: \emph{edge homophily} and \emph{node homophily} \cite{zhu2020homophily}. Edge homophily is defined as the fraction of edges in the graph that connect nodes of the same class, while node homophily measures the proportion of same-class neighbours for each node, averaged over all nodes. Formally, for a graph $ G = (V, E) $, they are expressed as:
\begin{align}\label{eq:edge_node_homophily}
h_\text{edge} := \frac{\big|\{(i, j) \in E : y_i = y_j\}\big|}{|E|}, \quad h_\text{node} := \frac{1}{|V|} \sum_{i \in V} \frac{\big|\{j \in V : (i,j)\in E, y_i = y_j\}\big|}{\big|\{j \in V : (i,j)\in E\}\big|}.
\end{align}
Intuitively, high homophily aligns with better MPNN performance because the message-passing mechanism relies on aggregating information from neighbouring nodes. When nodes with the same class label are more likely to be connected, the aggregated features are more likely to contain relevant information for predicting the node's label, leading to improved representations and model accuracy. However in practice, high homophily is not always necessary---many works in the literature have presented cases where MPNNs perform well in heterophilic (low homophily) settings, and have proposed their own measures of homophily to more accurately capture MPNN performance in heterophilic graphs \cite{zhu2020homophily,ma2021homophily,luan2021heterophilyrealnightmaregraph}.

These measures primarily focus on \emph{direct} connections. For a more generalised form that can extend to multi-hop relationships, we consider weighted homophily, as introduced by Rossi et al. \cite{rossi2023edgedirectionalityimproveslearning}:
\begin{align}\label{eq:def_weighted_homophily}
h(\mat{S}) := \frac{1}{|V|} \sum_{i,j \in V} S_{ij}\delta_{y_i y_j},
\end{align} 
where $\mat{S}$ is a choice of message-passing matrix, and $\delta_{y_i y_j}$ is the Kronecker delta. This measure can be seen as a generalisation of edge and node homophily: reducing to $h_\text{edge}$ when $\mat{S} = \frac{1}{\avg{d}} \mat{A}$ (where $\avg{d}:= \frac{1}{|V|}\sum_{i,j\in V}{A_{ij}}$ is the mean degree), and $h_\text{node}$ when $\mat{S} = \mat{D}^{-1} \mat{A}$ is the random-walk normalised adjacency matrix (where $\mat{D}_{ii} := \sum_{j\in V}{A_{ij}}$ is the diagonal degree matrix). We note that in the original form of weighted homophily as defined by Rossi et al. \cite{rossi2023edgedirectionalityimproveslearning}, the authors consider a normalised definition using a normalised $\mat{S}'$ defined as $S'_{ij} := \frac{S_{ij}}{\sum{j\in V}{S_{ij}}}$ instead of $\mat{S}$. We show in this paper that with the correct choice of message passing matrices $\mat{S}$, the unnormalised definition is more natural to use.

\section*{Results}

\subsection*{Signal-to-noise ratio can be decomposed into feature covariances and feature-agnostic model sensitivities}

To analyse the behaviour of MPNNs, we introduce three key feature-agnostic metrics that capture the model's sensitivity to different aspects of the input data:

\begin{equation}\label{eq:sensitivities_intro}
\begin{aligned}
  \sigsenexp[i]{\ell}{p,q,r} &:= \left[\nabla_{\vect{\mu}} H_{ip}^{(\ell)} \left(\nabla_{\vect{\mu}} H_{ip}^{(\ell)}\right)^T\big|_{\mat{X}=\mat{0}}\right]_{qr}, \quad \noisesenexp[i]{\ell}{p,q,r}:= \left[\nabla_{\mat{\epsilon}} H_{ip}^{(\ell)} \left(\nabla_{\mat{\epsilon}} H_{ip}^{(\ell)}\right)^T\big|_{\mat{X}=\mat{0}}\right]_{qr}, \\
  & \qquad \qquad \globsenexp[i]{\ell}{p,q,r}:= \left[\nabla_{\mat{\gamma}} H_{ip}^{(\ell)} \left(\nabla_{\mat{\gamma}} H_{ip}^{(\ell)}\right)^T\big|_{\mat{X}=\mat{0}}\right]_{qr}.
\end{aligned}
\end{equation}

We term these as signal, noise and global sensitivity, respectively. These sensitivities can be viewed as induced metrics on the latent representation space, quantifying the MPNN's local sensitivity to variations in class-wise signal, node-level noise, and global shifts of the input features. Importantly, all three sensitivity measures are feature-independent, as they depend only on the model architecture and the class labels, not on the specific feature values. In other words, these measures depend on the graph structure, its partition into classes, and the node representation update function from Eq. \eqref{eq:mpnn}, but they do not depend on a specific choice of node features $X$. Our analysis is thus robust across different feature distributions, and also allows us to isolate the effects of the graph structure.

Using the feature decomposition in Eq. \eqref{eq:featuredecomp}, signal sensitivity can be calculated as:
\begin{align}
\sigsenexp[i]{\ell}{p,q,r} = \sum_{j, k \in V} \frac{\partial H_{ip}^{(\ell)}}{\partial X_{jq}} \frac{\partial H_{ip}^{(\ell)}}{\partial X_{kr}}\Bigg|_{\mat{X}=\mat{0}} \delta_{y_j y_k}\label{eq:signal_sensitivity}
\end{align}

where $H_{ip}^{(\ell)}$ denotes the $\nth{p}$ feature of the representation of node $i$ at layer $\ell$ and $X_{jq}$ is the $\nth{q}$ feature of node $j$. We evaluate the derivatives at $\mat{X}=\mat{0}$ assuming the features are sufficiently concentrated near the origin. Signal sensitivity is equivalent to the sensitivity to \emph{coherent} changes among features of input nodes of the \emph{same} class, which provides an initial intuition behind the link between homophily and information propagation through the graph: the product of derivatives $\frac{\partial H^{(\ell)}_{ip}}{\partial X_{jq}} \frac{\partial H^{(\ell)}_{ip}}{\partial X_{kr}}$ measures whether the $\nth{p}$ output dimension of node $i$ changes in the same or different direction with changes to respectively the $\nth{q}$ and $\nth{r}$ inputs of nodes $j$ and $k$, while $\delta_{y_j y_k}$ collects terms corresponding to the same class.

Similarly, the noise and global sensitivities can be calculated as:
\begin{align}\label{eq:noise_and_global_sensitivity} 
  \noisesenexp[i]{\ell}{p,q,r} := \sum_{j \in V} \frac{\partial H_{ip}^{(\ell)}}{\partial X_{jq}}\frac{\partial H_{ip}^{(\ell)}}{\partial X_{jr}}\Bigg|_{\mat{X}=\mat{0}}\quad \globsenexp[i]{\ell}{p,q,r} := \sum_{j, k \in V} \frac{\partial H_{ip}^{(\ell)}}{\partial X_{jq}} \frac{\partial H_{ip}^{(\ell)}}{\partial X_{kr}}\Bigg|_{\mat{X}=\mat{0}} 
\end{align}

The noise sensitivity measures how responsive the MPNN is to random, unstructured variations in the input features (i.e., the IID noise component in the feature decomposition). The global sensitivity measures the MPNN's sensitivity to global background changes of the input features, regardless of their alignment with the class structure.

\paragraph{The signal-to-noise ratio of MPNNs.}

To evaluate the quality of feature embeddings of MPNNs and non-relational models, we consider the signal-to-noise ratio (SNR) of their feature representations. For an $\ell$-layer MPNN, we define the SNR as:

\begin{align}\label{eq:snr_def}
  \snrsub{H_{ip}^{(\ell)}}{} := \frac{\varsub{\condexpectsub{H_{ip}^{(\ell)}}{\vect{\mu}}{\vect{\gamma}, \vect{\epsilon}}}{\vect{\mu}}}{\expectsub{\varsub{H_{ip}^{(\ell)}\,\Big|\, \vect{\mu}}{\vect{\gamma}, \vect{\epsilon}}}{\vect{\mu}}}
\end{align}

This formulation of the SNR aligns with the classical definition in statistical signal processing and information theory. The numerator, $\varsub{\condexpectsub{H_{ip}^{(\ell)}}{\vect{\mu}}{\vect{\gamma}, \vect{\epsilon}}}{\vect{\mu}}$, quantifies the variance in output feature dimension $p$ explained by class-wise feature variability, which can be interpreted as the ``signal'' strength---the extent to which the model distinguishes between classes. The denominator, $\expectsub{\varsub{H_{ip}^{(\ell)}\,\Big|\, \vect{\mu}}{\vect{\gamma}, \vect{\epsilon}}}{\vect{\mu}}$, represents the residual variation not explained by class-wise feature variability, which can be viewed as the ``noise''. By taking the ratio of these terms, $\snrsub{H_{ip}^{(\ell)}}{}$ measures how well the model separates classes (signal) relative to the intrinsic variability within classes (noise), making it a valid and meaningful measure of the model's discriminative power.

\begin{restatable}[SNR sensitivity relation]{theorem}{snrsensitivity}
\label{theorem:snrsensitivity}

Consider a feature distribution following the covariance structure in Eq.
\eqref{eq:total_covariance}. Assuming the feature distribution is concentrated near the origin, the SNR of an MPNN for the
$\nth{p}$ output feature of node $i$ at layer $\ell$, in Eq. \eqref{eq:snr_def}, is approximated by
\begin{align}\label{eq:snrsensitivity}
  \snrsub{H_{ip}^{(\ell)}}{}
  \simeq
  \frac{\displaystyle
    \sum_{q,r=1}^{d_{\mathrm{in}}}
      \Sigma_{qr}\,
      \sigsenexp[i]{\ell}{p,q,r}
    }
       {\displaystyle
    \sum_{q,r=1}^{d_{\mathrm{in}}}
      \Phi_{qr}\,\globsenexp[i]{\ell}{p,q,r}
    \;+\;
    \sum_{q,r=1}^{d_{\mathrm{in}}}
      \Psi_{qr}\,\noisesenexp[i]{\ell}{p,q,r}},
\end{align}
where the approximation denoted by $\simeq$ relies on the first-order Taylor expansion of
$H_{ip}^{(\ell)}$ around $\mat{X}=\mathbf{0}$ when computing the
variances that define the SNR.
\end{restatable}

It is intuitive that as class-specific feature variability $\left(\Sigma_{qr}\right)$ increases relative to node and global noise $\left(\Psi_{qr}, \Phi_{qr}\right)$, we expect the SNR to increase and classification performance to improve. If we further assume that different feature dimensions are IID, with variance of signal, local and global noise components defined as $\sigma^2:=\Sigma_{ii}, \psi^2 := \Psi_{ii}, \phi^2:=\Phi_{qq}$ respectively, then Theorem \ref{theorem:snrsensitivity} shows that (non-relational) feedforward neural network (FNN) models are fundamentally limited in their ability to improve the signal-to-noise ratio of their input:

\begin{align*}
  \snrsub{H_{ip}^{(\ell)}}{} \simeq \frac{ \sigma^2}{\phi^2+\psi^2} = \snrsub{X_{ip}^{(\ell)}}{},
\end{align*}

as for a given FNN model, $\sigsenexp[i]{\ell}{p,q,r} = \noisesenexp[i]{\ell}{p,q,r} = \globsenexp[i]{\ell}{p,q,r}$ due to the lack of interaction between nodes in the forward pass computation. Theorem \ref{theorem:snrsensitivity} shows that MPNNs have the potential to enhance the SNR beyond this limit. However, this improved performance is subject to the following condition, which we term the ``sensitivity condition''.\\

\begin{restatable}[Sensitivity condition]{corollary}{sensitivityratiocondition}\label{corollary:gdl_criterion}
Consider a feature distribution following the covariance structure in Eq. \eqref{eq:total_covariance}, and having IID feature dimensions. Let $\rho := \frac{\psi^2}{\phi^2+\psi^2}$ be the local noise proportion, i.e. the proportion of noise accounted for by local perturbations where $0\leq \rho \leq 1$. Then an MPNN improves the SNR of any input feature distribution for the $\nth{p}$ output feature of node $i$ if and only if:
\begin{align}\label{eq:gdl_criterion}
 \sum_{q=1}^{d_\mathrm{in}}\sigsenexp[i]{\ell}{p,q,q} >\rho \sum_{q=1}^{d_\mathrm{in}}\noisesenexp[i]{\ell}{p,q,q} + (1-\rho)\sum_{q=1}^{d_\mathrm{in}}\globsenexp[i]{\ell}{p,q,q}.
\end{align}
\end{restatable}

\begin{figure*}[!htbp]
  \centering
  
  \begin{subfigure}[t]{0.55\textwidth}
    \includegraphics[width=\textwidth]{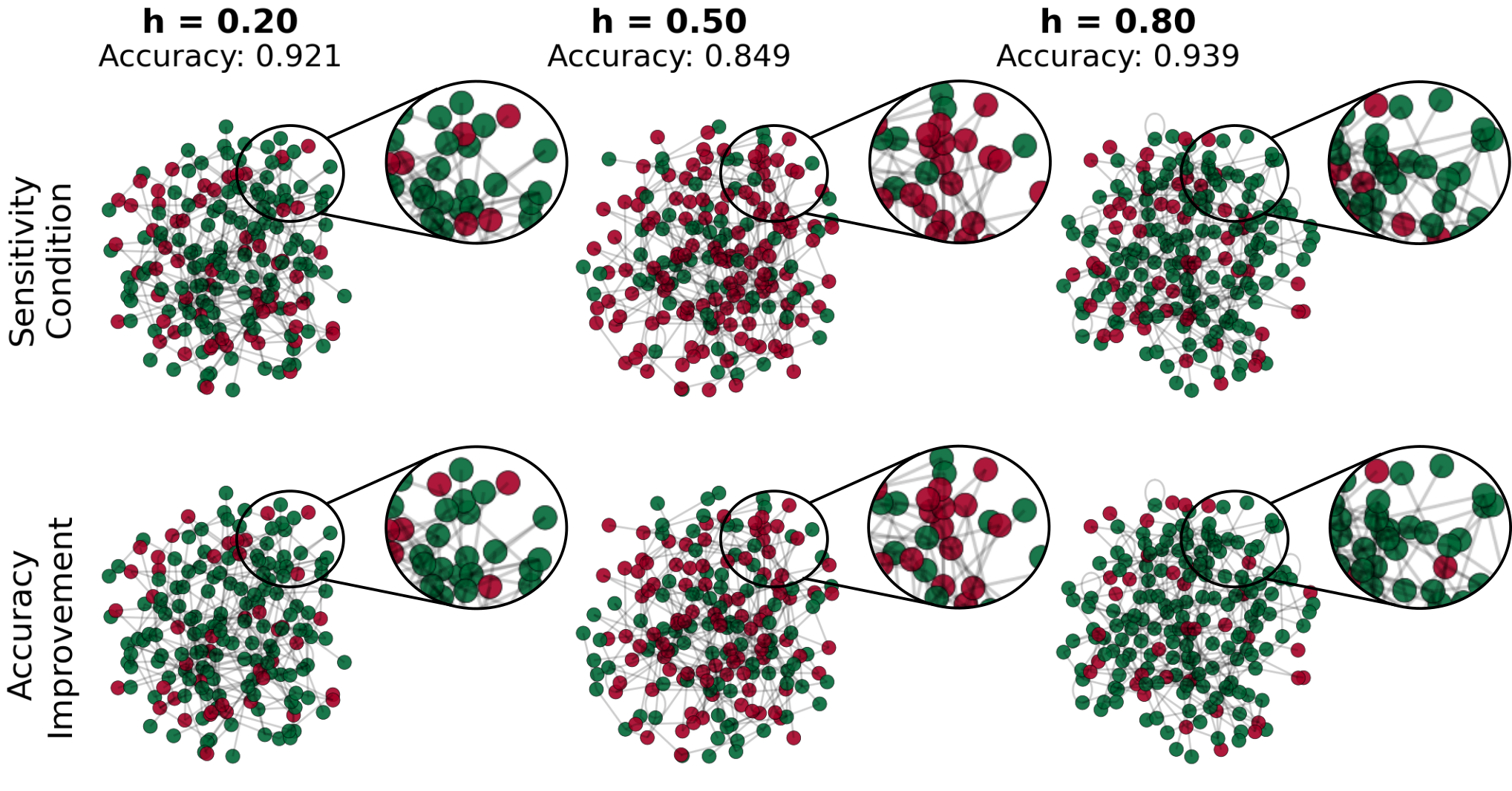}
    \caption{Node-level sensitivity condition vs. accuracy improvement}
    \label{fig:snr_node_predictions}
  \end{subfigure}
  \hfill
  \begin{subfigure}[t]{0.44\textwidth}
    \includegraphics[width=\textwidth]{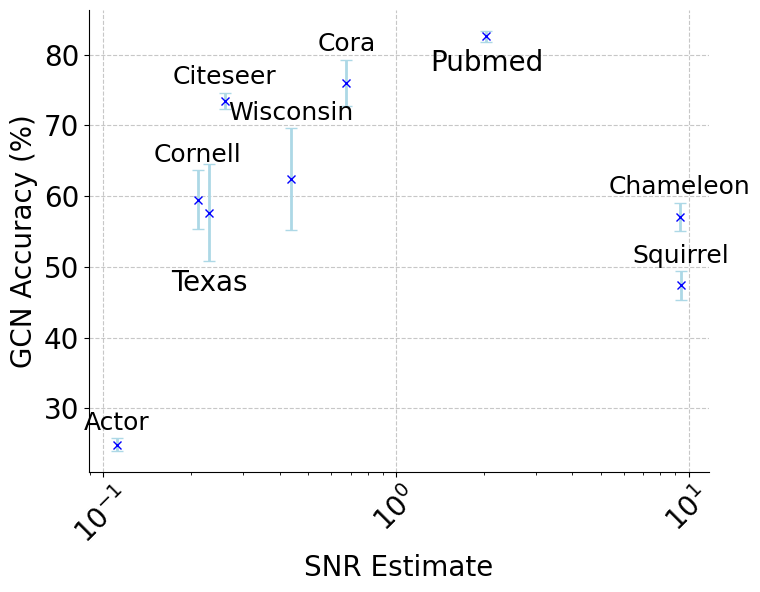}
    \caption{Mean SNR and accuracy on real-world graphs}
    \label{fig:homophily_snr}
  \end{subfigure}
  
  \vspace{0.6cm}
  
  \begin{subfigure}[t]{0.9\textwidth}
    \includegraphics[width=\textwidth]{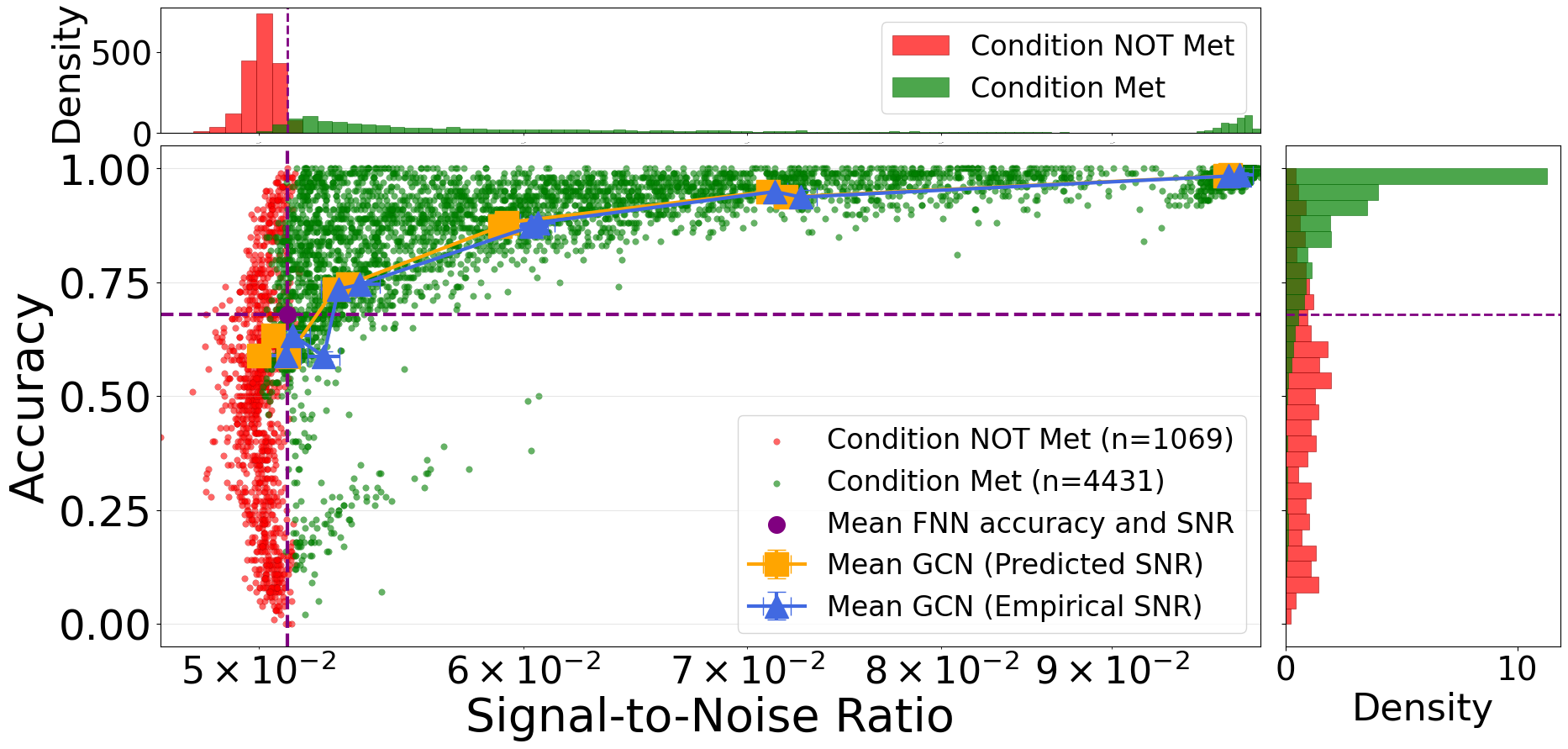}
    \caption{Node-level SNR vs. accuracy}
    \label{fig:snr_vs_acc}
  \end{subfigure}

  \caption{ 
  \textbf{The sensitivity condition correctly identifies nodes for which MPNNs outperform FNNs.} ~\textbf{(a)} ~The sensitivity condition (Eq.~\eqref{eq:gdl_criterion}) provides a local, node-level predictor for a GCN's performance advantage over an FNN. Nodes coloured {green}/{red} indicate (i) where the condition is satisfied/not satisfied in the top row of graphs, and (ii) whether the GCN accuracy is improved/not improved over the FNN in the bottom row, respectively. The accuracy of the sensitivity condition ranges between 0.8 and 0.9, which highlights the condition's ability to identify nodes where the graph structure aids classification.
  \textbf{(b)} The predicted SNR from Theorem \ref{theorem:snrsensitivity} averaged over the whole graph correlates with GCN test accuracy for various real-world graph datasets, demonstrating the applicability of this estimate as a diagnostic tool in a wide range of settings. \textsc{Chameleon} and \textsc{Squirrel} datasets appear to break the trend, as they are widely known to be problematic datasets in the GNN literature due to having duplicate nodes and train/test data leakage \cite{platonov2024criticallookevaluationgnns}.
  \textbf{(c)} Empirical relationship between predicted SNR and test accuracy, with their marginal distributions. Higher SNR strongly correlates with improved accuracy, validating SNR as a meaningful performance metric. Individual nodes' SNR are plotted, coloured by whether they satisfy the sensitivity condition. The empirical and predicted SNR averaged over all nodes for each graph are shown in {blue} and {orange} respectively, and can be seen to closely match. The {purple} dashed lines indicate the baseline FNN accuracy (0.7) and corresponding SNR threshold (0.05). In the marginal distribution plots, we can see that the majority of nodes which satisfy the sensitivity condition tend to lie right of the purple dashed line for SNR and above the dashed line for accuracy; and vice versa for nodes that do not satisfy the sensitivity condition.
  Experimental details, including graph generation, feature sampling, model training, empirical SNR estimation (Eq.~\eqref{eq:final_snr_est}), and sensitivity calculation via Jacobians, are provided in the \nameref{sec:methods} section; see \nameref{sec:experimental_details}. 
  }
  \label{fig:snr_analysis}
\end{figure*}

Theorem \ref{theorem:snrsensitivity} and Corollary \ref{corollary:gdl_criterion} reveal how the SNR of an MPNN is directly influenced by its sensitivity to the signal $\sigsenexp[i]{\ell}{p,q,r}$, to the noise $\noisesenexp[i]{\ell}{p,q,r}$, and to global shifts $\globsenexp[i]{\ell}{p,q,r}$ in the features. The sensitivity condition in Eq. \eqref{eq:gdl_criterion} establishes that for an MPNN to outperform a non-relational FFN, the signal sensitivity must exceed a convex combination of the noise and global sensitivities, controlled by the local noise
proportion $\rho$. The condition surprisingly does not depend on the class-wise variance $\sigma^2$, suggesting that the degree to which message passing may improve class-specific separability over FNNs does not depend on class-wise signal quality, but on having the appropriate \emph{kind} of noise. The local noise proportion $\rho$ in Eq. \eqref{eq:gdl_criterion} controls the difficulty of the classification task on a particular feature distribution: In the high \textit{global} sensitivity regime where $\globsenexp[i]{\ell}{p,q,q}>\noisesenexp[i]{\ell}{p,q,q}$ (such as GCNs with low-pass graph filters) \textit{larger} $\rho$ makes the condition easier to satisfy; but the high \textit{local} sensitivity regime where $\globsenexp[i]{\ell}{p,q,q}<\noisesenexp[i]{\ell}{p,q,q}$ (such as GCNs with high-pass graph filters) \textit{smaller} $\rho$ makes the condition easier to satisfy.

By quantifying this balance, practitioners can use the sensitivity measures as a \emph{feature-independent} and \emph{localised} (i.e. dependent on $i$) diagnostic tool to evaluate whether their MPNN architecture is suitable for the given task, and predict when and where the model will struggle in noise-dominated environments. Figure \ref{fig:homophily_snr} demonstrates how classification accuracy correlates with the SNR calculated using Theorem \ref{theorem:snrsensitivity}, and Figures \ref{fig:snr_node_predictions} and \ref{fig:snr_vs_acc} show how accuracies can be predicted based on the sensitivity criterion in Corollary \ref{corollary:gdl_criterion}, in both synthetic and real-world datasets.

Having established the relationship between an MPNN's SNR and its signal, noise, and global sensitivities (Theorem \ref{theorem:snrsensitivity}), an important question arises: what determines these sensitivities? Since the sensitivities are feature-independent, their values must be governed by the underlying graph structure and the MPNN architecture. The following section explores this relationship, introducing the concept of class-bottlenecks and higher-order homophily to quantify how graph connectivity patterns directly influence and bound the signal sensitivity, thereby impacting the potential SNR gains.

\subsection*{Class-bottlenecks restrict the signal sensitivity of message passing}

The problem of how homophily and bottlenecks dictate the fundamental performance limits of MPNNs can be first tackled by examining the condition for relational learning (Corollary \ref{corollary:gdl_criterion}) through the lens of graph structural properties, by specifically focusing on how connectivity patterns affect signal propagation. We demonstrate that limits on signal propagation arise from specific graph structures, which we term as ``class-bottlenecks''.

\paragraph{Motivating example: simple graph convolution.}
To illustrate the concept of class-bottlenecking, consider a Simple Graph Convolution (SGC) model \cite{wu2019simplifyinggraphconvolutionalnetworks}. In an SGC, node representations are updated linearly by averaging over neighbours' features, followed by a linear transformation. The $\ell$-layer update rule is:
\begin{align}\label{eq:sgc}
\mat{H}^{(\ell)} := \hat{\mat{A}}_{\mathrm{sym}} \mat{H}^{(\ell-1)} \mat{W}^{(\ell)},
\end{align}
where $ \mat{H}^{(\ell)} $ is the representation matrix at layer $ \ell $, $ \hat{\mat{A}}_{\mathrm{sym}} := \mat{D}^{-\frac{1}{2}} \mat{A} \mat{D}^{-\frac{1}{2}}$ is the symmetric normalised adjacency matrix, $\mat{D}$ is the diagonal degree matrix, and $ \mat{W}^{(\ell)} $ is the layer's weight matrix. Self-loops are also typically added to the graph for stability when calculating $A_{\mathrm{sym}}$. The overall transformation after $\ell$ layers is $\mat{H}^{(\ell)} = \hat{\mat{A}}_{\mathrm{sym}}^{\ell} \mat{X}\mat{W}$, where $\mat{W}:=\mat{W}^{(1)} \cdots \mat{W}^{(\ell)}$. Due to linearity, the sensitivities at a specific node $i$ for output dimension $p$ with respect to input dimensions $q,r$ can be calculated exactly; they are directly proportional to specific local graph structural properties:
\begin{equation}\label{eq:individual_signal_sensitivity_sgc}
  \resizebox{\hsize}{!}{
    $\sigsenexp[i]{\ell}{p,q,r} = W_{pq}W_{pr} \cdot \localhomo{\ell,\ell}{\hat{\mat{A}}_{\mathrm{sym}}}{i}, \quad \globsenexp[i]{\ell}{p,q,r} = W_{pq}W_{pr} \cdot \localconn{\ell,\ell}{\hat{\mat{A}}_{\mathrm{sym}}}{i}, \quad \noisesenexp[i]{\ell}{p,q,r} = W_{pq}W_{pr}\cdot \localselfconn{\ell,\ell}{\hat{\mat{A}}_{\mathrm{sym}}}{i}.$
    }
\end{equation}

Here, we define the local quantities based on the graph shift operator $\hat{\mat{A}}$, over (potentially equal) pairs of source nodes $j,k$:
\begin{equation}\label{eq:p_order_local_homophily}
\begin{aligned}
  \textit{Class-bottlenecking score: } \localhomo{r,s}{\mat{\hat{A}}}{i} &:= \sum_{j,k \in V} \left[{\hat{\mat{A}}^r}\right]_{ij}\left[{\hat{\mat{A}}^s}\right]_{ik}\delta_{y_j y_k}, \\
  \textit{Self-bottlenecking score: } \localselfconn{r,s}{\mat{\hat{A}}}{i} &:= \sum_{j \in V} \left[{\hat{\mat{A}}^r}\right]_{ij}\left[{\hat{\mat{A}}^s}\right]_{ij}, \\
  \textit{Total-bottlenecking score: } \localconn{r,s}{\mat{\hat{A}}}{i} &:= \sum_{j,k\in V} \left[{\hat{\mat{A}}^r}\right]_{ij}\left[{\hat{\mat{A}}^s}\right]_{ik}.
\end{aligned}
\end{equation}

The class-bottlenecking score $\localhomo{r,s}{\hat{\mat{A}}}{i}$ quantifies the number of path pairs of lengths $r$ and $s$, originating from pairs of source nodes that belong to the \emph{same class}, and terminating at target node $i$. A low score indicates that same-class signals arriving at node $i$ via paths of length $r$ and $s$ are scarce, creating a bottleneck for class-specific information aggregation at node $i$. Importantly, these scores are agnostic to model parameters and depend purely on the graph structure and node class labels. 

\paragraph{Class-bottlenecks.} Class-bottlenecks occur at target nodes $i$ where $\localhomo{r,s}{\hat{\mat{A}}}{i}$ is low. The score in Eq. \eqref{eq:p_order_local_homophily} depends on two factors: the class alignment or ``homophily'' $\delta_{y_j y_k}$, and the strength of connectivity $[{\hat{\mat{A}}^r}]_{ij}[{\hat{\mat{A}}^s}]_{ik}$. The connectivity factors capture the amount of structural bottlenecking, as demonstrated by Topping et al. \cite{topping2021understanding}, where powers of the symmetric normalised adjacency matrix are bounded by the Cheeger constant---a quantity well known in the graph theory literature to capture structural bottlenecks in the graph \cite{MOHAR1989274}. We show that it is more specifically class-dependent bottlenecking that determines the MPNN's signal sensitivity. For a fixed graph structure with a structural bottleneck, if most source node pairs at radii $r,s$ from $i$ are of different classes, i.e. $\delta_{y_j y_k}=0$, as shown in Figure \ref{fig:homophilic_bottleneck_bad}, we have a more severe class-bottleneck with a lower $\localhomo{r,s}{\hat{\mat{A}}}{i}$. If most pairs share the same class, the bottleneck shown in Figure \ref{fig:homophilic_bottleneck_good}, results in a milder reduction in $\localhomo{r,s}{\hat{\mat{A}}}{i}$.

\paragraph{The SNR of an SGC is given by feature and graph-level quantities.}
Assuming IID feature dimensions as in Corollary \ref{corollary:gdl_criterion}, substituting the closed-form sensitivities from
Eq.~\eqref{eq:individual_signal_sensitivity_sgc}
into Theorem~\ref{theorem:snrsensitivity} gives an explicit expression for the SNR of an $\ell$-layer SGC at node $i$ along output dimension $p$:
\begin{equation}
  \snrsub{H_{ip}^{(\ell)}}{\mathrm{}}
  \;=\;
  \frac{\sigma^{2}}{\phi^{2}+\psi^{2}}
  \cdot
  \frac{\localhomo{\ell,\ell}{\hat{\mat{A}}_{\mathrm{sym}}}{i}}
       {\rho\,\localselfconn{\ell,\ell}{\hat{\mat{A}}_{\mathrm{sym}}}{i}
        \;+\;
        (1-\rho)\,\localconn{\ell,\ell}{\hat{\mat{A}}_{\mathrm{sym}}}{i}},
  \qquad
  \rho:=\frac{\psi^{2}}{\phi^{2}+\psi^{2}},
  \label{eq:sgc_snr_isotropic}
\end{equation}

\noindent
where the weight factors cancel out. Hence all dependence on trainable parameters and raw feature statistics collapses into the scalar pre-factor
$\sigma^{2}\!/\!(\phi^{2}+\psi^{2})$ and local noise proportion $\rho:=\frac{\psi^{2}}{\phi^{2}+\psi^{2}}$, and the
$\ell$-hop connectivity patterns are captured by the three local scores introduced in Eq.~\eqref{eq:p_order_local_homophily}.

Figure \ref{fig:homophilic_bottleneck_bad} illustrates a class-bottleneck at  node $T$. If only paths between nodes of different classes pass through $T$ $\localhomo{\ell,\ell}{\hat{\mat{A}}}{i}$ will be low, directly reducing the signal sensitivity $\sigsenexp[i]{\ell}{p,q,r}$ according to Eq. \eqref{eq:individual_signal_sensitivity_sgc}. This low signal sensitivity makes it harder to satisfy the sensitivity condition (Corollary \ref{corollary:gdl_criterion}), potentially preventing the SGC from outperforming an FNN. For instance, in the example graph in Figure \ref{fig:homophilic_bottleneck_bad} if $\ell=1$, then $\localhomo{1,1}{\hat{\mat{A}}_{\mathrm{sym}}}{T} = \frac{1}{2}$, $\localconn{1,1}{\hat{\mat{A}}_{\mathrm{sym}}}{T} = 1$, and $\localselfconn{1,1}{\hat{\mat{A}}_{\mathrm{sym}}}{T} = \frac{1}{2}$. The SGC cannot improve the SNR over an FNN at node $T$ because the sensitivity condition requires $\localhomo{1,1}{\hat{\mat{A}}_{\mathrm{sym}}}{T} > \rho \ \localselfconn{1,1}{\hat{\mat{A}}_{\mathrm{sym}}}{T} + (1-\rho) \ \localconn{1,1}{\hat{\mat{A}}_{\mathrm{sym}}}{T}$, which simplifies to $\frac{1}{2} > \rho \cdot \frac{1}{2} + (1-\rho) \cdot 1$, implying $\rho > 1$, which cannot happen as $\rho \in [0,1]$.

\paragraph{Higher-order homophily measures average amount of class-bottlenecking.}
A key insight emerges when we examine the global behaviour of class-bottlenecks: the local class-bottlenecking scores, when averaged across all nodes, can be expressed exactly in terms of the weighted homophily measure from Eq. \eqref{eq:def_weighted_homophily}:
\begin{align}\label{eq:bottleneck_homophily_connection}
    \frac{1}{n}\sum_{i=1}^n \localhomo{r,s}{\hat{\mat{A}}_{\mathrm{sym}}}{i} = \rhomo{r+s}{\hat{\mat{A}}_{\mathrm{sym}}},
\end{align}

for an undirected graph $G$. Here, using $\hat{\mat{A}}_{\mathrm{sym}}^{r+s}$ as the argument to weighted homophily in Eq.~\eqref{eq:def_weighted_homophily}, results in a measure of \emph{higher-order homophily}---the tendency for same-class nodes to be preferentially connected through multi-hop paths. For directed graphs, the general form $[\hat{\mat{A}}_{\mathrm{sym}}^r]^T\hat{\mat{A}}_{\mathrm{sym}}^s$ should be used instead of $\hat{\mat{A}}_{\mathrm{sym}}^{r+s}$. This relationship establishes that higher-order homophily measures average amount of class-bottlenecking. Unlike first-order homophily measures like edge and node homophily---that only consider direct edges---higher-order homophily captures richer connectivity patterns that determine the effectiveness of message-passing in deeper networks.

To provide a complete characterisation of graph connectivity patterns relevant to MPNN performance, we define analogous global measures for self-connectivity---how well nodes connect back to themselves through multi-hop paths---and total connectivity:
\begin{align}\label{eq:p_order_connectivity}
  \rselfconn{}{\mat{S}} \coloneqq \frac{1}{n} \sum_{i=1}^n \left[{\mat{S}}\right]_{ii}, \quad \rconn{}{\mat{S}} \coloneqq \frac{1}{n} \sum_{i,j=1}^n \left[{\mat{S}}\right]_{ij}.
\end{align}
The relationship to their local counterparts is analogous: $\frac{1}{n}\sum_{i=1}^n \localselfconn{r,s}{\mat{\hat{A}}}{i} =\rselfconn{r+s}{\hat{\mat{A}}_{\mathrm{sym}}}$ and $\frac{1}{n}\sum_{i=1}^n \localconn{r,s}{\mat{\hat{A}}}{i} =\rconn{r+s}{\hat{\mat{A}}_{\mathrm{sym}}}$.

\paragraph{SGC sensitivities as higher-order homophily measures.} These global connectivity measures lead directly to an important result for SGCs. The average sensitivities, defined as $\overbar{\sigsenexp[p]{\ell}{q,r}}:=\frac{1}{n}\sum_{i=1}^{n}\sigsenexp[i]{\ell}{p,q,r}$, $\overbar{\globsenexp[p]{\ell}{q,r}}:=\frac{1}{n}\sum_{i=1}^{n}\globsenexp[i]{\ell}{p,q,r}$, and $\overbar{\noisesenexp[p]{\ell}{q,r}}:=\frac{1}{n}\sum_{i=1}^{n}\noisesenexp[i]{\ell}{p,q,r}$, can be expressed entirely in terms of these higher-order connectivity measures:
\begin{equation}\label{eq:mean_signal_sensitivity_sgc}
\resizebox{\hsize}{!}{
  $\overbar{\sigsenexp[p]{\ell}{q,r}} = W_{pq}W_{pr} \cdot \rhomo{2\ell}{\hat{\mat{A}}_{\mathrm{sym}}},
  \quad \overbar{\globsenexp[p]{\ell}{q,r}} = W_{pq}W_{pr} \cdot \rconn{2\ell}{\hat{\mat{A}}_{\mathrm{sym}}},
  \quad \overbar{\noisesenexp[p]{\ell}{q,r}} = W_{pq}W_{pr} \cdot \rselfconn{2\ell}{\hat{\mat{A}}_{\mathrm{sym}}}.$  
  }
\end{equation}
This result is significant because it shows that the average signal sensitivity of an $\ell$-layer SGC is completely determined by the $2\ell$-order homophily of the graph. In other words, the model's ability to distinguish between classes depends entirely on how well same-class nodes are connected through paths of length $2\ell$. Similarly, the noise and global sensitivities depend on self-connectivity and total connectivity at order $2\ell$, respectively. This provides a direct, computable link between graph structure and potential MPNN performance, independent of the specific feature values.

\paragraph{Bounding sensitivities of general isotropic MPNNs.}
The connection between class-bottlenecks and sensitivity extends beyond the linear SGC. For a general isotropic MPNN---where the message function $M_\ell$ in Eq. \eqref{eq:mpnn} does not depend on the source node's own representation, i.e., $\norm{\nabla_1 M_\ell}=0$---the sensitivities at node $i$ can be bounded. Assuming $\norm{ \nabla_1 U_s}\leq \alpha_{1}$, $\norm{ \nabla_2 U_s}\leq \alpha_{2}$, and $ \norm{\nabla_2 M_s} \leq \beta$ exist for layers $s=1,\dots,\ell$:
\begin{equation}\label{eq:individual_signal_sensitivity_mpnn}
\begin{aligned}
  \abs{\sigsenexp[i]{\ell}{p,q,r}}&\leq \sum_{s,t=0}^{\ell}{\binom{\ell}{s}\binom{\ell}{t}}\alpha_1^{2 \ell - s-t} \left(\alpha_2\beta\right)^{s+t}\localhomo{s,t}{\hat{\mat{A}}}{i}, \\
  \abs{\globsenexp[i]{\ell}{p,q,r}} &\leq \sum_{s,t=0}^{\ell}{\binom{\ell}{s}\binom{\ell}{t}}\alpha_1^{2 \ell - s-t} \left(\alpha_2\beta\right)^{s+t}\localconn{s,t}{\hat{\mat{A}}}{i}, \\
  \abs{\noisesenexp[i]{\ell}{p,q,r}} &\leq \sum_{s,t=0}^{\ell}{\binom{\ell}{s}\binom{\ell}{t}}\alpha_1^{2 \ell - s - t} \left(\alpha_2\beta\right)^{s+t}\localselfconn{s,t}{\hat{\mat{A}}}{i}.
\end{aligned}
\end{equation}
These bounds directly link the local signal sensitivity $\sigsenexp[i]{\ell}{p,q,r}$ to the class-bottlenecking score $\localhomo{s,t}{\hat{\mat{A}}}{i}$ at that node across different path lengths $s, t$. Nodes suffering from strong class-bottlenecks (low scores) will have inherently limited signal sensitivity, regardless of the specific MPNN architecture (within the isotropic class).

Averaging these bounds across all nodes and applying Vandermonde's identity (see Corollary \ref{corollary:avg_sensitivity_bounds} in \nameref{sec:appendix_A}) yields bounds on the average sensitivities in terms of higher-order homophily:
\begin{equation}\label{eq:avg_sensitivity_bound}
\begin{aligned}
\abs{\overbar{\sigsenexp[p]{\ell}{q,r}}} &\leq \sum_{u=0}^{2\ell} \binom{2\ell}{u}\,\alpha_1^{2\ell-u} (\alpha_2\beta)^u\, \rhomo{u}{\hat{\mat{A}}}, \\
\abs{\overbar{\globsenexp[p]{\ell}{q,r}}} &\leq \sum_{u=0}^{2\ell} \binom{2\ell}{u}\,\alpha_1^{2\ell-u} (\alpha_2\beta)^u\, \rconn{u}{\hat{\mat{A}}}, \\
\abs{\overbar{\noisesenexp[p]{\ell}{q,r}}} &\leq \sum_{u=0}^{2\ell} \binom{2\ell}{u}\,\alpha_1^{2\ell-u} (\alpha_2\beta)^u\, \rselfconn{u}{\hat{\mat{A}}},
\end{aligned}
\end{equation}
for a symmetric graph shift operator $\hat{\mat{A}}$. For asymmetric graph shift operators, such as the random-walk normalised adjacency matrix or directed graphs, the general form $\left[\hat{\mat{A}}^r\right]^T\hat{\mat{A}}^s$ should be used instead of $\hat{\mat{A}}^{r+s}$ (see more in the proof of Corollary \ref{corollary:avg_sensitivity_bounds} in \nameref{sec:proofs}). 

Eqs. \eqref{eq:individual_signal_sensitivity_mpnn} and \eqref{eq:avg_sensitivity_bound} establish that low class-bottlenecking scores restrict signal sensitivity locally, while low higher-order homophily restricts it globally. This provides a fundamental reason why MPNNs may struggle on graphs where same-class nodes are poorly connected over multiple hops (i.e., some heterophilic graphs or graphs with strong community structures misaligned with classes). A more general formulation for anisotropic models is given in Theorem \ref{theorem:general_signal_sensntivity_bound} in \nameref{sec:appendix_A}.

Applying the general bound in Eq. \eqref{eq:individual_signal_sensitivity_mpnn} to a standard GCN with $\mathbf{H}^{(\ell+1)} = \operatorname{ReLU}(\hat{\mathbf{A}}_{\mathrm{sym}}\,\mathbf{H}^{(\ell)}\,\mathbf{W}^{(\ell)})$, we note $\alpha_1 = 0$. Using the bound $\norm{\nabla_2 U_s} \le 1$ (due to ReLU) and $\norm{\nabla_2 M_s} \le \max_k \norm{\mat{W}^{(k)}} =: \beta$, the bounds resemble those for the SGC: $\sigsenexp[i]{\ell}{p,q,r} \leq \beta^{2\ell} \localhomo{\ell,\ell}{\mat{\hat{A}}_{\mathrm{sym}}}{i}$, $\globsenexp[i]{\ell}{p,q,r} \leq \beta^{2\ell} \localconn{\ell,\ell}{\mat{\hat{A}}_{\mathrm{sym}}}{i}$, and $ \noisesenexp[i]{\ell}{p,q,r} \leq \beta^{2\ell} \localselfconn{\ell,\ell}{\mat{\hat{A}}_{\mathrm{sym}}}{i}$. In Figure \ref{fig:bottleneckingscore_estimate_plot}, we empirically show that under low variance conditions these upper bounds, and those of Eq. \eqref{eq:individual_signal_sensitivity_mpnn} in general, are tight as the model is close to being linear.

Eqs. \eqref{eq:individual_signal_sensitivity_mpnn} and \eqref{eq:avg_sensitivity_bound} reveal the critical role of class-bottlenecking (locally) and higher-order homophily (globally) in bounding MPNN sensitivities and thus performance potential. Structures limiting same-class connectivity across multiple hops impede signal propagation. To gain a more quantitative understanding of how specific graph topologies create these bottlenecks and influence higher-order homophily, we now shift our view from the discrete analysis of specific graph instances to a statistical analysis using graph ensembles. By considering random graph models we can decompose the factors affecting information flow---oversquashing and underreaching---into interpretable graph properties and quantitatively link them to bottlenecking. In particular, we study how fundamental graph properties such as the mean degree and edge homophily affect class-bottlenecking scores (and consequently signal sensitivity), as detailed in the next section.

\subsection*{Graph ensembles enable a geodesic-based decomposition of higher-order homophily into oversquashing and underreaching}

A key aspect of understanding higher-order homophily and bottlenecking in MPNNs involves characterising matrix powers of the graph shift operator $\hat{\mat{A}}$, as in Eq. \eqref{eq:avg_sensitivity_bound}. These matrix powers appear in many sensitivity analyses and are a fundamental result of layered nature of MPNNs \cite{topping2021understanding,digiovanni2023does,black2023understanding}. We can make significant progress in understanding these matrix powers by relaxing from considering the performance of  a particular graph instance, $\mat{A}$, to a graph ensemble $\expect{\mat{A}}$ that could have generated the graph $\mat{A}$. We can then compute characteristic results for classification SNR in terms of expected higher-order homophily.

Specifically, let the (undirected and simple) graph be a sample from a general random graph family with conditionally independent edges, that is, without loss of generality for node indices $i<j:A_{ij}\sim\bernoulli{\left[\expect{\mat{A}}\right]_{ij}}$ and $A_{ji}=A_{ij}$. In other words, the graph ensemble is completely characterised by the expected adjacency matrix $\expect{\mat{A}}$, and includes many widely used random graph models like stochastic block models (SBMs, \cite{holland1983sbm}) and random dot product graphs \cite{young2007rdpg}. Let $\lambda_{ij}$ be the shortest path length between nodes $i,j$. The ensemble induces a distribution on these lengths \cite{loomba2021geodesic}, allowing us to decompose the expectation of powers of the graph shift operator $\hat{\mat{A}}$ as: 
\begin{align}
\label{eq:underover}
  \expect{\hat{\mat{A}}^r}_{ij} = \sum_{t=1}^r\underbrace{\condexpect{\left[\hat{\mat{A}}^r\right]_{ij}}{\lambda_{ij} = t}}_{\text{oversquashing}} \cdot \underbrace{\prob{\lambda_{ij} = t}}_{\text{underreaching}},
\end{align}
as $\lambda_{ij}> r \implies \left[\hat{\mat{A}}^r\right]_{ij}=0$. We call $r$ the receptive field size. The first factor captures connection density within a $r$-hop radius of node $i$, the receptive field of the $\nth{r}$ layer of the MPNN---contributing to oversquashing as shown by Topping et al. \cite{topping2021understanding}---while the second is the probability that node $j$ is reachable from $i$ in exactly $r$ hops---contributing to underreaching as defined by Alon and Yahav \cite{alon2021bottleneck}. Unlike previous works, we view these quantities in expectation, which allows us to make analytical progress. For sparse graph ensembles, i.e. bounded degree graphs, asymptotic approximations can be derived for the underreaching and oversquashing factors in terms of $\expect{\mat{A}}$: \\

\begin{restatable}[Underreaching and oversquashing in sparse graph ensembles]{theorem}{underreachingoversquashing}\label{theorem:underreaching_oversquashing}
For an undirected and simple graph with $n$ nodes encoded by the adjacency matrix $\mat{A}$, sampled from a general random graph family with conditionally independent edges and expected adjacency matrix $\expect{\mat{A}}$, under conditions for sufficient sparsity (see Lemma \ref{lemma:underreaching} in the \nameref{sec:appendix_A}), we have that:

\begin{align}
    & \prob{\lambda_{ij}= r} \approx 
    \left[\expect{\mat{A}}^r\right]_{ij},\label{eq:underreachingoversquashing_underreaching}\\ 
    & \condexpect{\left[ \hat{\mat{A}}^r_\mathrm{sym} \right]_{ij} }{ \lambda_{ij} = r } \approx \frac{ \left[ \left(\avg{\mat{D}}^{-\frac{1}{2}} \expect{\mat{A}}\avg{\mat{D}}^{-\frac{1}{2}} \right)^{r}\right]_{ij} }{ \left[ \expect{\mat{A}}^r \right]_{ij} } + O\left( \frac{1}{\avg{d}^{r+1}} \right),\label{eq:underreachingoversquashing_oversquashing}
\end{align}

where $\approx$ is used throughout to mean equality up to $o\left(\frac{1}{n}\right)$ terms as $n\to\infty$, $\avg{\mat{D}}\coloneqq \diag{\expect{\mat{A}}\ones{n}}$ is the diagonal matrix of expected degrees, $\vect{1}_n$ is the vector of all ones, and $\avg{d}$ is the overall mean degree which is assumed to be large but much smaller than the number of nodes, i.e. $\avg{d}=o(n)$. For all shortest paths of length $t <r$, the oversquashing factor scales as:

\begin{align}\label{eq:underreachingoversquashing_residual}
  \condexpect{\left[\hat{\mat{A}}^r_\mathrm{sym}\right]_{ij}}{\lambda_{ij} = t} \approx O\left(\frac{1}{\avg{d}^{r}}\right).
\end{align}
\end{restatable}

Theorem \ref{theorem:underreaching_oversquashing} provides asymptotic approximations for the underreaching and oversquashing factors introduced in Eq.~\eqref{eq:underover}, directly relating them to the expected properties of the graph ensemble. The key idea of the Theorem, given by Eqs. \eqref{eq:underreachingoversquashing_oversquashing} and \eqref{eq:underreachingoversquashing_residual}, is that the effects of oversquashing are sharply concentrated at the longest possible shortest path length $\lambda_{ij}= r$, i.e. when it is equal to the receptive field size, while contributions from potentially shorter shortest paths are relatively negligible---scaling as $O\left(\frac{1}{\avg{d}^{r}}\right)$, and because the probability of their occurrence scales at most as $\prob{\lambda_{ij}= r-1} \approx \left[\expect{\mat{A}}^{r-1}\right]_{ij} = O\left(\frac{\avg{d}^{r-1}}{n}\right)$ their joint contribution to higher-order homophily in Eq. \eqref{eq:underover} vanishes as $O\left(\frac{1}{\avg{d}}\right)$ when summing over all $n$ nodes. The intuition behind this result stems from noting that paths from $i$ that reach the receptive field boundary at $j$ are asymptotically independent (non-overlapping) in sparse graphs with conditionally independent edges. Eq. \eqref{eq:underreachingoversquashing_oversquashing}, which approximates this boundary-oversquashing, thus gives the expected powers of the normalised adjacency matrix using powers of the normalised expected adjacency matrix. Similarly, Eq. \eqref{eq:underreachingoversquashing_underreaching} approximates the underreaching term $\prob{\lambda_{ij}= r}$, the probability that the shortest path between nodes $i$ and $j$ has length exactly $r$, using the $\nth{(i,j)}$ entry of the $\nth{r}$ power of the expected adjacency matrix $\expect{\mat{A}}$. 

Together, these approximations enable the estimation of expected higher-order homophily directly from the parameters of a specific sparse graph ensemble, as we do next.

\paragraph{Stochastic block models.}

Working within graph ensembles, we can now vary a small set of interpretable parameters---such as the probability of nodes to connect conditioned on their class labels---to create controlled experiments that interpolate smoothly between homophilic and heterophilic regimes; the Stochastic Block Model (SBM) provides precisely this sandbox. The graph ensemble is specified by a block-probability matrix $\mat{B}\!\in\!\mathbb{R}^{k\times k}$ whose entry $B_{uv}$ gives the connection probability for nodes in classes $u$ and $v$, and by a diagonal matrix of expected class proportions $\mat{\Pi}\!=\!\operatorname{diag}(\pi_1,\dots,\pi_k)$ that provides the relative class sizes.  Conditioned on the class labels $\vect{y}$, edges are sampled independently with probability $\frac{1}{n}B_{y_iy_j}$, so that $\mathbb{E}[A_{ij}] = \frac{1}{n}B_{y_iy_j}$.  Because $\mat{B}$ can be tuned from being purely diagonal (perfect homophily) to purely off-diagonal (perfect heterophily) and everything in between, the SBM lets us explore how gradual changes in class-graph correlation drive the transition between easy and hard regimes for message passing.  Moreover, by working with $\mathbb{E}[\mathbf{A}]$ rather than a single adjacency matrix $\mat{A}$ we obtain tractable approximations, such as Eq.~\eqref{eq:expected_l_homophily_approx}, that connect the parameters $(\mat{B},\mat{\Pi})$ directly to structural limits on MPNN sensitivity.  In what follows, we therefore adopt the SBM as the graph's generative model whenever we wish to reason analytically about how graph structure and class structure correlate.

Applying Theorem \ref{theorem:underreaching_oversquashing} together with the underreaching-oversquashing decomposition in Eq. \eqref{eq:underover}, we have that for sparse SBM graphs with sufficiently large mean class degrees, the $\ell$-order homophily, self-connectivity and total connectivity can be approximated in expectation as:

\begin{equation}\label{eq:expected_l_homophily_approx}
\begin{aligned}
  \expect{\rhomo{\ell}{\hat{\mat{A}}_\mathrm{sym}}} &\approx \trace \left(\mat{\Pi}^{\frac{1}{2}} \hat{\mat{B}}^{2 \ell} \mat{\Pi}^{\frac{1}{2}}\right) + O\left(\frac{1}{\avg{d}}\right), \quad \expect{\rconn{\ell}{\hat{\mat{A}}_\mathrm{sym}}} \approx \vect{1}_k^T\mat{\Pi}^{\frac{1}{2}} \hat{\mat{B}}^{2 \ell} \mat{\Pi}^{\frac{1}{2}}\vect{1}_k+ O\left(\frac{1}{\avg{d}}\right),\\
  & \qquad \qquad \expect{\rselfconn{\ell}{\hat{\mat{A}}_\mathrm{sym}}} \approx O\left(\frac{1}{\avg{d}^{\ell}}\right),
\end{aligned}
\end{equation}

where $\hat{\mat{B}}:=\mat{D}^{-\frac{1}{2}}\mat{\Pi}^{\frac{1}{2}}\mat{B}\mat{\Pi}^{\frac{1}{2}}\mat{D}^{-\frac{1}{2}}$, $\vect{1}_k$ is the vector of all ones, and $\mat{D}:= \diag{\mat{B}\vect{\pi}}$ is the diagonal matrix of expected class-wise degrees. See Theorem \ref{theorem:sbm_higher_order_homophily} in \nameref{sec:appendix_A} for an explicit derivation of Eq. \eqref{eq:expected_l_homophily_approx}.

\paragraph{Planted partition SBM.} To illustrate the point, we now consider a specific SBM: a sparse ``planted partition'' SBM with $k$ equi-sized classes such that $\expect{\mat{A}}_{ij}\coloneqq \frac{B_{y_iy_j}}{n}$ where $\mat{B} \coloneqq k d \left[\begin{smallmatrix}h&\dots&\frac{1-h}{k-1}\\\vdots&\ddots&\vdots\\\frac{1-h}{k-1}&\dots&h\end{smallmatrix}\right]$, i.e. with $kd \cdot h$ on the diagonal and $kd\cdot \frac{1-h}{k-1}$ on the off-diagonal, and $d>0$ is the expected mean degree of every node while $0\le h\le 1$ is the expected edge homophily as defined in Eq. \eqref{eq:edge_node_homophily}. Eq. \eqref{eq:expected_l_homophily_approx} yields the expected higher-order homophily:
\begin{align}\label{eq:sbm_homophily_estimates}
    \expect{\rhomo{\ell}{\hat{\mat{A}}_\mathrm{sym}}} \approx \frac{1}{k} + \frac{k-1}{k}\left(\frac{k}{k-1}h-\frac{1}{k-1}\right)^\ell + O \left(\frac{1}{d}\right),
\end{align}
and the expected self-connectivity and total connectivity: $\expect{\rselfconn{\ell}{\hat{\mat{A}}_\mathrm{sym}}} \approx O \left(\frac{1}{d^{\ell}}\right)$ and $\expect{\rconn{\ell}{\hat{\mat{A}}_\mathrm{sym}}} \approx 1 + O \left(\frac{1}{d}\right)$. For a full derivation, see Lemma \ref{lemma:planted_partition_connectivity} in \nameref{sec:appendix_A}. Figure \ref{fig:homophily_analysis} shows, for $k=2$, that our analytic estimates of $\expect{\rhomo{\ell}{\hat{\mat{A}}_\mathrm{sym}}}$ strongly track empirical values. Notable is the symmetric variation of performance with homophily in Eq. \eqref{eq:sbm_homophily_estimates} around ``ambiphily'' ($h=\frac{1}{k}$) for even $\ell$, and specifically when $k=2$ we have equal values for extremely heterophilic ($h=0$) and homophilic ($h=1$) graphs. For $k>2$ this symmetry breaks, but one can still find heterophilic SBMs with very high $2\ell$-order homophily (see Theorem \ref{theorem:optimal_connectivity}). By Eq. \eqref{eq:avg_sensitivity_bound}, for a standard GCN model, signal sensitivity is directly correlated to $2 \ell$-order homophily, so this behaviour explains the phenomenon termed byLuan et al. \cite{luan2024graphneuralnetworkshelp} as the ``mid-homophily pitfall'', where minimal performance is observed at $h=\frac{1}{k}$. 

\begin{figure}[htbp]
\centering

\begin{subfigure}[b]{\textwidth}
\centering
\includegraphics[width=\textwidth]{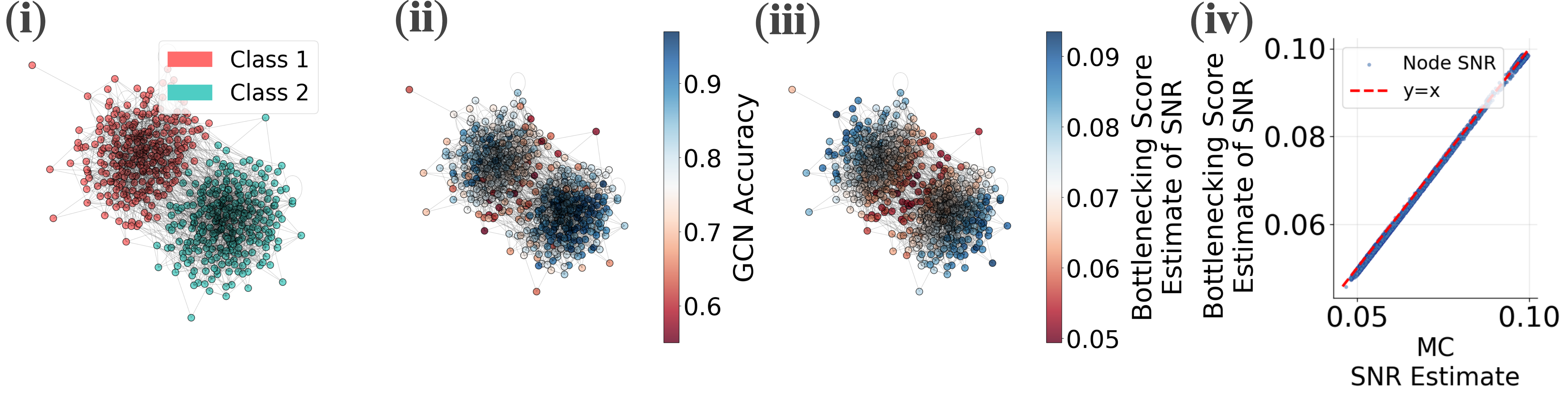}
\caption{Bottlenecking scores correctly estimate the SNR and classification accuracy of a GCN.}
\label{fig:bottleneckingscore_estimate_plot}
\end{subfigure}
\begin{subfigure}[b]{0.49\textwidth}
\centering
\includegraphics[width=\textwidth]{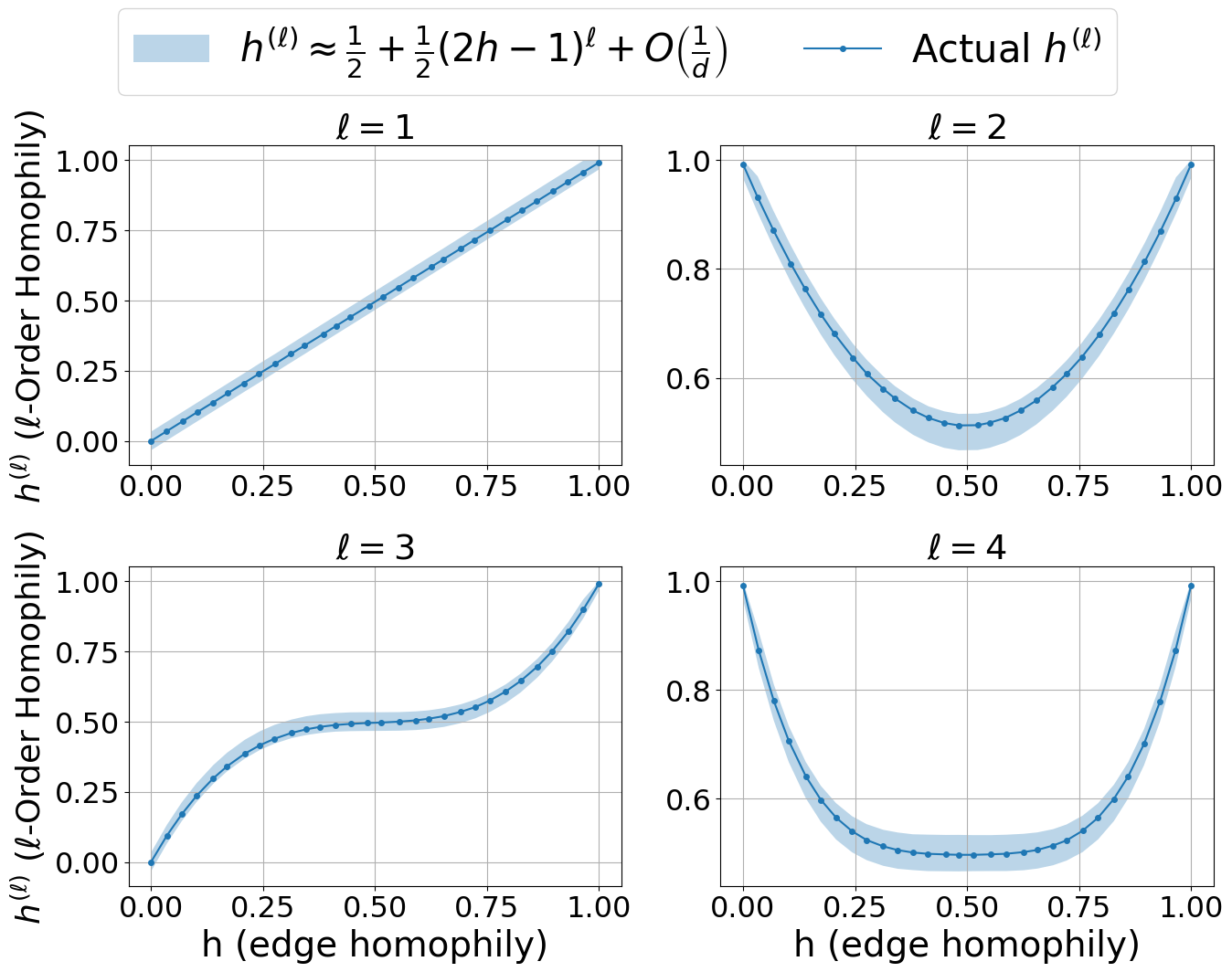}
 \caption{Higher-order homophily as a function of edge homophily in a 2-block SBM.}
\label{fig:approx_higher_order_homophily_h}
\end{subfigure}
\hfill
\begin{subfigure}[b]{0.49\textwidth}
\centering
\includegraphics[width=\textwidth]{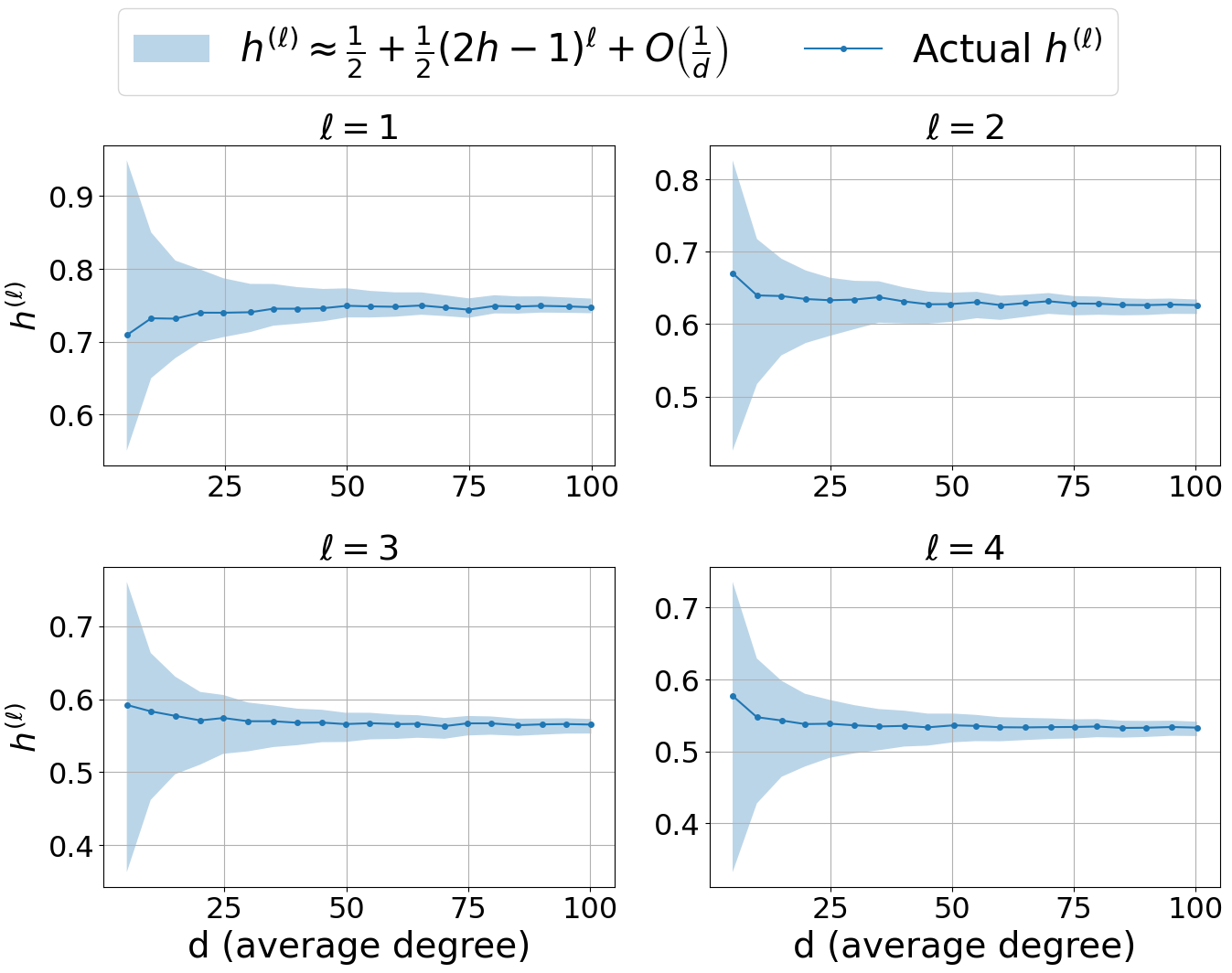}
\caption{Higher-order homophily as a function of mean degree in a 2-block SBM,}
\label{fig:approx_higher_order_homophily_d}
\end{subfigure}
\caption{\textbf{Bottlenecking scores can be used to correctly estimate the SNR and regions of low classification accuracy for a GCN, and can in turn be approximated on average by interpretable graph properties.} 
\textbf{(a)} Bottlenecking scores (Eqs. \eqref{eq:p_order_local_homophily}) can be used to accurately approximate the SNR, and give a faithful, feature-agnostic proxy for the classification accuracy of a GCN. Importantly, we see the phenomenon of class-bottlenecking clearly occurring where nodes across the two classes connect.
(a.ii): nodes are shaded by the empirical accuracy of a 2-layer GCN averaged over 100 training runs; nodes with poor accuracy (red) indicate the difficult-to-classify parts of the graph.
(a.iii): the same graph coloured by the SNR of the GCN---approximated as an SGC using Eq. \eqref{eq:sgc_snr_isotropic}---estimated using the sensitivities from the bottlenecking scores and the upper bounds of Eq. \eqref{eq:individual_signal_sensitivity_mpnn}. The close visual concordance between (a.ii) and (a.iii) shows that class-bottlenecks capture model performance limits purely in terms of the graph structure, validating the hierarchy in Eq. \eqref{eq:framework-hierarchy}. The scatter plot between the Monte Carlo-based estimate of the SNR and the SGC approximation-based SNR shows the accuracy of the SGC approximation, and that the bound in Eq. \eqref{eq:individual_signal_sensitivity_mpnn} is tight.
\textbf{(b, c)} We empirically calculate the $\ell$-order homophily $h^{(\ell)}:=h(\hat{\mat{A}}_{\mathrm{sym}}^{\ell})$, $\ell\in\{1,2,3,4\}$, of graphs with $n=3000$ nodes sampled from a 2-block planted partition SBM, shown in blue markers, and use blue shading to indicate closed-form predictions based on Eq. \eqref{eq:sbm_homophily_estimates}: $\expect{h^{(\ell)}} \approx \frac{1}{2} + \frac{1}{2}(2h-1)^\ell + O\left(\frac{1}{d}\right)$, showing the error term $O\left(\frac{1}{d}\right)$ as the shaded region between $\pm \frac{1}{d}$. \textbf{(b)} Graphs have a fixed average degree $d=30$ but varying edge homophilies $h$, revealing distinct patterns: linear scaling for $\ell=1$, symmetric U-shaped curves for even $\ell$ indicating minimal performance at ambiphily $(h=0.5)$, and asymmetric S-shaped curves for odd $\ell>1$. We note that odd $\ell$ values do not contribute to the signal sensitivity of standard GCN and SGC models (by Eqs. \eqref{eq:mean_signal_sensitivity_sgc} and \eqref{eq:avg_sensitivity_bound}), but do contribute when residual connections are added; see Theorem \ref{theorem:general_signal_sensntivity_bound} in \nameref{sec:appendix_A}. \textbf{(c)} Graphs have a fixed edge homophily $h=0.75$ but varying mean degree $d$, showing the convergence of our approximations for larger $d$.}
\label{fig:homophily_analysis}
\end{figure}

The preceding analysis demonstrates how graph ensembles, particularly the SBM, allow us to derive tractable analytical approximations for higher-order homophily based on a few fundamental graph parameters, like edge homophily and mean degree (Eqs. \eqref{eq:expected_l_homophily_approx}, \eqref{eq:sbm_homophily_estimates}; Figure \ref{fig:homophily_analysis}), thus providing a concrete link between graph structure and the sensitivity bounds established earlier. Equipped with this quantifiable relationship, we naturally arrive at a question of design: what underlying graph connectivity structures are optimal for maximising MPNN performance in a given task? Since higher-order homophily bounds signal sensitivity (Eq. \eqref{eq:individual_signal_sensitivity_mpnn}), and signal sensitivity determines the MPNN's SNR (Theorem \ref{theorem:snrsensitivity}), optimising the expected higher-order homophily should lead to better potential SNR. The following section addresses this by analytically deriving the optimal SBM block connectivity structure(s) that maximise expected higher-order homophily.

\subsection*{SBMs enable a continuous relaxation of optimising over discrete graph structures for message passing}

The analysis so far highlights the interplay between graph structure and MPNN performance: higher-order homophily controls the sensitivity measures and, consequently, the signal-to-noise ratio. To optimise performance, we must consider the ideal graph connectivity structures that optimise these sensitivity measures. By analysing the graph ensemble instead of a given graph instance, we turn an optimisation over discrete graph structures---intractable due to its combinatorial nature---into an optimisation over a continuous graph ensemble---that is analytically solvable using simple linear algebra.

\begin{restatable}[Optimal SBM connectivity]{theorem}{optimalconnectivity}
\label{theorem:optimal_connectivity}
  The general class of SBM connection probability block matrices $\mat{B}\in\mathbb{R}_{\ge 0}^{k\times k}$ that maximise $\trace \left(\hat{\mat{C}}^T \hat{\mat{B}}^{ \ell} \hat{\mat{C}}\right)$, where $\hat{\mat{C}} \in \mathbb{R}^{k\times k}$ is any full rank matrix, and $\hat{\mat{B}}:=\mat{D}^{-\frac{1}{2}}\mat{\Pi}^{\frac{1}{2}}\mat{B}\mat{\Pi}^{\frac{1}{2}}\mat{D}^{-\frac{1}{2}}$, is given by:
  \begin{align}\label{eq:optimal_block_connectivity_matrix}
    \mat{B}= \frac{\langle d\rangle}{k} \mat{\Pi}^{-1} {\mat{P}_k}\mat{\Pi}^{-1},
  \end{align}
  for any symmetric permutation matrix $\mat{P}_k$ if $\ell$ is even, and $\mat{P}_k=\mat{I}_k$ if $\ell$ is odd. Here, $\mat{\Pi}:=\diag{\vect{\pi}}$ is the diagonal matrix of expected class proportions i.e. $\vect{\pi}$ is a size-$k$ simplex vector, $\mat{D}:=\diag{\mat{B}\vect{\pi}}$ is the diagonal matrix of expected class-wise degrees, $\mat{I}_k$ is the identity matrix, and $\avg{d}$ is the mean degree. The optimal value is:
  \begin{align}\label{eq:optimal_block_connectivity_trace_value}
    \max_{\hat{\mat{B}}}\trace \left(\hat{\mat{C}}^T \hat{\mat{B}}^{ \ell} \hat{\mat{C}}\right) = \trace \left(\hat{\mat{C}}^T \hat{\mat{C}}\right).
  \end{align}
\end{restatable}

From Eq. \eqref{eq:expected_l_homophily_approx}, we know that $\expect{\rhomo{2\ell}{\hat{\mat{A}}_\mathrm{sym}}} \approx \trace \left(\mat{\Pi}^{\frac{1}{2}} \hat{\mat{B}}^{2 \ell} \mat{\Pi}^{\frac{1}{2}}\right)$. Setting $\hat{\mat{C}} = \mat{\Pi}^{\frac{1}{2}}$ in Theorem \ref{theorem:optimal_connectivity} reveals that for sparse graph ensembles---with sufficiently large mean degree---the general class of graphs that maximise the expected $2\ell$-order homophily in Eq. \eqref{eq:expected_l_homophily_approx} corresponds to a disjoint union of single-class and two-class-bipartite clusters, where nodes within a class are either connected only amongst themselves or connected only to nodes of another class. We note that this general class of optimal structures includes the trivial fully homophilic case, where $\mat{B}$ is a diagonal matrix, but also non-trivial cases such as the completely heterophilic planted partition model with $h=0$ from Figure \ref{fig:homophily_analysis}, corresponding to $\mat{P}_2=\begin{pmatrix} 0 & 1 \\ 1 & 0 \end{pmatrix}$. For a fixed class assignment and mean degree $\avg{d}$, the size of the set of such optimal matrices $\mat{B}$, is the number of symmetric permutations of $k$ elements, given by the \nth{k} telephone number $T(k)$ which grows hyper-exponentially with $k$ as $T(k) \sim \left(\frac{k}{e}\right)^{k/2} \frac{e^{\sqrt{k}}}{(4e)^{1/4}}$ \cite{knuth1973sorting}. 

Theorem \ref{theorem:optimal_connectivity} provides a clear theoretical characterisation of the optimal graph connectivity structures within the SBM model for maximising higher-order homophily, and motivates a practical approach: modifying graphs to better approximate these optimal structures. Therefore, our concluding contribution is a principled graph rewiring algorithm that provably enhances MPNN performance by explicitly increasing higher-order homophily based on predicted class labels. We now elaborate on this algorithm and present empirical results validating its effectiveness on both synthetic and real-world datasets.

\subsubsection*{BRIDGE: Block Resampling from Inference-Derived Graph Ensembles}

\begin{figure}[htbp]
  \centering
  \begin{subfigure}[t]{\textwidth}
    \centering
    \includegraphics[width=0.9\linewidth]{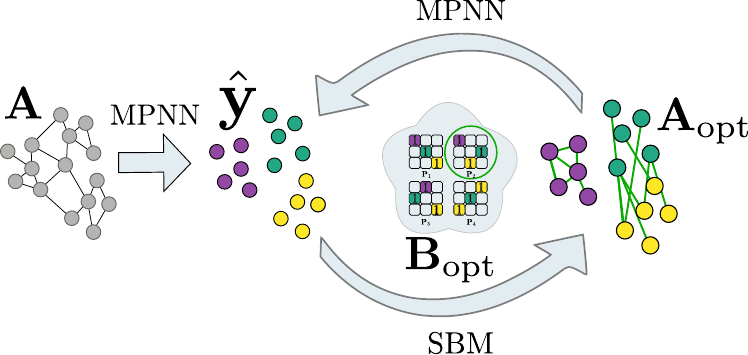}
    \caption{The BRIDGE algorithm uses Theorem \ref{theorem:optimal_connectivity} to enhance message passing by globally maximising higher-order homophily.}
    \label{fig:rewiring-cartoon}
  \end{subfigure}

  \vspace{1.2em}

    \begin{subfigure}[t]{\textwidth}
        \centering
        \includegraphics[width=\linewidth]{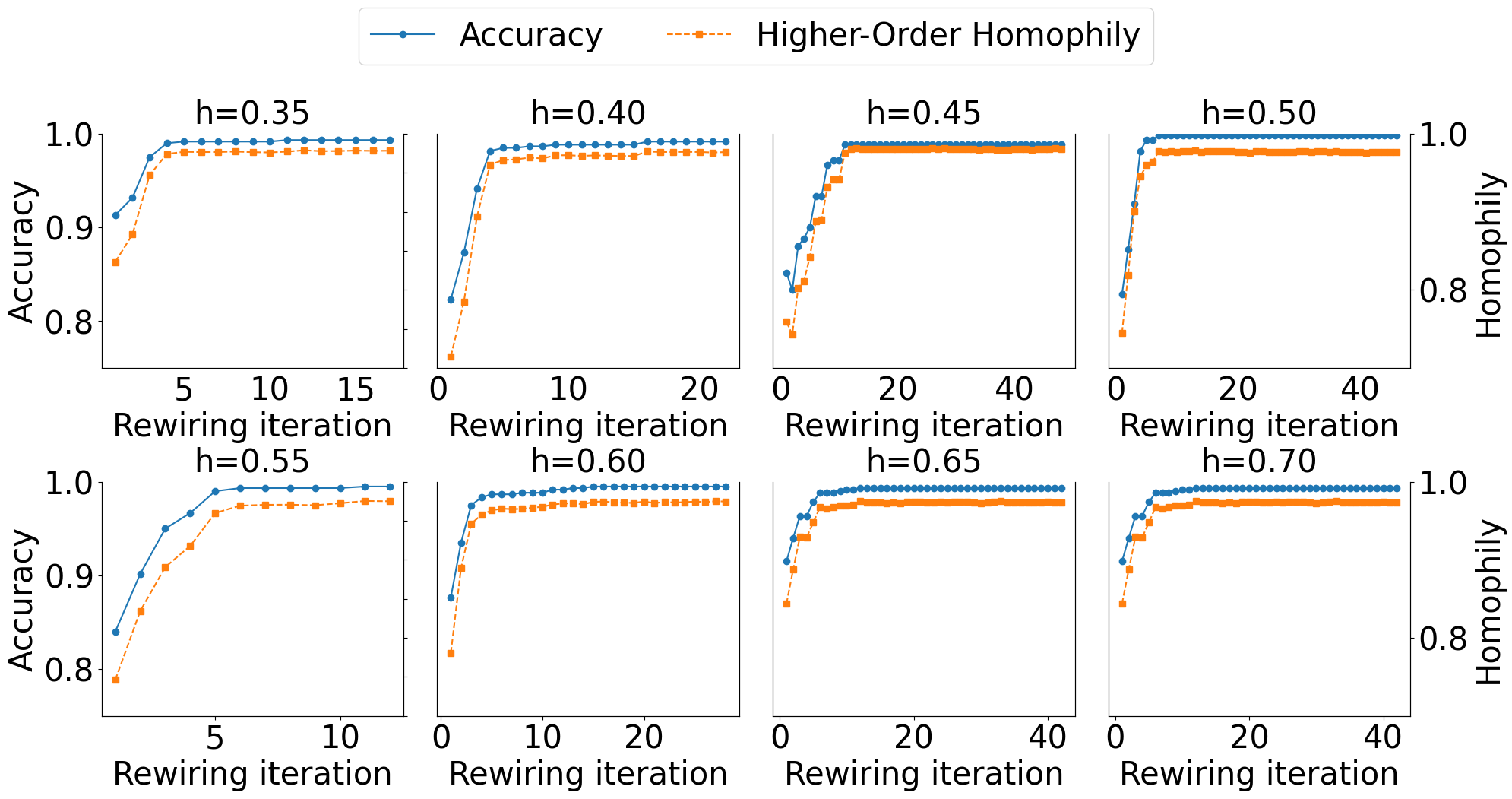}
        \caption{Both higher-order homophily and test accuracy increase until saturation over resampling iterations.}
    \label{fig:rewiring-iterations}
    \end{subfigure}
  \caption{
\textbf{The BRIDGE algorithm improves the higher-order homophily and consequently the classification accuracy of 2-layer GCNs in 2-block planted partition SBMs. (a)} Schematic illustration of the BRIDGE algorithm, which transforms an input graph \(\mathbf{A}\) into an optimised graph \(\mathbf{A}_{\mathrm{opt}}\) by iteratively modifying edges based on predicted class labels and optimal block structures derived from Theorem \ref{theorem:optimal_connectivity}. The central panel shows the block‐matrix structure that guides the rewiring process, with different colours representing different classes.
\textbf{(b)} Across eight 2-block SBM benchmarks, with fixed degree $d=10$ and varying edge homophily \(h\), each iteration of BRIDGE steadily increases both the test accuracy (blue) and the mean higher-order homophily (orange), until the improvements saturate. Notably, graphs across all homophily regimes converge to similarly high performance levels (\(\approx99\%\) accuracy), demonstrating BRIDGE’s ability to overcome structural limitations regardless of the initial graph configuration.}
  \label{fig:rewiring}
\end{figure}

In practice, Theorem \ref{theorem:optimal_connectivity} is most useful when the class membership is not known with complete certainty. If the block structure follows labels $\hat{y}_i$ that are different from the node class labels $y_i$, the optimal connectivity matrix  is still given by Theorem \ref{theorem:optimal_connectivity}---see Theorem \ref{theorem:sbm_higher_order_homophily} in \nameref{sec:appendix_A}---but the optimal higher-order homophily is limited by the accuracy of the predictions. This optimal structure can thus be used to modify the graph's edges based on \emph{predicted} classes, to improve the $\ell$-order homophily of the graph, forming the basis of a graph rewiring scheme which we call Block Resampling from Inference-Derived Graph Ensembles, or BRIDGE.

To obtain the rewired graph, we first use a ``cold-start'' GCN to estimate node-level class predictions, which are then used to compute an optimal block matrix structure according to Theorem \ref{theorem:optimal_connectivity}---as illustrated in Figure \ref{fig:rewiring}a---while treating the choice of the permutation matrix $\mat{P}_k$ as a hyperparameter. We then sample a new graph from this SBM, use the MPNN to predict new classes using the resampled graph, and iterate over this resampling procedure using new class predictions. 

The optimal higher-order homophily achieved for a given set of class predictions in each iteration is given by Eq.~\ref{eq:optimal_block_connectivity_trace_value} in terms of the correlation of predicted and true classes. We elaborate more on the optimum achieved in the \nameref{sec:imperfect_predictions} subsection of the \nameref{sec:methods} section. In this manner, the expected higher-order homophily achieved by the optimal block matrix increases with increased class prediction accuracy at each iteration of the rewiring, and the class prediction accuracy increases with higher higher-order homophily, in a virtuous cycle.

The complete procedure is detailed in the subsection \nameref{sec:methods_BRIDGE} of the \nameref{sec:methods} section.

\paragraph*{BRIDGE achieves near-perfect node classification accuracy in SBMs.}

SBMs are widely used as synthetic benchmarks in the GNN literature for node classification due to their ability to easily control graph structures and their correlation to node classes, and allow for fair model comparisons \cite{abbe2018community,dwivedi2023benchmarking,joshi2023graphworld,abuelhaija2019mixhophigherordergraphconvolutional,garrity2025gnn,duranthon2025statistical}. The experimental results in Table \ref{tab:accuracies_synthetic} demonstrate substantial improvements in GCNs' performance following the application of the BRIDGE algorithm across synthetic 2-block planted partition SBM datasets. BRIDGE achieves near-perfect classification on SBM benchmarks across all homophily regimes.

\paragraph{Baseline behaviour.} The baseline GCN performance exhibits the characteristic ``mid-homophily pitfall'' phenomenon, with accuracies ranging from 85.42\% to 92.05\% across different homophily levels ($h=0.35$ to $h=0.65$), and notably showing minimum performance around the ambiphily region ($h=0.50$ at 85.48\%). This U-shaped performance curve aligns with the theoretical predictions derived from our analysis of higher-order homophily in Eq. \eqref{eq:sbm_homophily_estimates}, where signal sensitivity is bounded by the $2\ell$-order homophily that varies symmetrically around ambiphily ($h=\frac{1}{k}$) according to $h^{(\ell)} \approx \frac{1}{k} + \frac{k-1}{k}(\frac{k}{k-1}h-\frac{1}{k-1})^\ell$. The relatively modest baseline performance, particularly in the mid-homophily regime where same-class connectivity through $2\ell$-hop paths is minimised, suggests that the original graph structures suffer from class-bottlenecks that restrict effective signal propagation. We include the curvature-based Stochastic Discrete Ricci Flow (SDRF) rewiring \cite{topping2021understanding} as well as the random walk heuristic-based Diffusion Improves Graph Learning (DIGL) rewiring \cite{gasteiger2022diffusionimprovesgraphlearning} procedures as literature-standard benchmarks in rewiring, and as a point of reference for our SBM graph datasets. DIGL rewiring leads to performance decreases at all homophily levels. While SDRF offers slight improvements in some cases, its impact is marginal, inconsistent, and not statistically significant---even leading to a performance decrease at $h=0.60$ ($-0.64\%$). Importantly, both procedures fail to mitigate the mid-homophily pitfall, with their performance closely tracking the original GCN baseline. This indicates that current rewiring methods are insufficient to resolve the fundamental structural issues that limit message passing in node classification tasks. 

\paragraph{Effect of BRIDGE resampling.} The impact of the BRIDGE resampling algorithm is evident in the consistently high performance achieved across all homophily levels, with accuracy improvements to approximately 99\% regardless of the initial edge homophily configuration. This dramatic performance boost validates the paper's theoretical framework linking higher-order homophily to MPNN performance limits---by optimally restructuring the graph connectivity to approximate disjoint unions of single-class and two-class-bipartite clusters (as prescribed by Theorem~\ref{theorem:optimal_connectivity}), the rewiring process effectively maximises the class-bottlenecking scores and eliminates the structural impediments to signal sensitivity. The near-perfect accuracy achieved across diverse homophily regimes demonstrates that the BRIDGE algorithm successfully addresses the fundamental architectural limitations of MPNNs by transforming suboptimal graph structures into connectivity patterns that support effective message passing. This result provides compelling empirical evidence for one of the paper's central claims that it is class-correlated graph structures---and more specifically class-bottlenecks---rather than structural bottlenecks alone, that fundamentally determine MPNN performance limits.

\begin{table}[htbp]
\centering
\caption{Mean accuracy before and after rewiring on 2‑block SBM datasets. Blue marks the highest mean accuracy in each row, red the second highest. * means the method’s accuracy differs from the GCN baseline at $p<0.05$.}
\label{tab:accuracies_synthetic}
\makebox[\textwidth]{
\begin{tabular}{lcccc}
\toprule
Dataset & GCN & GCN + DIGL & GCN + SDRF & GCN + BRIDGE \\
\midrule
$h=0.35$ & 91.48 $\pm$ 1.23 & 90.18 $\pm$ 1.41 & {\color{red} 91.57 $\pm$ 1.00} & {\color{blue} 99.27 $\pm$ 0.42*} \\
$h=0.40$ & 88.22 $\pm$ 1.10 & 87.25 $\pm$ 1.93 & {\color{red} 88.73 $\pm$ 1.14} & {\color{blue} 99.33 $\pm$ 0.60*} \\
$h=0.45$ & 86.17 $\pm$ 1.59 & 84.35 $\pm$ 1.55 & {\color{red} 86.70 $\pm$ 1.33} & {\color{blue} 99.05 $\pm$ 0.57*} \\
$h=0.50$ & 85.48 $\pm$ 1.24 & 83.25 $\pm$ 1.47 & {\color{red} 85.80 $\pm$ 1.15} & {\color{blue} 99.50 $\pm$ 0.39*} \\
$h=0.55$ & 85.42 $\pm$ 1.52 & 85.35 $\pm$ 1.69 & {\color{red} 86.28 $\pm$ 0.85} & {\color{blue} 99.55 $\pm$ 0.29*} \\
$h=0.60$ & {\color{red} 88.72 $\pm$ 0.82} & 85.95 $\pm$ 2.41 & 88.08 $\pm$ 1.49 & {\color{blue} 99.37 $\pm$ 0.26*} \\
$h=0.65$ & 92.05 $\pm$ 0.79 & 91.20 $\pm$ 1.91 & {\color{red} 92.40 $\pm$ 0.77} & {\color{blue} 99.23 $\pm$ 0.56*} \\
\bottomrule
\end{tabular}}
\end{table}

\paragraph{Performance on real-world networks: low homophily networks benefit the most.}
We also evaluate BRIDGE on nine widely used citation and web graphs (Table~\ref{tab:accuracies_real}).  On heterophilic or mixed-homophily datasets drawn from the \textsc{WebKB} (\textsc{Texas}, \textsc{Cornell}, \textsc{Wisconsin}), \textsc{Actor}, \textsc{Squirrel}, and \textsc{Chameleon} benchmarks, BRIDGE almost consistently boosts the test accuracies between 2 to 5 percentage points.
For example, the accuracy on \textsc{Actor} climbs from {25.95\%} to {30.79\%}, and on \textsc{Chameleon} from {68.79\%} to {71.49\%}.  These improvements mirror the results for synthetic graphs: in low-homophily regimes, the original connectivity exhibits severe class-bottlenecks that BRIDGE rewiring alleviates. By contrast, the classical citation networks \textsc{Cora}, \textsc{Citeseer}, and \textsc{Pubmed} are strongly homophilic; their original structure is already close to the class-wise single-cluster optimum identified by Theorem~\ref{theorem:optimal_connectivity}.  Rewiring therefore yields negligible change; \textsc{Cora}: $-0.02\%$, \textsc{Citeseer}: $-1.0\%$, \textsc{Pubmed}: $+0.06\%$. 

\begin{table}[htbp]
\centering
\caption{Mean accuracy before and after rewiring on real graph datasets. Blue marks the highest mean accuracy in each row, red the second highest. * means the method’s accuracy differs from the GCN baseline at $p<0.05$. The datasets are respectively divided in the table into three types: Large ($n>1000$ nodes) heterophilic, small ($n\le 1000$ nodes) heterophilic and large homophilic.}
\label{tab:accuracies_real}
\makebox[\textwidth]{
\begin{tabular}{lcccc}
\toprule
Dataset & GCN & GCN + DIGL & GCN + SDRF & GCN + BRIDGE \\
\midrule
\textsc{Actor}      & 25.95 $\pm$ 1.27 & 27.84 $\pm$ 1.38$\ast$ & {\color{red} 30.15 $\pm$ 1.08*} & {\color{blue} 30.79 $\pm$ 1.62*} \\
\textsc{Squirrel}   & {\color{blue} 58.48 $\pm$ 1.91} & 47.53 $\pm$ 1.12 & 51.02 $\pm$ 1.67 & {\color{red} 58.28 $\pm$ 1.25} \\
\textsc{Chameleon}  & 68.79 $\pm$ 2.52 & 61.64 $\pm$ 2.83 & {\color{red} 69.28 $\pm$ 2.45} & {\color{blue} 71.49 $\pm$ 2.52*} \\
\midrule
\textsc{Wisconsin}  & 60.39 $\pm$ 4.11 & 45.10 $\pm$ 5.62 & {\color{blue} 69.41 $\pm$ 5.00*} & {\color{red} 62.16 $\pm$ 5.99} \\
\textsc{Cornell}    & 55.41 $\pm$ 6.27 & 51.62 $\pm$ 6.91 & {\color{red} 58.11 $\pm$ 6.01} & {\color{blue} 58.82 $\pm$ 7.03} \\
\textsc{Texas}      & 62.43 $\pm$ 7.15 & 57.30 $\pm$ 8.90 & {\color{blue} 69.73 $\pm$ 7.19*} & {\color{red} 64.86 $\pm$ 7.56} \\
\midrule
\textsc{Cora}       & {\color{blue} 87.47 $\pm$ 1.25} & 84.93 $\pm$ 1.19 & 86.38 $\pm$ 1.06 & {\color{red} 87.45 $\pm$ 1.25} \\
\textsc{Citeseer}   & {\color{blue} 74.55 $\pm$ 1.57} & 72.36 $\pm$ 1.40 & {\color{red} 74.52 $\pm$ 1.55} & 73.53 $\pm$ 1.57 \\
\textsc{Pubmed}     & 85.11 $\pm$ 0.68 & 84.90 $\pm$ 0.67 & {\color{blue} 85.20 $\pm$ 0.71} & {\color{red} 85.17 $\pm$ 0.64} \\
\bottomrule
\end{tabular}}
\end{table}

It is important to note that BRIDGE maintains and even improves performance on real-world datasets despite completely discarding the original graph structure and reconstructing it from scratch, based only on predicted class labels. Unlike traditional rewiring methods such as SDRF or DIGL that modify existing edges, BRIDGE replaces the entire adjacency matrix with a sampled realisation from an optimal SBM. As a result, some of the potential gains from the increased higher-order homophily might get reduced due to this loss of data. One fruitful extension of this work would be to incorporate priors from the original graph, potentially through more advanced ensemble models like degree corrected or hierarchical SBMs, so as to keep some of the original graph's structure while also improving higher-order homophily. 

\section*{Discussion}

In this paper, we have provided a unified statistical approach to understand how graph structure fundamentally affects the performance of message passing neural networks (MPNNs) in semi-supervised node classification tasks. Our results establish a clear, quantifiable relationship between graph structure, the sensitivity of learned representations, and node classification performance, providing insights that were previously only empirically studied or understood in isolation.

First, by introducing a novel statistical measure of the node-level signal-to-noise ratio (SNR; Eq. \eqref{eq:snr_def}) of an MPNN, we showed in Theorem~
\ref{theorem:snrsensitivity} how the quality of node representations is governed by their sensitivity to class-driven signals versus noisy or global variations in the input, that is, we showed that the SNR decomposes into interpretable measures of signal sensitivity (Eq. \eqref{eq:signal_sensitivity}), and noise and global sensitivities (Eq. \eqref{eq:noise_and_global_sensitivity}). Figure
\ref{fig:snr_analysis} validated this relationship empirically by confirming that our theoretical estimates of the SNR accurately predict actual MPNN performance in terms of node classification accuracy---in particular, the improvement in classification accuracy of MPNNs over feedforward neural networks is directly linked to satisfying the sensitivity condition in Corollary~
\ref{corollary:gdl_criterion}. We also demonstrate these results on real-world graph datasets widely used in the literature, showing that the estimated SNR averaged over the graph strongly correlates with overall test accuracy. Since the SNR estimation is done using only graph-level quantities, our theory can be used in practice on wide-ranging real-world examples to predict model performance before any MPNN training even takes place.

Importantly, the sensitivity condition in Corollary~
\ref{corollary:gdl_criterion} clarifies a previously ambiguous trade-off: low sensitivity to inputs can simultaneously limit expressive power due to oversquashing \cite{digiovanni2023does} while improving generalisation \cite{novak2018sensitivity}. We showed that the critical determinant of improved classification accuracy is not \emph{overall} input sensitivity, but rather the selective enhancement of signal sensitivity relative to noise and global sensitivities. In other words, this distinction resolves the apparent contradiction highlighted in prior works \cite{digiovanni2023does, novak2018sensitivity} by clearly describing when high sensitivity is beneficial or detrimental.

We then introduced higher-order homophily measures in Eq.~\eqref{eq:bottleneck_homophily_connection} that generalise the canonical notion of edge homophily to capture multi-hop interactions between nodes of the same class. Our theoretical analysis in Eq.~\eqref{eq:avg_sensitivity_bound} revealed that an MPNN's signal sensitivity---and hence its discriminative power---is explicitly bounded by higher-order homophily. Low higher-order homophily corresponds directly to the presence of class-bottlenecks, illustrated in Figures~
\ref{fig:homophilic_bottleneck_bad} and  \ref{fig:homophilic_bottleneck_good}, which restrict the ability of MPNNs to effectively propagate class-specific information. Figure \ref{fig:bottleneckingscore_estimate_plot} validated this finding by showing how bottlenecking estimates correctly track MPNN performance in terms of node classification accuracy. In particular, Eqs.~\eqref{eq:mean_signal_sensitivity_sgc} and \eqref{eq:individual_signal_sensitivity_mpnn} explain why MPNNs struggle in graphs with heterophily, consistent with observations made in prior empirical studies \cite{zhu2020homophily, luan2024graphneuralnetworkshelp}. However, we also found that this finding is more nuanced; Figure \ref{fig:approx_higher_order_homophily_h} showed that extremely heterophilous graphs can induce the same levels of higher-order homophily as extremely homophilous graphs, and it is possible for mid-homophily graphs to struggle more than either of those \cite{luan2024graphneuralnetworkshelp}.

To further unpack this relationship, we decomposed the impact of structural bottlenecks, in Eq. \eqref{eq:underover}, into two distinct phenomena: oversquashing and underreaching. Using sparse random graph ensembles, we showed analytically in Theorem
\ref{theorem:underreaching_oversquashing} how the interplay of these two phenomena affects MPNN sensitivities at different message-passing depths. By further specifying a stochastic block model as the graph ensemble, we provided explicit and easily computable expressions for higher-order homophily in Eq.~\eqref{eq:expected_l_homophily_approx}, enabling practitioners to predict the suitability of a graph structure for message-passing models using simple graph properties such as edge homophily and average degree, and systematically diagnose structural limitations in their graphs.

Building on these theoretical insights we developed Block Resampling from Inference-Derived Graph Ensembles, or BRIDGE: a principled graph rewiring algorithm that directly applies our result on graph structures that maximise the expected higher-order homophily---and therefore the MPNN's potential signal sensitivity---from Theorem~\ref{theorem:optimal_connectivity}. BRIDGE iteratively modifies the graph structure to approximate the theoretical optimum of the disjoint union of single-class and two-class-bipartite clusters, thereby maximising the expected higher-order homophily. Our experimental results on synthetic planted partition SBM datasets demonstrate the impact of this approach: while baseline GCN performance exhibits the characteristic ``mid-homophily pitfall'' \cite{luan2024graphneuralnetworkshelp}---with accuracies ranging from 85.42\% to 92.05\% across different homophily levels---and other rewiring methods \cite{topping2021understanding, gasteiger2022diffusionimprovesgraphlearning} offering only marginal gains, BRIDGE-rewired graphs achieve near-perfect classification accuracy of 99\% regardless of the original graph's edge homophily; see Table~\ref{tab:accuracies_synthetic}. This dramatic improvement across all homophily regimes validates one of our central claims that it is class-correlated graph structures---and more specifically \emph{class} bottlenecks---rather than structural bottlenecks alone that fundamentally determine an MPNN's performance limits.

Importantly, applying BRIDGE to real-world graphs also consistently improved performance in heterophilic or mixed-homophily datasets such as \textsc{Actor}, \textsc{Chameleon}, \textsc{Wisconsin}, \textsc{Cornell}, and \textsc{Texas}; see Table \ref{tab:accuracies_real}. For instance, the classification accuracy on \textsc{Actor} increased from 25.95\% to 30.79\%, and on \textsc{Chameleon} from 68.79\% to 71.49\%. These improvements show that---even as a simple demonstration with coarse block-level resampling---BRIDGE has the ability to address the problem of class-bottlenecks prevalent in real-world graphs.

\paragraph{Limitations and future work.} Despite the theoretical insights and empirical successes demonstrated in this work, several limitations warrant consideration. 
First, our theoretical framework relies on specific assumptions about feature distributions in Eq.~\eqref{eq:featuredecomp} and graph sparsity conditions in Theorem~\ref{theorem:underreaching_oversquashing}, which---while standard in the GNN literature---may not always hold in practice, but allow us to make significant analytical progress. The feature decomposition into class-wise signal, node-level noise, and global shift components is broadly applicable, but may oversimplify the complex feature structures present in rich domains, like molecular graphs or knowledge graphs. 
Second, while BRIDGE achieves substantial performance gains on synthetic SBM datasets, as well as significant improvements in real-world heterophilic graph datasets, on strongly homophilic citation networks (like \textsc{Cora}, \textsc{Citeseer}, and \textsc{Pubmed}) BRIDGE maintains baseline performance on the original structures, as the original structure is already close to the single-class clusters optimum and any potential gains from the increased higher-order homophily are reduced by the discarding of original graph data. Using more complex graph ensembles that incorporate priors from the original graph, such as using degree corrected or hierarchical SBMs \cite{Karrer_2011,PhysRevX.4.011047}, would allow retaining some of the graphs' original information while also improving their expected higher-order homophily.

\section{Methods}\label{sec:methods}

\subsection*{Feature distribution}
We model the distribution of node features $\mat{X}$ in relation to their class labels $\vect{y}$. To avoid assuming a specific---potentially restrictive---feature distribution while still allowing for structured analysis, we pursue a feature decomposition by expressing the feature vector of node $j$, denoted by $\mat{X}_j$, through three independently sampled components:
\begin{align*}
  \mat{X}_j = \vect{\mu}_{y_j} + \vect{\gamma} + \vect{\epsilon}_j.
\end{align*}
Here, the vector $\vect{\mu}_{y_j}\in \mathbb{R}^{d_{\mathrm{in}}}$ captures the class-specific mean signal: $\expect{\vect{\mu}_{c}}:=\condexpect{\mat{X}_j}{y_j=c}$. The vector $\vect{\gamma} \in \mathbb{R}^{d_{\mathrm{in}}}$ represents zero-mean global variations shared across all nodes. Finally, $\vect{\epsilon}_j \in \mathbb{R}^{d_{\mathrm{in}}}$ are node-wise IID zero-mean vectors representing unstructured noise.

We assume the following feature covariance structure for each of these components: The class-wise signal covariance is $\Sigma_{qr} := \mathrm{Cov}(\mu_{y_j,q}, \mu_{y_j,r})$. The global shift covariance is $\Phi_{qr}:= \mathrm{Cov}(\gamma_{q}, \gamma_{r})$, and the noise covariance is $\Psi_{qr}:= \mathrm{Cov}(\epsilon_{jq}, \epsilon_{kr})$. All covariance matrices $\mat{\Sigma}, \mat{\Phi}, \mat{\Psi}$ are $d_{\mathrm{in}}\times d_{\mathrm{in}}$ semi-positive definite symmetric matrices.
This decomposition allows us to separate the class-discriminative signal from non-discriminative noisy and global shifts.

While these underlying components---mean vectors and covariances---are useful for theoretical modelling, they are often not directly observable or easily estimable, especially with high-dimensional or complex features. Therefore, our analysis focuses on model-specific quantities that capture how an MPNN responds to the features, rather than requiring explicit estimation of these feature parameters.

\subsection*{Quantifying MPNN's sensitivty to inputs: signal, noise, and global sensitivity}
To understand how an MPNN processes input features, we introduce three sensitivity measures that quantify the model's responsiveness to different input components in Eq.~\eqref{eq:featuredecomp}, independent of the specific feature values $\mat{X}$. Let $\mat{H}_i^{(\ell)}$ denote the representation of node $i$ at layer $\ell$.
\begin{itemize}
  \item \textbf{Signal sensitivity} $\sigsenexp[i]{\ell}{p,q,r}$ measures the responsiveness of the $\nth{p}$ output feature $H_{ip}^{(\ell)}$ to coherent changes in the $\nth{q}$ and $\nth{r}$ input features $X_{jq}, X_{kr}$ of nodes $j\ne k$ belonging to the \emph{same class}, i.e. $y_j=y_k$. It captures the model's ability to process class-specific information.
  \begin{align*}
    \sigsenexp[i]{\ell}{p,q,r} &:= \sum_{j, k \in V} \frac{\partial H_{ip}^{(\ell)}}{\partial X_{jq}} \frac{\partial H_{ip}^{(\ell)}}{\partial X_{kr}} \bigg|_{\mat{X}=\mat{0}}\delta_{y_j y_k}.
  \end{align*}
  \item \textbf{Noise sensitivity} $\noisesenexp[i]{\ell}{p,q,r}$ measures the responsiveness to changes in the $\nth{q}$ and $\nth{r}$ features of the input node $j$. It quantifies sensitivity to unstructured, node-specific variations.
  \begin{align*}
    \noisesenexp[i]{\ell}{p,q,r}&:= \sum_{j \in V} \frac{\partial H_{ip}^{(\ell)}}{\partial X_{jq}}\frac{\partial H_{ip}^{(\ell)}}{\partial X_{jr}}\bigg|_{\mat{X}=\mat{0}}.
  \end{align*}
  \item \textbf{Global sensitivity} $\globsenexp[i]{\ell}{p,q,r}$ measures the responsiveness to changes in the $\nth{q}$ and $\nth{r}$ features across \emph{all pairs} of input nodes $j, k$, regardless of class labels. It reflects the overall sensitivity to any input perturbation, including global shifts.
  \begin{align*}
     \globsenexp[i]{\ell}{p,q,r}&:= \sum_{j, k \in V} \frac{\partial H_{ip}^{(\ell)}}{\partial X_{jq}} \frac{\partial H_{ip}^{(\ell)}}{\partial X_{kr}}\bigg|_{\mat{X}=\mat{0}}.
  \end{align*}
\end{itemize}
These sensitivities, derived from the model's Jacobian, allow us to analyse the MPNN's behaviour without needing access to the underlying feature generation process and form the basis for understanding the SNR of the learned representations.

\subsection{Experimental setup for SNR analysis}\label{sec:experimental_details}
To empirically validate the analytic relationship between sensitivities and the SNR in Theorem~\ref{theorem:snrsensitivity}, and the analytic sensitivity condition for MPNNs outperforming FNNs in Corollary~\ref{corollary:gdl_criterion}, we conducted experiments using synthetic data whose results are shown in Figure~\ref{fig:snr_analysis}.

\paragraph{Graph generation.} We generated synthetic graphs using the 2-block planted partition SBM with $n=500$ nodes. We varied the edge homophily $h$ from 0 to 1 to create graphs ranging from purely heterophilic to purely homophilic, while the average degree was fixed at $\avg{d}=10$. 100 graphs were sampled for every configuration.

\paragraph{Feature sampling.} Node features were sampled according to Eq.~\eqref{eq:featuredecomp} with $d_{\mathrm{in}}=5$ feature dimensions. Components were drawn independently from zero-mean Gaussian distributions with diagonal covariance matrices: $\mat{\Sigma}=10^{-5} \eye{5}$, $\mat{\Psi}=10^{-4}\eye{5}$, and $\mat{\Phi}=10^{-4}\eye{5}$.

\paragraph{Models and training.} We compared two-layer GCN using the standard symmetric normalised adjacency matrix $\hat{\mat{A}}_{\mathrm{sym}}$ against a single-layer linear FNN as a baseline. Both models were trained for 100 epochs using the Adam optimiser \cite{kingma2014adam} with a learning rate of 0.01 and L2 weight decay of $5 \times 10^{-4}$. For each generated graph we performed 100 training runs to estimate the average test accuracy and SNR at the node-level.
\paragraph{Empirical SNR estimation.}
To estimate the empirical SNR for Figure~\ref{fig:snr_vs_acc}, as defined in Eq.~\eqref{eq:snr_def}, we employed a Monte Carlo approach. 
First, we generated $N_\mu = 300$ sets of class mean vectors $\left\{\vect{\mu}_{c}^{(m)}\right\}_{c\in[k]}$ for $m\in[N_\mu]$, and $N_{\gamma\epsilon} = 300$ sets of noise and global shift vectors $\left\{\vect{\gamma}^{(s)}, \left\{\vect{\epsilon}_j^{(s)}\right\}_{j\in[n]}\right\}$ for $s\in[N_{\gamma\epsilon}]$. This procedure resulted in $N_\mu \times N_{\gamma\epsilon}$ distinct feature matrices $\mat{X}^{(m,s)}$. 
We trained a single GCN on the first feature matrix sample $\mat{X}^{(1,1)}$, and used it to obtain the corresponding output representations $\left[\mat{H}^{(\ell)}\right]^{(m,s)}$.
We estimated the conditional expectation $\condexpectsub{H_{ip}^{(\ell)}}{\vect{\mu}^{(m)}}{\vect{\gamma}, \vect{\epsilon}}$ by averaging the output representations over the noise and global shifts:
\begin{align}\label{eq:cond_exp_est}
\widehat{\mathbb{E}}\left[H_{ip}^{(\ell)}\,\middle|\,\vect{\mu}^{(m)}\right] = \frac{1}{N_{\gamma\epsilon}} \sum_{s=1}^{N_{\gamma\epsilon}} \left[H_{ip}^{(\ell)}\right]^{(m,s)}.
\end{align}
The numerator of the SNR, i.e. the inter-class variance or the ``signal'' $\varsub{\condexpectsub{H_{ip}^{(\ell)}}{\vect{\mu}}{\vect{\gamma}, \vect{\epsilon}}}{\vect{\mu}}$, was estimated using the sample variance of these estimated conditional expectations:
\begin{align}\label{eq:signal_var_est}
\widehat{\mathrm{Var}}_{\vect{\mu}}\left(\condexpectsub{H_{ip}^{(\ell)}}{\vect{\mu}}{\vect{\gamma}, \vect{\epsilon}}\right) = \frac{1}{N_\mu-1} \sum_{m=1}^{N_\mu} \left(\widehat{\mathbb{E}}_{\vect{\gamma}, \vect{\epsilon}}\left[H_{ip}^{(\ell)}\,\middle|\,\vect{\mu}^{(m)}\right] - \widehat{H}_{ip}^{(\ell)}\right)^2,
\end{align}
where $\widehat{H}_{ip}^{(\ell)}$ is the mean of the estimated conditional expectations: $\widehat{H}_{ip}^{(\ell)} = \frac{1}{N_\mu} \sum_{m=1}^{N_\mu} \widehat{\mathbb{E}}_{\vect{\gamma}, \vect{\epsilon}}\left[H_{ip}^{(\ell)}\,\middle|\,\vect{\mu}^{(m)}\right].$
Similarly, the conditional variances $\varsub{H_{ip}^{(\ell)}\,\middle|\,\vect{\mu}^{(m)}}{\vect{\gamma}, \vect{\epsilon}}$ were estimated by calculating the sample variance of the representations over the noise and global shifts:
\begin{align}\label{eq:noise_var_est}
\widehat{\mathrm{Var}}_{\vect{\gamma}, \vect{\epsilon}}\left[H_{ip}^{(\ell)}\,\middle|\,\vect{\mu}^{(m)}\right] = \frac{1}{N_{\gamma\epsilon}-1} \sum_{s=1}^{N_{\gamma\epsilon}} \left(\left[H_{ip}^{(\ell)}\right]^{(m,s)} - \widehat{\mathbb{E}}_{\vect{\gamma}, \vect{\epsilon}}\left[H_{ip}^{(\ell)}\,\middle|\,\vect{\mu}^{(m)}\right]\right)^2.
\end{align}
The denominator of the SNR, i.e. the intra-class variance or ``noise'' $\expectsub{\varsub{H_{ip}^{(\ell)}\mid \vect{\mu}}{\vect{\gamma}, \vect{\epsilon}}}{\vect{\mu}}$, was estimated by averaging the conditional variance estimates from Eq.~\eqref{eq:noise_var_est} over the class means:
\begin{align}\label{eq:avg_noise_var_est}
\widehat{\mathbb{E}}_{\vect{\mu}}\left[\mathrm{Var}_{\vect{\gamma}, \vect{\epsilon}}\left(H_{ip}^{(\ell)}\,\middle|\,\vect{\mu}\right)\right] = \frac{1}{N_\mu} \sum_{m=1}^{N_\mu} \widehat{\mathrm{Var}}_{\vect{\gamma}, \vect{\epsilon}}\left(H_{ip}^{(\ell)}\,\middle|\,\vect{\mu}^{(m)}\right).
\end{align}
Finally, the empirical SNR was estimated as the ratio of the estimated numerator (Eq.~\eqref{eq:signal_var_est}) and denominator (Eq.~\eqref{eq:avg_noise_var_est}):
\begin{align}\label{eq:final_snr_est}
\widehat{\mathrm{SNR}} = \frac{\widehat{\mathrm{Var}}_{\vect{\mu}}\left(\mathbb{E}_{\vect{\gamma}, \vect{\epsilon}}\left[H_{ip}^{(\ell)}\,\middle|\,\vect{\mu}\right] \right)}{\widehat{\mathbb{E}}_{\vect{\mu}}\left[\mathrm{Var}_{\vect{\gamma}, \vect{\epsilon}}\left(H_{ip}^{(\ell)}\,\middle|\,\vect{\mu}\right)\right]}.
\end{align}

It should be noted that using the ratio of the estimators as an estimator of the ratio generally yields a biased estimator. However, with a large enough sample size, the bias scales as $O\left(\frac{1}{N_{\mu}}\right)$, and is therefore negligible for the purposes of this paper's methods. 

This Monte Carlo estimate, calculated using Eq.~\eqref{eq:final_snr_est} for a given sampled graph, was recalculated over $100$ sampled graphs (along with node-level features) to obtain the expected SNR and $95\%$ confidence intervals, which were then compared against the theoretical approximations derived from sensitivities in Theorem~\ref{theorem:snrsensitivity}.

\paragraph{Empirical sensitivity estimation.} To compute the theoretical SNR approximation and check whether the sensitivity condition in Corollary~\ref{corollary:gdl_criterion} is satisfied, we calculated the signal, noise, and global sensitivities. This required computing the Jacobian of the GCN's output $\mat{H}^{(\ell)}$ with respect to the input features $\mat{X}$ using PyTorch's automatic differentiation \cite{paszke2017automatic}. The computed Jacobians $\frac{\partial H_{ip}^{(\ell)}}{\partial X_{jq}}$ were then used in the definitions in Eqs.~\eqref{eq:signal_sensitivity} and \eqref{eq:noise_and_global_sensitivity} to obtain the sensitivity values for each node $i$ and output dimension $p$.

\subsection{Estimating higher-order homophily using imperfect class predictions}\label{sec:imperfect_predictions}

Applying the theoretical insights from our results would ideally require knowledge of the node class labels to estimate homophily. However, the true class labels are often unknown or partially observed, so we rely on predicted class labels obtained from a trained model. Inevitably, these predictions will contain errors, which means the estimated higher-order homophily will deviate from the ideal scenario that assumes perfectly known labels.

To handle misclassifications, we introduce a confusion matrix $\mat{C}\in\mathbb{R}^{k\times k}$ that captures the discrepancies between the true class labels $\vect{y}$ and the predicted labels $\hat{\vect{y}}$. Specifically, for a graph with $n$ nodes and $k$ classes, the entries of $\mat{C}$ are defined by
\begin{align}\label{eq:confusion_mat_def}
  C_{uv} \;:=\; \frac{1}{n}\sum_{i=1}^n \delta_{\hat{y}_i u}\,\delta_{y_i v},
\end{align}
where $\delta$ is the Kronecker delta function. The matrix $\mat{C}$ aggregates the fraction of nodes that are predicted as class $u$ but belong to class $v$. In the ideal case of perfect classification, $\mat{C}$ would be diagonal.

Even in the presence of errors, the key theoretical insight about optimal connectivity structures remains unchanged. The derivation of higher-order homophily using the block matrix $\mat{B}$ and the mean class degree matrix $\mat{D}$ still holds, except now we replace the unknown true labels $\vect{y}$ by the predicted labels $\hat{\vect{y}}$. In other words, when computing the optimal $\mat{B}$, we use the estimated class memberships to form the probabilities $\hat{\pi}_v$ of each predicted class $v$, so that the diagonal matrix $\mat{\Pi} = \mathrm{diag}(\hat{\vect{\pi}})$ and the associated expected adjacency $\mathbb{E}[\mat{A}]$ in the SBM formulation are constructed from predicted labels. Theorem \ref{theorem:sbm_higher_order_homophily} extends the SBM estimates of higher-order homophily using predicted labels, giving:

\begin{align}\label{eq:imperfect_labels_sbm_homophily}
  \expect{\rhomo{2\ell}{\hat{\mat{A}}_\mathrm{sym}}} \approx \trace \left(\mat{C}^T\mat{\Pi}^{-\frac{1}{2}}\hat{\mat{B}}^{2\ell}\mat{\Pi}^{-\frac{1}{2}} \mat{C}\right) + O\left(\frac{1}{\avg{d}}\right)
\end{align}

where $\hat{\mat{B}} = \mat{D}^{-\frac{1}{2}}\mat{\Pi}^{\frac{1}{2}}\mat{B}\mat{\Pi}^{\frac{1}{2}}\mat{D}^{-\frac{1}{2}}$ is a normalised version of the block matrix.

Eq. \eqref{eq:imperfect_labels_sbm_homophily} has the same form as Eq. \eqref{eq:optimal_block_connectivity_trace_value} in Theorem \ref{theorem:optimal_connectivity}, applied when $\hat{\mat{C}} :=\mat{\Pi}^{-\frac{1}{2}}\mat{C}$. Thus, the formula for the optimal block matrix $\mat{B}$ retains exactly the same form as in the case of perfectly known labels---a disjoint union of single-class and two-class-bipartite clusters, except now the classes are taken to be the \emph{predicted} classes. What does change is the optimal higher-order homophily achieved: Eq. \eqref{eq:optimal_block_connectivity_trace_value} states that the optimal higher-order homophily for a given set of predicted labels is controlled by the correlation between the true and predicted labels, as $\max_{\hat{\mat{B}}}\trace \left(\hat{\mat{C}}^T \hat{\mat{B}}^{ \ell} \hat{\mat{C}}\right) = \trace \left(\hat{\mat{C}}^T \hat{\mat{C}}\right) = \trace \left(\mat{C}^T \mat{\Pi}^{-1} \mat{C}\right)$, which is higher for more diagonal confusion matrices $\mat{C}$. Therefore, more accurate partitions of predicted classes---which are more closely correlated with the true classes---result in larger optimal higher-order homophily.

In this way, the theoretical framework can be applied to real-world settings with imperfect class label information, enabling practitioners to estimate, rewire, and optimise for higher-order homophily based on model-inferred labels, as we demonstrate in the following subsection.

\subsection{BRIDGE: Block Resampling from Inference-Derived Graph Ensembles}\label{sec:methods_BRIDGE}

Theorem~\ref{theorem:optimal_connectivity} shows that, for a fixed class assignment, the SBM graph that maximises higher-order homophily is a union of single-class and two-class bipartite clusters. BRIDGE resamples a new graph  so that its connectivity approximates this optimal pattern, even when the true class labels are unknown.

\paragraph{Overview.}
BRIDGE alternates between two steps:
\begin{enumerate}
    \item Class-prediction: 
   Use a GCN on the current graph $G^{(m)}$---initially trained on the original graph $G^{(0)}:=G$ and then retrained once more on the first iteration's sampled graph $G^{(1)}$ --- to infer predicted classes $\hat{\vect{y}}^{(m)}$ at iteration $m$, which give noisy estimates of the true classes.
\item Resampling:
   Use $\hat{\vect{y}}^{(m)}$ to build the optimal block-probability matrix  
   \[
     \mat B_{\mathrm{opt}}
     \;=\;
     \frac{\langle d\rangle}{k}\,
     \mat\Pi^{-1}\,\mat P_k\,\mat\Pi^{-1},
   \]
   where $\langle d\rangle$ is a target mean degree,  
   $\mat\Pi=\operatorname{diag}(\hat\pi_1,\dots,\hat\pi_k)$ holds the predicted class proportions,  
   and $\mat P_k$ is a symmetric permutation matrix (treated as a hyperparameter). Sample a new adjacency matrix
   \[
     \bigl[\mat A_{\mathrm{opt}}^{(m+1)}\bigr]_{ij}
     \sim
     \operatorname{Bernoulli}\!\left(
       \frac{1}{n}\,
       \bigl[\mat B_{\mathrm{opt}}\bigr]_{\hat y^{(m)}_i\,\hat y^{(m)}_j}
     \right)
   \]
   to obtain the corresponding new graph $G^{(m+1)}$.
\end{enumerate}

The procedure stops after a preset number of iterations $M$.  
Because better class predictions raise the optimal higher-order homophily in \eqref{eq:optimal_block_connectivity_trace_value}, and higher-order homophily in turn improves predictions, these two steps form a positive feedback loop.

\paragraph{Hyperparameters.}
In addition to the standard GCN hyperparameters, we search over  
(i) the permutation matrix $\mat P_k$ (ordered by expected edge homophily),  
(ii) the target mean degree $\langle d\rangle$, and  
(iii) the number of BRIDGE iterations $M$.

\paragraph*{Experimental setup.}
We implement this hyperparameter search automatically using Optuna \cite{optuna_2019}, with 100 trials. The optimal hyperparameters for the baseline GCN and BRIDGE, along with benchmark SDRF and DIGL rewiring methods  are presented in \nameref{sec:appendix_C}, with the baseline GCN's hyperparameters given in Tables \ref{tab:hp_base_synth} and \ref{tab:hp_base_real}, the BRIDGE hyperparameters presented in Tables \ref{tab:hp_rewire_synth} and \ref{tab:hp_rewire_real}, the SDRF hyperparameters in Tables \ref{tab:hp_sdrf_synth} and \ref{tab:hp_sdrf_real}, and the DIGL hyperparameters in Tables \ref{tab:hp_digl_synth} and \ref{tab:hp_digl_real}. We report the choice of permutation matrix hyperparameter $\mat{P}_k$ in cycle notation which writes a permutation as a list of parentheses, each showing elements sent to the next in order until the first reappears. For the DIGL rewiring method we used the personalised PageRank diffusion.  The mean accuracy score is calculated over 10 random 60\%/20\%/20\% train/test/validation splits. The synthetic datasets are sampled from a planted partition SBM, with $2$ equal sized classes, expected mean degree of $\avg{d}=10$, and varying expected edge homophily from $h=0.35$ to $h=0.65$ to get a full range of accuracies (outside of this interval accuracies saturate at $100\%$).

\paragraph*{Implementation details.}

All experiments are implemented using the Deep Graph Library package \cite{wang2020deepgraphlibrarygraphcentric} and conducted on the Imperial College London HPC \cite{imperial_hpc} with NVIDIA A100 GPUs. Code for reproducing the experiments is available at \href{https://github.com/jr419/BRIDGE}{\texttt{https://github.com/jr419/BRIDGE}}.

\section{Acknowledgements}
J.R. is supported by the UKRI CDT in AI for Healthcare \href{http://
ai4health.io}{\texttt{http://ai4health.io}} [EP/S023283/1]. 

\newpage
\bibliographystyle{naturemag}
{\small
\bibliography{reference}
}
\newglossaryentry{G}{name=\ensuremath{G}, 
 description={Graph $G = (V, E)$ with node set $V:=[n]$ consisting of $n$ nodes and possibly directed edge set $E:=\{(i,j)\in V^2\,:\,i\text{ and }j\text{ are directly connected}\}$}, 
 sort=1}

\newglossaryentry{MPNN}{name=MPNN, 
 description={Message Passing Neural Network; a neural network architecture that aggregates and propagates information along edges of a graph}, 
 sort=2}

\newglossaryentry{FNN}{name=FNN, 
 description={Feedforward Neural Network; a neural network architecture that updates information purely based on the nodes own features}, 
 sort=3}

\newglossaryentry{GCN}{name=GCN, 
 description={Graph Convolutional Network; a type of MPNN that applies convolution-like operations to aggregate information on graphs}, 
 sort=4}

\newglossaryentry{SGC}{name=SGC, 
 description={Simple Graph Convolution; a simple type of GCN that uses linear aggregation}, 
 sort=5}

\newglossaryentry{homophily}{name=Homophily, 
 description={The tendency of nodes to connect to others with similar attributes (e.g., with same class label)}, 
 sort=6}

\newglossaryentry{oversquashing}{name=Oversquashing, 
 description={A phenomenon where information from many nodes is compressed into a fixed-size vector, thereby ``squashing'' the signal}, 
 sort=7}

\newglossaryentry{underreaching}{name=Underreaching, 
 description={A phenomenon where information from distant nodes fails to reach a target node due to few short-distance paths}, 
 sort=8}

\newglossaryentry{Jacobian}{name=Jacobian, 
 description={The matrix of first-order partial derivatives of a vector-valued function; used to measure local sensitivity}, 
 sort=9}

\newglossaryentry{yi}{name=\ensuremath{y_i}, 
 description={Class label of node~\(\,i\)}, 
 sort=12}

\newglossaryentry{A}{name=\ensuremath{\mat{A}}, 
 description={Adjacency matrix of the graph}, 
 sort=13}

\newglossaryentry{Ahat}{name=\ensuremath{\hat{\mat{A}}}, 
 description={Graph shift operator, or normalised adjacency matrix}, 
 sort=14}

\newglossaryentry{D}{name=\ensuremath{\mat{D}}, 
 description={Diagonal degree matrix}, 
 sort=15}

\newglossaryentry{X}{name=\ensuremath{\mat{X}}, 
 description={Matrix of node features}, 
 sort=16}

\newglossaryentry{Hi}{name=\ensuremath{\mat{H}_i^{(\ell)}}, 
 description={Representation (embedding) of node~\(\,i\) at layer~\(\ell\)}, 
 sort=17}

\newglossaryentry{din}{name=\ensuremath{d_\mathrm{in}}, 
 description={Dimension of input features}, 
 sort=18}

\newglossaryentry{dout}{name=\ensuremath{d_\mathrm{out}}, 
 description={Dimension of output features}, 
 sort=19}

\newglossaryentry{mu}{name=\ensuremath{\vect{\mu}_{c}}, 
 description={Mean (class-specific) signal vector for class~\(c\)}, 
 sort=20}

\newglossaryentry{gamma}{name=\ensuremath{\gamma}, 
 description={Global shift or mean of node features}, 
 sort=21}

\newglossaryentry{epsilon}{name=\ensuremath{\vect{\epsilon}_{j}}, 
 description={IID noise vector for node~\(\,j\)}, 
 sort=22}

\newglossaryentry{sigma}{name=\ensuremath{\mat{\Sigma}}, 
 description={Covariance matrix of class-specific signals in the features}, 
 sort=23}

\newglossaryentry{phi}{name=\ensuremath{\mat{\Phi}}, 
 description={Covariance matrix of global shift in features}, 
 sort=24}

\newglossaryentry{psi}{name=\ensuremath{\mat{\Psi}}, 
 description={Covariance matrix of noise in features}, 
 sort=25}

\newglossaryentry{rho}{name=\ensuremath{\rho}, 
 description={Local noise proportion---a parameter combining all variance components for IID feature dimensions, characterising the baseline difficulty of classifying feature sets for MPNNs}, 
 sort=26}

\newglossaryentry{Ssig}{name=\ensuremath{\sigsenexp[i]{\ell}{p,q,r}}, 
 description={Signal sensitivity of node~\(\,i\) at layer~\(\ell\) for output dimension~\(p\) with respect to input dimensions~\(q,r\); measures response to coherent class-specific changes}, 
 sort=27}

\newglossaryentry{Snoise}{name=\ensuremath{\noisesenexp[i]{\ell}{p,q,r}}, 
 description={Noise sensitivity of node~\(\,i\) at layer~\(\ell\) for output dimension~\(p\) with respect to input dimensions~\(q,r\); measures response to unstructured, local, IID noise}, 
 sort=28}

\newglossaryentry{Sglob}{name=\ensuremath{\globsenexp[i]{\ell}{p,q,r}}, 
 description={Global sensitivity of node~\(\,i\) at layer~\(\ell\) for output dimension~\(p\) with respect to input dimensions~\(q,r\); measures response to global shifts in the input}, 
 sort=29}

\newglossaryentry{SNR}{name=\ensuremath{\snrsub{H_{ip}^{(\ell)}}{}}, 
 description={Signal-to-noise ratio of the representation of node~\(\,i\) for output dimension~\(p\) at layer~\(\ell\)}, 
 sort=30}


\newglossaryentry{rhomor}{name=\ensuremath{\rhomo{r}{\hat{\mat{A}}}}, 
 description={higher- or \(r\)-order homophily based on the graph shift operator \(\hat{\mat{A}}\); measures the extent to which nodes within \(r\) hops have the same class label}, 
 sort=32}

\newglossaryentry{rselfconnr}{name=\ensuremath{\rselfconn{r}{\hat{\mat{A}}}}, 
 description={higher- or \(r\)-order self-connectivity of the graph, averaging diagonal entries of \(\hat{\mat{A}}^r\)}, 
 sort=33}

\newglossaryentry{rconnr}{name=\ensuremath{\rconn{r}{\hat{\mat{A}}}}, 
 description={higher- or \(r\)-order total connectivity of the graph, averaging all entries of \(\hat{\mat{A}}^r\)}, 
 sort=34}

\newglossaryentry{localrhomo}{name=\ensuremath{\localhomo{r,s}{\hat{\mat{A}}}{i}}, 
 description={Class-bottlenecking score at node~\(\,i\); measure of the mixing of same-class signals over \(r\) and \(s\) hops}, 
 sort=35}

\newglossaryentry{localrselfconn}{name=\ensuremath{\localselfconn{r,s}{\hat{\mat{A}}}{i}}, 
 description={Self-bottlenecking score at node~\(\,i\); measure of the mixing of same-node signals over \(r\) and \(s\) hops}, 
 sort=36}

\newglossaryentry{localrconn}{name=\ensuremath{\localconn{r,s}{\hat{\mat{A}}}{i}}, 
 description={Total-bottlenecking score at node~\(\,i\); measure of the mixing of all node signals over \(r\) and \(s\) hops}, 
 sort=37}

\newglossaryentry{B}{name=\ensuremath{\mat{B}}, 
 description={Block probability matrix in the stochastic block model (SBM)}, 
 sort=38}

\newglossaryentry{pi}{name=\ensuremath{\vect{\pi}}, 
 description={Vector of expected class proportions; each entry is the probability of a node belonging to a given class}, 
 sort=39}

\newglossaryentry{Pi}{name=\ensuremath{\mat{\Pi}}, 
 description={$\diag{\vect{\pi}}$ i.e. diagonal matrix of expected class proportions}, 
 sort=40}

\newglossaryentry{hatB}{name=\ensuremath{\hat{\mat{B}}}, 
 description={Normalised block matrix \(\mat{D}^{-\frac{1}{2}}\mat{\Pi}^{\frac{1}{2}}\mat{B}\mat{\Pi}^{\frac{1}{2}}\mat{D}^{-\frac{1}{2}}\) for the SBM, where $\mat{D}:=\diag{\mat{B}\vect{\pi}}$ is the diagonal matrix of expected class-wise degrees}, 
 sort=41}

\newglossaryentry{Cmat}{name=\ensuremath{\mat{C}}, 
 description={Confusion matrix relating true and predicted class labels}, 
 sort=42}

\newglossaryentry{permmat}{name=\ensuremath{\mat{P_k}}, 
 description={Symmetric $k\times k$ permutation matrix.}, 
 sort=43}

\newglossaryentry{updatefunc}{name=\ensuremath{U_\ell(\cdot,\cdot)}, 
 description={The update function of the message passing function}, 
 sort=44}

\newglossaryentry{messagefunc}{name=\ensuremath{M_\ell(\cdot,\cdot)}, 
 description={The message function of the message passing function}, 
 sort=45}

\newglossaryentry{alphaone}{name=\ensuremath{\alpha_1}, 
 description={Upper bound on the norm of the derivative of the update function $U_\ell(\cdot,\cdot)$ with respect to its first argument i.e. a node's own representation}, 
 sort=46}

\newglossaryentry{alphatwo}{name=\ensuremath{\alpha_2}, 
 description={Upper bound on the norm of the derivative of the update function $U_\ell(\cdot,\cdot)$ with respect to its second argument i.e. a node's neighbourhood-aggregated message input}, 
 sort=47}

\newglossaryentry{betaone}{name=\ensuremath{\beta_1}, 
 description={Upper bound on the norm of the derivative of the message function $M_\ell(\cdot,\cdot)$ with respect to its first argument i.e. a node's own representation}, 
 sort=48}

\newglossaryentry{betatwo}{name=\ensuremath{\beta_2}, 
 description={Upper bound on the norm of the derivative of the message function $M_\ell(\cdot,\cdot)$ with respect to its second argument i.e. a node's neighbour's features. Denoted as $\beta$ when $M_\ell$ does not depend on its first argument}, 
 sort=49}

\newglossaryentry{Wmat}{name=\ensuremath{\mat{W}^{(\ell)}}, 
 description={Weight matrix of layer $\ell$ of an MPNN}, 
 sort=50}

\newglossaryentry{lambda}{name=\ensuremath{\lambda_{ij}}, 
 description={Shortest path length between nodes~\(\,i\) and~\(\,j\)}, 
 sort=51}

\newglossaryentry{avgdeg}{name=\ensuremath{\langle d \rangle}, 
 description={Average degree of nodes in the graph}, 
 sort=52}

 \newglossaryentry{edgehomophily}{name=\ensuremath{h}, 
 description={Edge homophily of the graph}, 
 sort=52}

\newglossaryentry{WG}{name=\ensuremath{W_G}, 
 description={Set of all walks in the graph $G$ (used when analysing message propagation paths)}, 
 sort=53}

\glsaddall
\newgeometry{left=0.2in, right=0.2in, top=1in, bottom=1in} 
\twocolumn
\printglossary[style=super,title=Glossary]
\onecolumn
\restoregeometry
\appendix

\section{Appendix A: Extended theorems}\label{sec:appendix_A}

This section provides the full statements of all theorems, lemmas, and corollaries presented in the main text, along with interpretations to clarify their significance and implications.

\subsection*{Signal-to-noise ratio and input sensitivity}

\snrsensitivity*

\paragraph*{Interpretation.} Theorem \ref{theorem:snrsensitivity} provides a fundamental decomposition of the SNR achieved by an MPNN. It shows how the SNR, which measures the distinguishability of class-specific signals relative to noise, is determined both by the quality of input features---captured by the covariance matrices $\mat{\Sigma}, \mat{\Phi}, \mat{\Psi}$---and by the MPNN's architecture and the graph structure, as captured by the feature-agnostic sensitivity measures $\sigsenexp[i]{\ell}{p,q,r}$, $\globsenexp[i]{\ell}{p,q,r}$, $\noisesenexp[i]{\ell}{p,q,r}$. Specifically, high signal sensitivity $\sigsenexp[i]{\ell}{p,q,r}$ amplifies the class-discriminative parts of the signal $\mat{\Sigma}$, while high global sensitivity $\globsenexp[i]{\ell}{p,q,r}$ and noise sensitivity $\noisesenexp[i]{\ell}{p,q,r}$ amplify the non-discriminative global shifts $\mat{\Phi}$ and node-specific noise $\mat{\Psi}$, respectively. This theorem establishes a quantitative link between the model's input processing abilities (sensitivities) and the resulting quality of learned representations (SNR), forming the basis for understanding when and how MPNNs can enhance class separability beyond what is present in the raw input features. The approximation holds well when features are concentrated near the origin, allowing for analysis based on the model's local behaviour via Jacobians.\\

\sensitivityratiocondition*

\paragraph*{Interpretation.} Corollary \ref{corollary:gdl_criterion} provides the precise condition under which an MPNN is guaranteed to improve the SNR compared to a simple feedforward network (FNN) baseline, assuming IID feature dimensions. The condition highlights that an MPNN outperforms an FNN when its signal sensitivity $\sigsenexp[i]{\ell}{p,q,r}$ sufficiently outweighs a convex combination of its noise sensitivity $\noisesenexp[i]{\ell}{p,q,r}$ and global sensitivity $\globsenexp[i]{\ell}{p,q,r}$. We note that, due to the semipositive definiteness of the sensitivities in Eq. \eqref{eq:sensitivities_intro}, these sums over $q$ are always non-negative.  The local noise proportion $\rho$ controls the difficulty of the classification task on a particular feature distribution: In the high global sensitivity regime where $\globsenexp[i]{\ell}{p,q,q}>\noisesenexp[i]{\ell}{p,q,q}$ (such as GCNs with low-pass graph filters), larger $\rho$ makes the condition easier to satisfy, but in the high local sensitivity regime where $\globsenexp[i]{\ell}{p,q,q}<\noisesenexp[i]{\ell}{p,q,q}$ (such as GCNs with high-pass graph filters), smaller $\rho$ makes the condition easier to satisfy. We can see that in the high global sensitivity regime, low global noise relative to local noise improves message passing benefit over feedforward models, and vice versa for high local sensitivity regime. This corollary provides a localised, feature-independent diagnostic tool for potential MPNN performance, as validated in Figure \ref{fig:snr_node_predictions}): by calculating the sensitivities for a given node and MPNN architecture, one can predict whether leveraging the graph structure via message passing is likely to improve the representation quality for that specific node, compared to just using its own features. It formalises the intuition that MPNNs help when they selectively amplify class signals more than noisy or background variations. The condition surprisingly does not depend on the class-wise variance $\sigma^2$, suggesting that the degree to which message passing may improve class-specific separability over FNNs does not depend on class-wise signal quality, but on having the appropriate \emph{kind} of noise.

\subsection*{Weighted homophily and sensitivity bounds}

\begin{restatable}[Weighted homophily bounds sensitivity]{theorem}{sensitivitybound}
\label{theorem:general_signal_sensntivity_bound}
  Let $\sigsenexp[i]{\ell}{p,q,r}$, $\globsenexp[i]{\ell}{p,q,r}$, $\noisesenexp[i]{\ell}{p,q,r}$ be the signal, global and noise sensitivities respectively of the $\nth{p}$ output feature dimension of node $i$ to input feature dimensions $q,r$ at the $\nth{\ell}$ layer of an MPNN that uses the graph shift operator $\hat{\mat{A}}$. Assuming that there exist constants $\alpha_{1}, \alpha_{2}, \beta$ such that $\forall r\in[\ell]$ the update and message functions satisfy $\norm{ \nabla_1 U_r}\leq \alpha_{1}$, $\norm{ \nabla_2 U_r}\leq \alpha_{2}$, and both $\norm{\nabla_1 M_r}, \norm{\nabla_2 M_r} \leq \beta$, the sensitivities can be bounded in terms of local bottlenecking scores:
  \begin{align*}
    &\abs{\sigsenexp[i]{\ell}{p,q,r}} \le \sum_{s=0}^{\ell}\sum_{t=0}^{\ell} \binom{\ell}{s}\,\binom{\ell}{t}\,\alpha_{1}^{2\ell-s-t}\,(\alpha_{2}\beta)^{s+t} \localhomo{s,t}{\hat{\mat{A}}+\diag{\hat{\mat{A}}\ones{n}}}{i},\\
    &\abs{\globsenexp[i]{\ell}{p,q,r}} \le \sum_{s=0}^{\ell}\sum_{t=0}^{\ell} \binom{\ell}{s}\,\binom{\ell}{t}\,\alpha_{1}^{2\ell-s-t}\,(\alpha_{2}\beta)^{s+t} \localconn{s,t}{\hat{\mat{A}}+\diag{\hat{\mat{A}}\ones{n}}}{i},\\
    &\abs{\noisesenexp[i]{\ell}{p,q,r}} \le \sum_{s=0}^{\ell}\sum_{t=0}^{\ell} \binom{\ell}{s}\,\binom{\ell}{t}\,\alpha_{1}^{2\ell-s-t}\,(\alpha_{2}\beta)^{s+t} \localselfconn{s,t}{\hat{\mat{A}}+\diag{\hat{\mat{A}}\ones{n}}}{i},
  \end{align*}
  where $\localhomo{s,t}{\cdot}{i}$, $\localconn{s,t}{\cdot}{i}$, and $\localselfconn{s,t}{\cdot}{i}$ are the class-bottlenecking score, total-bottlenecking score, and self-bottlenecking score defined in Eq.~\eqref{eq:p_order_local_homophily}.
  Specifically for \emph{isotropic} MPNN models, where $\norm{\nabla_1 M_r}=0$ i.e. messages depend only on the source node's features:

  \begin{align*}
    &\abs{\sigsenexp[i]{\ell}{p,q,r}} \le \sum_{s=0}^{\ell}\sum_{t=0}^{\ell} \binom{\ell}{s}\,\binom{\ell}{t}\,\alpha_{1}^{2\ell-s-t}\,(\alpha_{2}\beta)^{s+t} \localhomo{s,t}{\hat{\mat{A}}}{i},\\
    &\abs{\globsenexp[i]{\ell}{p,q,r}} \le \sum_{s=0}^{\ell}\sum_{t=0}^{\ell} \binom{\ell}{s}\,\binom{\ell}{t}\,\alpha_{1}^{2\ell-s-t}\,(\alpha_{2}\beta)^{s+t} \localconn{s,t}{\hat{\mat{A}}}{i},\\
    &\abs{\noisesenexp[i]{\ell}{p,q,r}} \le \sum_{s=0}^{\ell}\sum_{t=0}^{\ell} \binom{\ell}{s}\,\binom{\ell}{t}\,\alpha_{1}^{2\ell-s-t}\,(\alpha_{2}\beta)^{s+t} \localselfconn{s,t}{\hat{\mat{A}}}{i}.
  \end{align*}
\end{restatable}

\paragraph*{Interpretation.} Theorem \ref{theorem:general_signal_sensntivity_bound} establishes a fundamental limit on the achievable sensitivities of an MPNN, imposed by the graph structure itself, independent of specific features. It shows that the signal sensitivity $\sigsenexp[i]{\ell}{p,q,r}$, which drives the amplification of class-distinguishing information (as seen in Theorem \ref{theorem:snrsensitivity}), is locally bounded by the class-bottlenecking score $\localhomo{s,t}{\cdot}{i}$ at the target node $i$. This score---defined in Eq.~\eqref{eq:p_order_local_homophily}---measures the aggregate influence of pairs of same-class source nodes reaching node $i$ via paths of lengths $s$ and $t$. A low class-bottlenecking score directly implies a low upper bound on signal sensitivity, meaning that if the graph structure prevents same-class signals from effectively converging at node $i$---due to a lack of paths reaching $i$ or lack of breadth along paths---no MPNN architecture satisfying these derivative bounds can overcome this limitation to achieve high signal sensitivity at that node. Similarly, the total-bottlenecking score $\localconn{s,t}{\cdot}{i}$ and self-bottlenecking score $\localselfconn{s,t}{\cdot}{i}$ bound the global and noise sensitivities, respectively. The theorem draws a distinction between general (anisotropic) MPNNs and isotropic ones (like GCN), showing different dependencies on the graph shift operator. It identifies the class-bottlenecking score as the key structural quantity governing the local potential for signal amplification in MPNNs. Averaging these bounds over all graph nodes leads to the global bounds in Eq.~\eqref{eq:avg_sensitivity_bound} involving higher-order homophily.\\

\begin{restatable}{corollary}{avgsensbound}\label{corollary:avg_sensitivity_bounds}
  Under the assumptions of Theorem \ref{theorem:general_signal_sensntivity_bound}, and assuming a symmetric graph shift operator $\hat{\mat{A}}$, the average sensitivities over all nodes $i$ are bounded by higher-order homophily and connectivity measures defined in Eq.~\eqref{eq:p_order_connectivity}:
  \begin{align*}
  \abs{\overbar{\sigsenexp[p]{\ell}{q,r}}} &\leq \sum_{u=0}^{2\ell} \binom{2\ell}{u}\,\alpha_1^{2\ell-u} (\alpha_2\beta)^u\, \rhomo{u}{\hat{\mat{A}}}, \\
  \abs{\overbar{\globsenexp[p]{\ell}{q,r}}} &\leq \sum_{u=0}^{2\ell} \binom{2\ell}{u}\,\alpha_1^{2\ell-u} (\alpha_2\beta)^u\, \rconn{u}{\hat{\mat{A}}}, \\
  \abs{\overbar{\noisesenexp[p]{\ell}{q,r}}} &\leq \sum_{u=0}^{2\ell} \binom{2\ell}{u}\,\alpha_1^{2\ell-u} (\alpha_2\beta)^u\, \rselfconn{u}{\hat{\mat{A}}},
  \end{align*}
  where $\overbar{\phantom{(}\cdot\phantom{)}}$ denotes the average over nodes $i$.
\end{restatable}

\textbf{Interpretation:} Corollary \ref{corollary:avg_sensitivity_bounds} translates the local bounds from Theorem \ref{theorem:general_signal_sensntivity_bound} into global bounds on the average sensitivities across the entire graph. It shows that the average signal sensitivity is restricted by the graph's higher-order homophily $\rhomo{u}{\hat{\mat{A}}}$ up to order $2\ell$. This means that graphs lacking sufficient multi-hop connectivity between same-class nodes (i.e., low $\rhomo{u}{\hat{\mat{A}}}$ for relevant $u$) will inherently limit the average signal sensitivity achievable by any $\ell$-layer MPNN. This provides a graph-wide explanation for why MPNNs might struggle on globally heterophilic graphs or graphs where communities do not align well with classes. The dependence on homophily up to order $2\ell$ explains why MPNNs can sometimes perform well even on graphs with low first-order homophily (like bipartite graphs), provided they exhibit strong \emph{higher-order} homophily patterns. Similarly, average global and noise sensitivities are bounded by the average total and self-connectivities $\rconn{u}{\hat{\mat{A}}}, \rselfconn{u}{\hat{\mat{A}}}$.

\subsection*{Graph ensemble analysis: Underreaching and oversquashing}

\begin{restatable}[Underreaching in MPNNs for sparse graph ensembles; Loomba and Jones \cite{loomba2021geodesic}]{lemma}{underreaching}
\label{lemma:underreaching}
  For an undirected and simple graph with $n$ nodes encoded by the adjacency matrix $\mat{A}$, sampled from a general random graph family with conditionally independent edges and expected adjacency matrix $\expect{\mat{A}}$, if the graph is sparse in the sense that $\forall(i,j):\expect{A_{ij}}=\Theta\left(n^{-1}\right)$ or $0$, it has no bottlenecks in the sense that $\forall(i,j):\abs{\{k \in [n] \setminus\{i, j\}: \expect{A_{ik}}\expect{A_{kj}} > 0\}} =\Omega(n)$ or $0$, each node is on the giant component with probability $1-\smallo{1}$, and $\expect{\mat{A}}-\eye{n}$ (where $\eye{n}$ is the $n\times n$ identity matrix) is invertible, then asymptotically the cumulative distribution function of the length of the shortest path $\lambda_{ij}$ between nodes $i$ and $j\ne i$ is given by:
  \begin{align*}
    \prob{\lambda_{ij}\leq r} \approx 
    \left[\sum_{s=1}^r\expect{\mat{A}}^s\right]_{ij},
  \end{align*}
  where ``$\approx$'' indicates an asymptotic first-order approximation as $n\to\infty$.\\
\end{restatable}

\textbf{Interpretation:} Lemma \ref{lemma:underreaching} specifically focuses on the underreaching component of message passing in sparse random graphs. It provides a simple asymptotic formula for the probability that two nodes $i$ and $j$ are connected by a path of length at most $r$ \cite{loomba2021geodesic}. This probability is approximated by summing the $\nth{(i,j)}$ entries of the first $r$ powers of the expected adjacency matrix. This result quantifies the reachability between nodes based solely on the expected structure of the graph ensemble. It forms a key part of the analysis in Theorem \ref{theorem:underreaching_oversquashing} and is fundamental for understanding how graph sparsity limits the propagation distance of information in MPNNs. The conditions ensure that the graph is sparse enough for the approximations to hold but connected enough for paths to likely exist between all node pairs.\\

\begin{restatable}[Boundary oversquashing in MPNNs for sparse graph ensembles]{lemma}{oversquashing}
\label{lemma:oversquashing}
  Assume the same conditions as in Lemma \ref{lemma:underreaching}, and additionally assume large expected node degrees encoded in the diagonal matrix $\avg{\mat{D}}\coloneqq \diag{\expect{\mat{A}}\ones{n}}$ where $\ones{n}$ is the length-$n$ vector of ones. Then for the symmetric normalised adjacency matrix $\hat{\mat{A}}_\mathrm{sym}$ the boundary oversquashing between nodes $i$ and $j\ne i$, where $\lambda_{ij}$ is the shortest path distance from $i$ to $j$, is asymptotically bounded by:
  \begin{align}\label{eq:oversquashing}
    \condexpect{\left[\hat{\mat{A}}^r_\mathrm{sym}\right]_{ij}}{\lambda_{ij} = r} \lessapprox \frac{\left[\avg{\mat{D}}^{-\frac{1}{2}}\expect{\mat{A}}\left(\left\{\avg{\mat{D}}^{-1}-\avg{\mat{D}}^{-2}\left(\eye{n}-e^{-\avg{\mat{D}}}\right)\right\}\expect{\mat{A}}\right)^{r-1}\avg{\mat{D}}^{-\frac{1}{2}}\right]_{ij}}{\left[\expect{\mat{A}}^r\right]_{ij}},
  \end{align}
  and the bound gets tighter for larger mean degrees.
\end{restatable}

\textbf{Interpretation:} Lemma \ref{lemma:oversquashing} provides a specific asymptotic upper bound for the oversquashing factor, which quantifies the attenuation of information travelling along the shortest paths of a given length $r$. It states that the expected contribution of node $j$ to node $i$'s representation after $r$ steps of message passing using the normalised adjacency matrix $\hat{\mat{A}}_\mathrm{sym}$, given that the shortest path is indeed length $r$, can be bounded using only the expected adjacency matrix $\expect{\mat{A}}$ and the expected degree matrix $\avg{\mat{D}}$. The bound highlights that oversquashing depends inversely on node degrees, via $\avg{\mat{D}}^{-1/2}$ and $\avg{\mat{D}}^{-1}$, and involves complex interactions captured by the powers of the expected adjacency matrix, normalised by degree-related terms. This lemma formalises the intuition that even if a path exists (addressing underreaching), the actual amount of information transmitted can be significantly reduced due to the normalisation process leading to a lack of breadth for signals arriving on too few paths. The bound becomes tighter for graphs with larger average degrees.\\

\underreachingoversquashing*

\textbf{Interpretation:} Theorem \ref{theorem:underreaching_oversquashing} combines the results of Lemmas \ref{lemma:underreaching} and \ref{lemma:oversquashing}, and provides asymptotic approximations for the two components identified in the underreaching/oversquashing decomposition for sparse graph ensembles in Eq.~\eqref{eq:underover}.
1. \textbf{Underreaching $\prob{\lambda_{ij}= r}$:} It states that the probability of the shortest path between nodes $i$ and $j$ having length $r$ can be approximated by the $\nth{(i,j)}$ entry of the $\nth{r}$ power of the expected adjacency matrix $\expect{\mat{A}}$. This quantifies the likelihood that information can potentially reach from $j$ to $i$ in exactly $r$ hops, primarily limited by the graph's expected connectivity density.
2. \textbf{Oversquashing $\condexpect{\left[\hat{\mat{A}}^r_\mathrm{sym}\right]_{ij}}{\lambda_{ij}=r}$:} It approximates the expected value of the $\nth{(i,j)}$ entry in the $\nth{r}$ power of the normalised adjacency matrix $\hat{\mat{A}}_{\mathrm{sym}}$, given that the shortest path has length $r$. This term captures how much of the signal that does arrive via shortest paths of length $r$ is preserved after accounting for the lack of breadth for signals arriving on too few paths and the dampening effect of degree normalisation. The approximation involves powers of a normalised version of the expected adjacency matrix. The fact that this term decays rapidly when shortest paths are shorter than $r$ ($t<r$) confirms that $\hat{\mat{A}}^r_\mathrm{sym}$ primarily captures information flow along paths of length close to $r$.
Together, these approximations allow us to estimate the expected entries of $\hat{\mat{A}}_\mathrm{sym}^r$, and consequently the expected higher-order homophily measures, directly from the parameters of the graph ensemble (like the SBM with block matrix $\mat{B}$ and class proportions $\mat{\Pi}$), providing a way to predict structural limitations on message passing without needing to analyse specific graph instances.

\subsection*{Stochastic block model analysis}

\begin{restatable}[SBM higher-order homophily]{theorem}{sbmhigherorderhomophily}
\label{theorem:sbm_higher_order_homophily}
  Consider an undirected and simple graph with $n$ nodes encoded by the adjacency matrix $\mat{A}$ sampled from a sparse stochastic block model (SBM) such that node classes are IID as per $c \sim \categorical{\vect{\pi}}$ where $\vect{\pi} = (\pi_1, \pi_2, \ldots, \pi_k)^T$ is the probability distribution over the $k$ classes, with class membership denoted by $\{\hat{y}_i\}_{i\in[n]}$. Assume these generating classes $\{\hat{y}_i\}_{i\in[n]}$ differ from the true node class labels $\{y_i\}_{i\in[n]}$ used for evaluating homophily. Let nodes connect with probability $\expect{\mat{A}}_{ij}\coloneqq \frac{B_{\hat{y}_i\hat{y}_j}}{n}$, where $\mat{B}$ is the SBM block matrix. Let $\mat{\Pi}\coloneqq\diag{\vect{\pi}}$ be the diagonal matrix of expected generating-class proportions and $\mat{D}\coloneqq \diag{\mat{B}\vect{\pi}}$ be the diagonal matrix of expected generating-class-wise degrees). Define the confusion matrix $\mat{C} \in \mathbb{R}^{k \times k}$ relating true labels $y_i$ to generating labels $\hat{y}_i$ as:
  \begin{align}
    C_{uv} := \frac{1}{n} \sum_{i=1}^n \delta_{\hat{y}_i u} \delta_{y_i v}.
  \end{align}
  (Note that $C_{uv}$ is the proportion of nodes with generating label $u$ and true label $v$. If $y_i = \hat{y}_i$ for all $i$, then $\mat{C} = \mat{\Pi}$).
  Assuming the conditions of Theorem \ref{theorem:underreaching_oversquashing} hold, the expected $\ell$-order homophily, self-connectivity, and total connectivity (Eq. \eqref{eq:p_order_connectivity}) with respect to the true labels $\{y_i\}_{i\in[n]}$, using the symmetric normalised adjacency matrix $\hat{\mat{A}}_\mathrm{sym}$ as the graph shift operator, can be approximated by:
  \begin{align*}
  & \expect{\rhomo{\ell}{\hat{\mat{A}}_\mathrm{sym}}} \approx \trace \left(\mat{C}^T\mat{\Pi}^{-\frac{1}{2}}\hat{\mat{B}}^{\ell}\mat{\Pi}^{-\frac{1}{2}} \mat{C}\right) + O\left(\frac{1}{\avg{d}}\right),\\
  & \expect{\rconn{\ell}{\hat{\mat{A}}_\mathrm{sym}}} \approx \ones{k}^T\mat{\Pi}^{\frac{1}{2}}\hat{\mat{B}}^{\ell}\mat{\Pi}^{\frac{1}{2}}\ones{k} + O\left(\frac{1}{\avg{d}}\right),\\
  & \expect{\rselfconn{\ell}{\hat{\mat{A}}_\mathrm{sym}}} \approx O\left(\frac{1}{\avg{d}^{\ell}}\right),
  \end{align*}
  where $\hat{\mat{B}}:=\mat{D}^{-\frac{1}{2}}\mat{\Pi}^{\frac{1}{2}}\mat{B}\mat{\Pi}^{\frac{1}{2}}\mat{D}^{-\frac{1}{2}}$ is a normalised version of the block matrix, and $\avg{d}$ is the average degree.
\end{restatable}

\textbf{Interpretation:} Theorem \ref{theorem:sbm_higher_order_homophily} provides explicit approximations for the expected higher-order homophily, total connectivity, and self-connectivity for graphs generated by a sparse SBM. It relates these structural properties directly to the SBM parameters: the block matrix $\mat{B}$, the expected generating-class proportions $\mat{\Pi}$, and the confusion matrix $\mat{C}$ which accounts for potential mismatches between the SBM's generating class labels and the true class labels used for evaluation. The theorem shows that the expected $\ell$-order homophily is primarily determined by the $\nth{\ell}$ power of a normalised block matrix $\hat{\mat{B}}$, projected through the confusion matrix $\mat{C}$. This allows for prediction of the graph's suitability for MPNNs directly from the SBM parameters. Notably, the self-connectivity $\rselfconn{k}{\hat{\mat{A}}_{\mathrm{sym}}}$ is asymptotically negligible for sparse graphs, while the total connectivity $\rconn{\ell}{\hat{\mat{A}}_{\mathrm{sym}}}$ depends only on the SBM parameters (i.e. not on the confusion matrix). This theorem is key for deriving the optimal SBM structures in Theorem \ref{theorem:optimal_connectivity} and for understanding how imperfect label predictions may affect rewiring strategies (as discussed in the \nameref{sec:methods} section, Eq.~\eqref{eq:imperfect_labels_sbm_homophily}). The approximations become more accurate as the average degree $\avg{d}$ increases.\\

\begin{restatable}[Bounds for first and second order homophily in sparse SBMs]{lemma}{sbmhomophily}\label{lemma:sbm_homophily}
Consider an undirected and simple graph with $n$ nodes encoded by the adjacency matrix $\mat{A}$ sampled from a sparse stochastic block model (SBM) with block matrix $\mat{B}$, expected generating-class proportions $\vect{\pi}$, and confusion matrix $\mat{C}$ relating true class labels $\{y_i\}_{i\in[n]}$ to generating class labels $\{\hat{y}_i\}_{i\in[n]}$, as defined in Theorem \ref{theorem:sbm_higher_order_homophily}. Let $\mat{\Pi} := \diag{\vect{\pi}}$ and $\mat{D} := \diag{\mat{B}\vect{\pi}}$. Assuming the conditions of Theorem \ref{theorem:underreaching_oversquashing} hold, the expected first and second order homophily (with respect to true labels $y_i$) using the symmetric normalised adjacency matrix $\hat{\mat{A}}_\mathrm{sym}$ can be tightly bounded by:

\begin{align*}
  \expect{\rhomo{1}{\hat{\mat{A}}_\mathrm{sym}}} &\lessapprox \trace \left(\mat{C}^T\mat{D}^{-\frac{1}{2}} \mat{B}\mat{D}^{-\frac{1}{2}} \mat{C}\right),\\
  \expect{\rhomo{2}{\hat{\mat{A}}_\mathrm{sym}}} &\lessapprox \vect{\pi}^T\mat{D}^{-1}\mat{B}\mat{D}^{-1} \vect{\pi} + \trace \left(\mat{C}^T \mat{D}^{-1}\mat{B}\left\{\mat{D}^{-1}-\mat{D}^{-2}\left(\eye{k}-e^{-\mat{D}}\right)\right\}\mat{\Pi B} \mat{C}\right),
\end{align*}

where $\eye{k}$ is the size-$k$ identity matrix, and the bounds become tighter as the expected class-wise mean degrees (diagonal entries of $\mat{D}$) increase.
\end{restatable}

\textbf{Interpretation:} Lemma \ref{lemma:sbm_homophily} provides tighter upper bounds for the expected first and second order homophily in sparse SBMs, compared to the general $\ell$-order approximation in Theorem \ref{theorem:sbm_higher_order_homophily}. These bounds explicitly show the dependence on the SBM parameters ($\mat{B}$, $\mat{\Pi}$, $\mat{D}$) and the confusion matrix $\mat{C}$. For first order homophily, the bound resembles a normalised trace involving the block matrix and confusion matrix. For second order homophily, the bound has two terms: one related to return probabilities $\vect{\pi}^T\mat{D}^{-1}\mat{B}\mat{D}^{-1} \vect{\pi}$, and a more complex term involving the oversquashing correction factor seen in Lemma \ref{lemma:oversquashing}. These tighter bounds are particularly useful for analysing shallow MPNNs of a single layer, or situations where lower-order homophily dominates performance. They confirm that the core relationships derived from the simpler approximations in Theorem \ref{theorem:sbm_higher_order_homophily} hold, while providing more refined estimates that account for degree-dependent effects---especially relevant when average degrees are not extremely large.

\subsection*{Optimal graph structures}

\optimalconnectivity*

\textbf{Interpretation:} Theorem \ref{theorem:optimal_connectivity} identifies the theoretically optimal connectivity patterns within the SBM framework for maximising objectives related to powers of the normalised block matrix $\hat{\mat{B}}$, such as the expected higher-order homophily for which $\hat{\mat{C}} := \mat{\Pi}^{-1/2}\mat{C}$ from Theorem \ref{theorem:sbm_higher_order_homophily}. For even powers $\ell$ (relevant for the sensitivity bounds of standard GCNs/SGCs; see Eq.~\eqref{eq:mean_signal_sensitivity_sgc} and the discussion after Eq.~\eqref{eq:avg_sensitivity_bound}), the optimal block structures $\mat{B}$ correspond to graphs that are disjoint unions of single-class clusters (where a cluster consists of nodes from one class) and two-class-bipartite clusters  (where nodes of one class connect only to nodes of another specific class, and vice versa). These structures are encoded by symmetric permutation matrices $\mat{P}_k$. Thus, we see that perfect homophily ($\mat{P}_k = \mat{I}_k$) is optimal, but so are structures with perfect heterophily between pairs of classes (e.g., block-wise bipartite structures). For odd powers $\ell$, only the purely homophilic structure ($\mat{P}_k = \mat{I}_k$) is optimal. This theorem provides a fundamental insight for graph design and rewiring: aiming for these specific block structures---disjoint unions of single-class and two-class-bipartite clusters---is predicted to maximise the potential signal sensitivity of MPNNs operating on graphs that conform to an SBM structure. It transforms the combinatorial optimisation problem of finding the best graph into a continuous optimisation problem of finding the best graph ensemble parameters, solved by selecting an appropriate symmetric permutation.

\section*{Appendix B: Proofs}\label{sec:proofs}

\snrsensitivity*

\begin{proof}
Consider an $\ell$-layer MPNN with $\nth{p}$ output feature $H_{ip}^{(\ell)}$ at node $i$, and let $X_{jq}$ denote the $\nth{q}$ input feature of node $j$. Assume the feature decomposition
\begin{align*}
\mat{X}_j = \vect{\mu}_{y_j} + \vect{\gamma} + \vect{\epsilon}_j,
\end{align*}
where $\expect{\vect{\gamma}} = \vect{0}$ and $\mathrm{Cov}(\gamma_q, \gamma_r) := \Phi_{qr}$, and $\vect{\epsilon}_j$ are node-wise IID zero-mean noise vectors with element-wise covariance $\Psi_{qr} := \mathrm{Cov}(\epsilon_{jq}, \epsilon_{jr})$. The class-wise covariance is $\Sigma_{qr} := \mathrm{Cov}(\mu_{y_j,q}, \mu_{y_j,r})$.

To analyse the sensitivity of the MPNN's output to its input, we use the first-order Taylor expansion of $H_{ip}^{(\ell)}$ around $\mat{X} = \mathbf{0}$, assuming features are sufficiently concentrated near the origin:
\begin{align*}
H_{ip}^{(\ell)} \simeq H_{ip}^{(\ell)}\Big|_{\mat{X}=\mathbf{0}} + \sum_{j \in V} \sum_{q=1}^{d_{\mathrm{in}}} \frac{\partial H_{ip}^{(\ell)}}{\partial X_{jq}}\Bigg|_{\mat{X}=\mathbf{0}} X_{jq}.
\end{align*}
Substituting the feature decomposition $X_{jq} = \mu_{y_j,q} + \gamma_q + \epsilon_{jq}$:
\begin{equation}\label{eq:expanded_linear_mpnn_approximation}
\begin{aligned}
H_{ip}^{(\ell)} &\simeq H_{ip}^{(\ell)}\Big|_{\mat{X}=\mathbf{0}} + \sum_{j,q} \frac{\partial H_{ip}^{(\ell)}}{\partial X_{jq}}\Bigg|_{\mat{X}=\mathbf{0}} (\mu_{y_j,q} + \gamma_q + \epsilon_{jq}) \\
&= H_{ip}^{(\ell)}\Big|_{\mat{X}=\mathbf{0}} + \sum_{j,q} \frac{\partial H_{ip}^{(\ell)}}{\partial X_{jq}}\Bigg|_{\mat{X}=\mathbf{0}} \mu_{y_j,q} + \sum_{j,q} \frac{\partial H_{ip}^{(\ell)}}{\partial X_{jq}}\Bigg|_{\mat{X}=\mathbf{0}} \gamma_q + \sum_{j,q} \frac{\partial H_{ip}^{(\ell)}}{\partial X_{jq}}\Bigg|_{\mat{X}=\mathbf{0}} \epsilon_{jq}.
\end{aligned}
\end{equation}

Recall the definition of the SNR:
\begin{align}\label{eq:snr_def_recall}
  \snrsub{H_{ip}^{(\ell)}}{} := \frac{\varsub{\condexpectsub{H_{ip}^{(\ell)}}{\vect{\mu}}{\vect{\gamma}, \vect{\epsilon}}}{\vect{\mu}}}{\expectsub{\varsub{H_{ip}^{(\ell)}\,\Big|\, \vect{\mu}}{\vect{\gamma}, \vect{\epsilon}}}{\vect{\mu}}}.
\end{align}

Going forward in this proof we omit subscripts on $\mathbb{E}$ and $\mathrm{Var}$ for brevity, as the quantity being averaged over should be clear by the conditioning on $\vect{\mu}$. For the numerator of the SNR in Eq. \eqref{eq:snr_def_recall}, we first compute the conditional expectation of $H_{ip}^{(\ell)}$, approximated as in Eq. \eqref{eq:expanded_linear_mpnn_approximation} given the signal terms $\{\vect{\mu}_{y_j}\}_{j\in[n]}$:
\begin{align*}
\mathbb{E}\left[H_{ip}^{(\ell)}\,\middle|\,\{\vect{\mu}_{y_j}\}_{j\in[n]}\right] &\simeq \mathbb{E}\left[ H_{ip}^{(\ell)}\Big|_{\mat{X}=\mathbf{0}} + \sum_{j,q} \frac{\partial H_{ip}^{(\ell)}}{\partial X_{jq}}\Bigg|_{\mat{X}=\mathbf{0}} \mu_{y_j,q} + \sum_{j,q} \frac{\partial H_{ip}^{(\ell)}}{\partial X_{jq}}\Bigg|_{\mat{X}=\mathbf{0}} \gamma_q + \sum_{j,q} \frac{\partial H_{ip}^{(\ell)}}{\partial X_{jq}}\Bigg|_{\mat{X}=\mathbf{0}} \epsilon_{jq}\,\middle|\,\{\vect{\mu}_{y_j}\}_{j\in[n]} \right] \\
&= H_{ip}^{(\ell)}\Big|_{\mat{X}=\mathbf{0}} + \sum_{j,q} \frac{\partial H_{ip}^{(\ell)}}{\partial X_{jq}}\Bigg|_{\mat{X}=\mathbf{0}} \mu_{y_j,q} + \sum_{j,q} \frac{\partial H_{ip}^{(\ell)}}{\partial X_{jq}}\Bigg|_{\mat{X}=\mathbf{0}} \mathbb{E}[\gamma_q] + \sum_{j,q} \frac{\partial H_{ip}^{(\ell)}}{\partial X_{jq}}\Bigg|_{\mat{X}=\mathbf{0}} \mathbb{E}[\epsilon_{jq}] \\
&= H_{ip}^{(\ell)}\Big|_{\mat{X}=\mathbf{0}} + \sum_{j,q} \frac{\partial H_{ip}^{(\ell)}}{\partial X_{jq}}\Bigg|_{\mat{X}=\mathbf{0}} \mu_{y_j,q}.
\end{align*}
Next, we compute the variance of the conditional expectation:
\begin{align*}
\mathrm{Var}\Bigl(\mathbb{E}\left[H_{ip}^{(\ell)}\,\middle|\,\{\vect{\mu}_{y_j}\}_{j\in[n]}\right]\Bigr) &\simeq \mathrm{Var}\left(H_{ip}^{(\ell)}\Big|_{\mat{X}=\mathbf{0}} + \sum_{j,q} \frac{\partial H_{ip}^{(\ell)}}{\partial X_{jq}}\Bigg|_{\mat{X}=\mathbf{0}} \mu_{y_j,q}\right) \\
&= \mathrm{Var}\left(\sum_{j,q} \frac{\partial H_{ip}^{(\ell)}}{\partial X_{jq}}\Bigg|_{\mat{X}=\mathbf{0}} \mu_{y_j,q}\right) \\
&= \sum_{j,k}^{n} \sum_{q,r=1}^{d_{\mathrm{in}}} \frac{\partial H_{ip}^{(\ell)}}{\partial X_{jq}}\Bigg|_{\mat{X}=\mathbf{0}} \frac{\partial H_{ip}^{(\ell)}}{\partial X_{kr}}\Bigg|_{\mat{X}=\mathbf{0}} \mathrm{Cov}(\mu_{y_j,q}, \mu_{y_k,r}).
\end{align*}
Substituting the covariances of the signal terms $\Sigma_{qr} := \mathrm{Cov}(\mu_{y_j,q}, \mu_{y_k,r})$ for $y_j = y_k$ and  $\mathrm{Cov}(\mu_{y_j,q}, \mu_{y_k,r})=0$ for $y_j \neq y_k$:
\begin{equation}\label{eq:sub_signal_covariance}
\begin{aligned}
\mathrm{Var}\Bigl(\mathbb{E}\left[H_{ip}^{(\ell)} \,\middle|\, \{\vect{\mu}_{y_j}\}_{j\in[n]}\right]\Bigr) &\simeq \sum_{y_j = y_k} \sum_{q,r=1}^{d_{\mathrm{in}}} \frac{\partial H_{ip}^{(\ell)}}{\partial X_{jq}}\Bigg|_{\mat{X}=\mathbf{0}} \frac{\partial H_{ip}^{(\ell)}}{\partial X_{kr}}\Bigg|_{\mat{X}=\mathbf{0}} \Sigma_{qr}
\end{aligned}
\end{equation}
Recalling the definitions of signal and global sensitivity:
\begin{align}
\sigsenexp[i]{\ell}{p,q,r} &:= \sum_{y_j = y_k} \frac{\partial H_{ip}^{(\ell)}}{\partial X_{jq}} \frac{\partial H_{ip}^{(\ell)}}{\partial X_{kr}}\Bigg|_{\mat{X}=\mathbf{0}},
\end{align}
we can identify the signal sensitivity in the first term of the RHS of Eq. \eqref{eq:sub_signal_covariance} for when $y_j = y_k$, and reconstruct the second term for when $y_j \neq y_k$ by taking the difference of global sensitivity and signal sensitivity, giving:
\begin{align}
\mathrm{Var}\Bigl(\mathbb{E}\left[H_{ip}^{(\ell)} \,\middle|\, \{\vect{\mu}_{y_j}\}_{j\in[n]}\right]\Bigr) &\simeq \sum_{q,r=1}^{d_{\mathrm{in}}} \Sigma_{qr} \sigsenexp[i]{\ell}{p,q,r}\label{eq:final_numerator}
\end{align}
For the denominator of the SNR in Eq. \eqref{eq:snr_def_recall}, we compute the conditional variance of $H_{ip}^{(\ell)}$ given $\{\vect{\mu}_{y_j}\}_{j\in[n]}$ using the approximation in Eq. \eqref{eq:expanded_linear_mpnn_approximation}:
\begin{align*}
\mathrm{Var}\left(H_{ip}^{(\ell)} \,\middle|\, \{\vect{\mu}_{y_j}\}_{j\in[n]}\right) &\simeq \mathrm{Var}\left( H_{ip}^{(\ell)}\Big|_{\mat{X}=\mathbf{0}} + \sum_{j,q} \frac{\partial H_{ip}^{(\ell)}}{\partial X_{jq}}\Bigg|_{\mat{X}=\mathbf{0}} \mu_{y_j,q} + \sum_{j,q} \frac{\partial H_{ip}^{(\ell)}}{\partial X_{jq}}\Bigg|_{\mat{X}=\mathbf{0}} \gamma_q + \sum_{j,q} \frac{\partial H_{ip}^{(\ell)}}{\partial X_{jq}}\Bigg|_{\mat{X}=\mathbf{0}} \epsilon_{jq} \,\middle|\, \{\vect{\mu}_{y_j}\}_{j\in[n]} \right) \\
&= \mathrm{Var}\left(\sum_{j,q} \frac{\partial H_{ip}^{(\ell)}}{\partial X_{jq}}\Bigg|_{\mat{X}=\mathbf{0}} \gamma_q + \sum_{j,q} \frac{\partial H_{ip}^{(\ell)}}{\partial X_{jq}}\Bigg|_{\mat{X}=\mathbf{0}} \epsilon_{jq}\right) \\
&= \sum_{j,k}^{n} \sum_{q,r=1}^{d_{\mathrm{in}}} \frac{\partial H_{ip}^{(\ell)}}{\partial X_{jq}}\Bigg|_{\mat{X}=\mathbf{0}} \frac{\partial H_{ip}^{(\ell)}}{\partial X_{kr}}\Bigg|_{\mat{X}=\mathbf{0}} \mathrm{Cov}(\gamma_q, \gamma_r) \\
&+ \sum_{j,k}^{n} \sum_{q,r=1}^{d_{\mathrm{in}}} \frac{\partial H_{ip}^{(\ell)}}{\partial X_{jq}}\Bigg|_{\mat{X}=\mathbf{0}} \frac{\partial H_{ip}^{(\ell)}}{\partial X_{kr}}\Bigg|_{\mat{X}=\mathbf{0}} \mathrm{Cov}(\epsilon_{jq}, \epsilon_{kr}) \\
&= \sum_{q,r=1}^{d_{\mathrm{in}}} \Phi_{qr} \sum_{j,k}^{n} \frac{\partial H_{ip}^{(\ell)}}{\partial X_{jq}}\Bigg|_{\mat{X}=\mathbf{0}} \frac{\partial H_{ip}^{(\ell)}}{\partial X_{kr}}\Bigg|_{\mat{X}=\mathbf{0}} + \sum_{q,r=1}^{d_{\mathrm{in}}} \Psi_{qr} \sum_{j} \frac{\partial H_{ip}^{(\ell)}}{\partial X_{jq}}\Bigg|_{\mat{X}=\mathbf{0}} \frac{\partial H_{ip}^{(\ell)}}{\partial X_{jr}}\Bigg|_{\mat{X}=\mathbf{0}},
\end{align*}
where in the penultimate equality we use the definition of the covariances of the residuals and global shift terms, defined as $\Psi_{qr} := \mathrm{Cov}(\epsilon_{jq},\epsilon_{jr})$ and $\Phi_{qr}:=\mathrm{Cov}(\gamma_q, \gamma_r)$ respectively. In the last equality we can identify the noise sensitivity and global sensitivity defined respectively as:
\begin{align*}
  \noisesenexp[i]{\ell}{p,q,r} &:= \sum_{j,k=1}^{n} \frac{\partial H_{ip}^{(\ell)}}{\partial X_{jq}} \frac{\partial H_{ip}^{(\ell)}}{\partial X_{jr}}\Bigg|_{\mat{X}=\mathbf{0}}, \\
  \globsenexp[i]{\ell}{p,q,r} &:= \sum_{j,k=1}^{n} \frac{\partial H_{ip}^{(\ell)}}{\partial X_{jq}} \frac{\partial H_{ip}^{(\ell)}}{\partial X_{kr}}\Bigg|_{\mat{X}=\mathbf{0}},
\end{align*}

Taking the expectation over the signal variables:
\begin{align}\label{eq:final_denominator}
\mathbb{E}\Bigl[\mathrm{Var}\left(H_{ip}^{(\ell)}\, \middle|\, \{\vect{\mu}_{y_j}\}_{j\in[n]}\right)\Bigr] &\simeq \sum_{q,r=1}^{d_{\mathrm{in}}} \Phi_{qr} \globsenexp[i]{\ell}{p,q,r} + \sum_{q,r=1}^{d_{\mathrm{in}}} \Psi_{qr} \noisesenexp[i]{\ell}{p,q,r}.
\end{align}
Finally, the SNR is given by the ratio of the expression in Eq. \eqref{eq:final_numerator} in the numerator and Eq. \eqref{eq:final_denominator} in the denominator:
\begin{align*}
\snrsub{H_{ip}^{(\ell)}}{}
  \simeq
  \frac{\displaystyle
    \sum_{q,r=1}^{d_{\mathrm{in}}}
      \Sigma_{qr}\,
      \sigsenexp[i]{\ell}{p,q,r}
    }
       {\displaystyle
    \sum_{q,r=1}^{d_{\mathrm{in}}}
      \Phi_{qr}\,\globsenexp[i]{\ell}{p,q,r}
    \;+\;
    \sum_{q,r=1}^{d_{\mathrm{in}}}
      \Psi_{qr}\,\noisesenexp[i]{\ell}{p,q,r}}.
\end{align*}
\end{proof}

\sensitivityratiocondition*
\begin{proof}
We begin with Theorem~\ref{theorem:snrsensitivity}, which states that under the feature decomposition 
\begin{align*}
\mat{X}_j = \vect{\mu}_{y_j} + \vect{\gamma} + \vect{\epsilon}_j
\end{align*}
and the definitions of signal/global/noise sensitivities, the SNR of an $\ell$-layer MPNN’s output $H_{ip}^{(\ell)}$ at node $i$ and feature dimension $p$ satisfies:
\begin{align*}
\snrsub{H_{ip}^{(\ell)}}{}
  \simeq
  \frac{\displaystyle
    \sum_{q,r=1}^{d_{\mathrm{in}}}
      \Sigma_{qr}\,
      \sigsenexp[i]{\ell}{p,q,r}
    }
       {\displaystyle
    \sum_{q,r=1}^{d_{\mathrm{in}}}
      \Phi_{qr}\,\globsenexp[i]{\ell}{p,q,r}
    \;+\;
    \sum_{q,r=1}^{d_{\mathrm{in}}}
      \Psi_{qr}\,\noisesenexp[i]{\ell}{p,q,r}}.
\end{align*}
Under an IID\ assumption on feature dimensions each covariance matrix is diagonal, so we can write:
\begin{align*}
\Sigma_{qr} = \sigma^2\,\delta_{qr},
\quad
\Phi_{qr} = \phi^2\,\delta_{qr},
\quad
\Psi_{qr} = \psi^2\,\delta_{qr},
\end{align*}
where $\sigma, \phi, \psi$ are scalars. Summing over $q,r$ in the numerator and denominator then reduces the SNR expression to:
\begin{align*}
\snrsub{H_{ip}^{(\ell)}}{}
  \simeq
  \frac{\displaystyle
    \sigma^2 \sum_{q,r=1}^{d_{\mathrm{in}}}
      \sigsenexp[i]{\ell}{p,q,r}
    }
       {\displaystyle
    \phi^2\sum_{q,r=1}^{d_{\mathrm{in}}}
     \globsenexp[i]{\ell}{p,q,r}
    \;+\;
    \psi^2\sum_{q,r=1}^{d_{\mathrm{in}}}
      \noisesenexp[i]{\ell}{p,q,r}}.
\end{align*}

A non-relational feedforward model is limited to an SNR of $\frac{\sigma^2}{\phi^2 + \psi^2}$.
To say that the MPNN improves upon this baseline is to require that:
\begin{align*}
\frac{
\sigma^2 \,\sum_{q} \sigsenexp[i]{\ell}{p,q,q}
}{
\phi^2 \,\sum_{q} \globsenexp[i]{\ell}{p,q,q}
\;+\;
\psi^2 \,\sum_{q} \noisesenexp[i]{\ell}{p,q,q}
}
>
\frac{\sigma^2}{\phi^2 + \psi^2}.
\end{align*}

We have that by rearanging terms:
\begin{align*}
\sum_{q} \sigsenexp[i]{\ell}{p,q,q} > \frac{\psi^2}{\phi^2 + \psi^2} \sum_{q} \noisesenexp[i]{\ell}{p,q,q} + \left(1 - \frac{\psi^2}{\phi^2 + \psi^2} \right)\sum_{q} \globsenexp[i]{\ell}{p,q,q}
\end{align*}

Recalling that $\rho:=\frac{\psi^2}{\phi^2 + \psi^2},$ we obtain the inequality
\begin{align}\label{eq:final_condition_derivatation}
\sum_{q} \sigsenexp[i]{\ell}{p,q,q}
>
\rho \,\sum_{q} \noisesenexp[i]{\ell}{p,q,q}
+
(\,1 - \rho\,)\,\sum_{q} \globsenexp[i]{\ell}{p,q,q}.
\end{align}

As all the steps in this derivation are reversible, this proves that the condition in Eq. \eqref{eq:final_condition_derivatation} is necessary and sufficient for the MPNN to improve the SNR of the input features.
\end{proof}

\begin{restatable}[Bound for MPNN Jacobian]{lemma}{fulljacobian}
  \label{lemma:full_jacobian}
  Let $\nabla \mat{H}_{i}^{(\ell)}$ be the Jacobian of the $\nth{\ell}$ layer of an MPNN that uses the graph shift operator $\hat{\mat{A}}$ with message and update functions $\{M_k(\cdot,\cdot)\}_{k=1}^\ell$ and $\{U_k(\cdot,\cdot)\}_{k=1}^\ell$, as in Eq. \eqref{eq:mpnn}. Let $\norm{ \cdot }$ be the Euclidean norm, and $\nabla_1 f$ and $\nabla_2 f$ be the Jacobians of some function $f(\vect{x}_1,\vect{x}_2)$ with respect to $\vect{x}_1$ and $\vect{x}_2$, respectively. Assuming that there exist constants $\alpha_{1}, \alpha_{2}, \beta_{1}, \beta_{2}$ such that $\forall r\in[\ell]$ the message and update functions satisfy $\norm{ \nabla_1 U_r}\leq \alpha_{1}$, $\norm{ \nabla_2 U_r}\leq \alpha_{2}$, $\norm{\nabla_1 M_r}\leq \beta_{1}$, and $\norm{\nabla_2 M_r}\leq \beta_{2}$ then:
  \begin{align*}
    \left[\nabla H_{ip}^{(\ell)}\right]_{jq} \leq \left[\left(\alpha_{2} \beta_2 \hat{\mat{A}} + \alpha_{2}\beta_{1}\diag{\hat{\mat{A}}\ones{n}} + \alpha_{1} \eye{n}\right)^{\ell}\right]_{ij},
  \end{align*}
  where $\ones{n}$ is the size-$n$ vector of ones and $\eye{n}$ is the identity matrix of size $n$.\\\\
\end{restatable}

\begin{proof}
  Let $\left[\nabla \mat{H}_{i}^{(\ell)}\right]_{j} \in \mathbb{R}^{d_{\mathrm{in} \times d_{\mathrm{out}}}}$ be the Jacobian matrix between source node $j$ and target node $i$. By applying the chain rule to Eq. \eqref{eq:mpnn}, the Jacobian of the $\nth{\ell}$ MPNN layer is given by:
  \begin{align*}
    \left[\nabla \mat{H}_{i}^{(\ell)}\right]_{j} &= \nabla_1 U_\ell \left[\nabla{\mat{H}^{(\ell-1)}_i}\right]_j +\nabla_2 U_\ell \sum_{l\in N(i)}\hat{A}_{il}\left(\nabla_1M_\ell \left[\nabla{\mat{H}^{(\ell-1)}_i}\right]_{j} + \nabla_2M_\ell \left[\nabla{\mat{H}^{(\ell-1)}_l}\right]_{j}\right),\\
    &= \left(\nabla_1 U_\ell +\nabla_2U_\ell\nabla_1M_\ell\sum_{l\in N(i)}\hat{A}_{il}\right)\left[\nabla \mat{H}_{i}^{(\ell-1)}\right]_{j} +\nabla_2 U_\ell \sum_{l\in N(i)}\hat{A}_{il}\nabla_2M_\ell\left[\nabla \mat{H}_{l}^{(\ell-1)}\right]_{j}.
  \end{align*}
  Let $\norm{\cdot}_2$ be the induced 2-norm. By norm sub-additivity and sub-multiplicativity, we have:
  \begin{align*}
    \norm{\left[\nabla \mat{H}_{i}^{(\ell)}\right]_{j}}_2 &\le \left(\norm{\nabla_1 U_\ell}+\norm{\nabla_2 U_\ell}\norm{\nabla_1 M_\ell}\sum_{l\in N(i)}\hat{A}_{il}\right)\norm{ \left[\nabla \mat{H}_{i}^{(\ell-1)}\right]_{j}}_2 \\&\quad+ \norm{\nabla_2 U_\ell} \norm{\nabla_2M_\ell}\sum_{l\in N(i)}\hat{A}_{il}\norm{\left[\nabla \mat{H}_{l}^{(\ell-1)}\right]_{j}}_2\\
    & \le \left(\alpha_1+\alpha_2\beta_1\sum_{l\in N(i)}\hat{A}_{il}\right)\norm{\left[\nabla \mat{H}_{i}^{(\ell-1)}\right]_{j}}_2 + \alpha_2\beta_2\sum_{l\in N(i)}\hat{A}_{il}\norm{\left[\nabla \mat{H}_{l}^{(\ell-1)}\right]_{j}}_2\\
    & = \left[\left(\alpha_2\beta_2\hat{\mat{A}}\mat{J}^{(\ell-1)}\right)_{ij} + \left(\alpha_2\beta_1\diag{\hat{\mat{A}}\ones{n}}+\alpha_1 \eye{n}\right)_{ii} \mat{J}^{(\ell-1)}_{ij}\right],
  \end{align*}
  where $J^{(\ell)}_{ij}\coloneqq\norm{\left[\nabla \mat{H}_{i}^{(\ell)}\right]_{j}}_2$. The bound can be written as a single matrix multiplication
  
  \begin{align*}
    J^{(\ell)}_{ij}\coloneqq\norm{\left[\nabla \mat{H}_{i}^{(\ell)}\right]_{j}}_2 &\le \left[\left(\alpha_2\beta_2\hat{\mat{A}}+\alpha_2\beta_1\diag{\hat{\mat{A}}\ones{n}}+\alpha_1 \eye{n}\right) \mat{J}^{(\ell-1)}\right]_{ij},
  \end{align*}
  
  which when applied recursively yields
  \begin{align*}
  J^{(\ell)}_{ij} \le \left[\left(\alpha_2 \beta_2 \hat{\mat{A}} + \alpha_2\beta_1\diag{\hat{\mat{A}}\ones{n}} + \alpha_1 \eye{n}\right)^\ell \mat{J}^{(0)}\right]_{ij},
  \end{align*}
  We use the initial condition to obtain $\left[\nabla \mat{H}_{i}^{(0)}\right]_{j}=\delta_{ij} \eye{d_{\mathrm{in}}} \implies \mat{J}^{(0)} = \eye{n}$. The desired result follows as: 

  \begin{align}
    \left[\nabla H_{ip}^{(\ell)}\right]_{jq} \leq \norm{\left[\nabla \mat{H}_{i}^{(\ell)}\right]_{j}}_2 \leq \left[\left(\alpha_{2} \beta_2 \hat{\mat{A}} + \alpha_{2}\beta_{1}\diag{\hat{\mat{A}}\ones{n}} + \alpha_{1} \eye{n}\right)^{\ell}\right]_{ij}.
  \end{align}
\end{proof}

\sensitivitybound*

\begin{proof}
We begin by recalling from Lemma~\ref{lemma:full_jacobian} that for every node \(i\) and for each input node \(j\), the partial derivative of the output feature \(H_{ip}^{(\ell)}\) with respect to the input feature \(X_{jq}\) is bounded by
\begin{align}\label{eq:jacobian_bound}
\left|\frac{\partial H_{ip}^{(\ell)}}{\partial X_{jq}}\right| \le \left[\mat{K}^{\ell}\right]_{ij},
\end{align}
where the matrix \(\mat{K}\) is defined as
\begin{align*}
\mat{K} := \alpha_{2}\beta_{2}\,\hat{\mat{A}} + \alpha_{2}\beta_{1}\,\diag{\hat{\mat{A}}\ones{n}} + \alpha_{1}\,\eye{n}.
\end{align*}
Under the assumption that the upper bounds for the gradients of the message function satisfy \(\norm{\nabla_1 M_r},\,\norm{\nabla_2 M_r} \le \beta\) (so that we may set \(\beta_1 = \beta_2 = \beta\)), the matrix \(\mat{K}\) simplifies to
\begin{align}\label{eq:reduced_k_bound}
\mat{K} = \alpha_{1}\,\eye{n} + \alpha_{2}\beta\,\Bigl(\hat{\mat{A}} + \diag{\hat{\mat{A}}\ones{n}}\Bigr).
\end{align}
The signal sensitivity at node \(i\) is defined by
\begin{align}\label{eq:sigsen_def_thm}
\sigsenexp[i]{\ell}{p,q,r} = \sum_{j,k\in V} \frac{\partial H_{ip}^{(\ell)}}{\partial X_{jq}}\,\frac{\partial H_{ip}^{(\ell)}}{\partial X_{kr}}\big|_{\mat{X}=\mat{0}}\,\delta_{y_j y_k}.
\end{align}
Using the triangle inequality and applying the bound in Eq. \eqref{eq:jacobian_bound} on each derivative in Eq. \eqref{eq:sigsen_def_thm}, we obtain
\begin{align}\label{eq:double_k_sigsen_bound}
\abs{\sigsenexp[i]{\ell}{p,q,r}} \le \sum_{j,k\in V} \left[\mat{K}^{\ell}\right]_{ij}\,\left[\mat{K}^{\ell}\right]_{ik}\,\delta_{y_j y_k}.
\end{align}
Next, we note that the matrix \(\mat{K}^\ell\) in Eq. \eqref{eq:reduced_k_bound} can be expanded via the binomial theorem as
\begin{align*}
\mat{K}^{\ell} := \sum_{s=0}^{\ell} \binom{\ell}{s}\,\alpha_{1}^{\ell-s}\,(\alpha_{2}\beta)^{s}\,\Bigl(\hat{\mat{A}}+\diag{\hat{\mat{A}}\ones{n}}\Bigr)^{s}.
\end{align*}
Thus, for any node \(i\) and any other node \(j\), we have the entry‐wise bound
\begin{align*}
\left[\mat{K}^{\ell}\right]_{ij} \le \sum_{s=0}^{\ell} \binom{\ell}{s}\,\alpha_{1}^{\ell-s}\,(\alpha_{2}\beta)^{s}\,\Bigl[\Bigl(\hat{\mat{A}}+\diag{\hat{\mat{A}}\ones{n}}\Bigr)^{s}\Bigr]_{ij}.
\end{align*}
Multiplying the bounds for $\left[\mat{K}^{\ell}\right]_{ij}$ and $\left[\mat{K}^{\ell}\right]_{ik}$ together yields
\begin{align}\label{eq:double_derivative_bound}
\left[\mat{K}^{\ell}\right]_{ij}\,\left[\mat{K}^{\ell}\right]_{ik} \le \sum_{s=0}^{\ell}\sum_{t=0}^{\ell} \binom{\ell}{s}\,\binom{\ell}{t}\,\alpha_{1}^{2\ell-s-t}\,(\alpha_{2}\beta)^{s+t}\,\Bigl[\Bigl(\hat{\mat{A}}+\diag{\hat{\mat{A}}\ones{n}}\Bigr)^{s}\Bigr]_{ij}\,\Bigl[\Bigl(\hat{\mat{A}}+\diag{\hat{\mat{A}}\ones{n}}\Bigr)^{t}\Bigr]_{ik}.
\end{align}
Substituting the expression in Eq.\eqref{eq:double_derivative_bound} back into the bound for \(\sigsenexp[i]{\ell}{p,q,r}\) in Eq. \eqref{eq:double_k_sigsen_bound} and changing the order of summation, we obtain
\begin{align}\label{eq:sub_double_k_bound_into_sigsen}
\abs{\sigsenexp[i]{\ell}{p,q,r}} \le \sum_{s=0}^{\ell}\sum_{t=0}^{\ell} \binom{\ell}{s}\,\binom{\ell}{t}\,\alpha_{1}^{2\ell-s-t}\,(\alpha_{2}\beta)^{s+t} \sum_{j,k\in V} \Bigl[\Bigl(\hat{\mat{A}}+\diag{\hat{\mat{A}}\ones{n}}\Bigr)^{s}\Bigr]_{ij}\,\Bigl[\Bigl(\hat{\mat{A}}+\diag{\hat{\mat{A}}\ones{n}}\Bigr)^{t}\Bigr]_{ik}\,\delta_{y_j y_k}.
\end{align}
The inner sum over \(j\) and \(k\) with the indicator \(\delta_{y_j y_k}\) precisely defines the \(s,t\) order class-bottlenecking score of the graph shift operator \(\hat{\mat{A}}+\diag{\hat{\mat{A}}\ones{n}}\) at node \(i\), following Eq. \eqref{eq:bottleneck_homophily_connection}, which gives
\begin{align*}
\localhomo{s,t}{\hat{\mat{A}}+\diag{\hat{\mat{A}}\ones{n}}}{i} = \sum_{j,k\in V} \Bigl[\Bigl(\hat{\mat{A}}+\diag{\hat{\mat{A}}\ones{n}}\Bigr)^{s}\Bigr]_{ij}\,\Bigl[\Bigl(\hat{\mat{A}}+\diag{\hat{\mat{A}}\ones{n}}\Bigr)^{t}\Bigr]_{ik}\,\delta_{y_j y_k}.
\end{align*}
We can therefore rewrite the bound in Eq. \eqref{eq:sub_double_k_bound_into_sigsen} as

\begin{align}\label{eq:final_signsen_bound}
\abs{\sigsenexp[i]{\ell}{p,q,r}} \le \sum_{s=0}^{\ell}\sum_{t=0}^{\ell} \binom{\ell}{s}\,\binom{\ell}{t}\,\alpha_{1}^{2\ell-s-t}\,(\alpha_{2}\beta)^{s+t} \localhomo{s,t}{\hat{\mat{A}}+\diag{\hat{\mat{A}}\ones{n}}}{i}.
\end{align}
Eq. \eqref{eq:final_signsen_bound} is exactly the desired bound on the signal sensitivity. 
As for global sensitivity, recall that its definition
is given by:
\begin{align}\label{eq:recall_globsen_def}
\globsenexp[i]{\ell}{p,q,r} = \sum_{j,k\in V} \frac{\partial H_{ip}^{(\ell)}}{\partial X_{jq}} \frac{\partial H_{ip}^{(\ell)}}{\partial X_{mr}}\Bigg|_{\mat{X}=\mathbf{0}}.
\end{align}
Note that Eq. \eqref{eq:recall_globsen_def} involves the summation over all pairs of source nodes instead of just pairs of source nodes with the same class, without the Kronecker delta function as with signal sensitivity in Eq. \eqref{eq:sigsen_def_thm}. Applying the same bound as in Eq.~\eqref{eq:jacobian_bound} and following the same steps but without the Kronecker delta function, leads directly to an expression analogous to that obtained for the signal sensitivity in Eq. \eqref{eq:final_signsen_bound}---the only change is that the class-bottlenecking score $\localhomo{s,t}{\hat{\mat{A}}+\diag{\hat{\mat{A}}\ones{n}}}{i}$ is replaced by the total-bottlenecking $\localconn{s,t}{\hat{\mat{A}}+\diag{\hat{\mat{A}}\ones{n}}}{i}$, defined in Eq. \eqref{eq:p_order_local_homophily}, giving
\begin{align*}
\abs{\globsenexp[i]{\ell}{p,q,r}} \le \sum_{s=0}^{\ell}\sum_{t=0}^{\ell} \binom{\ell}{s}\,\binom{\ell}{t}\,\alpha_{1}^{2\ell-s-t}\,(\alpha_{2}\beta)^{s+t} \localconn{s,t}{\hat{\mat{A}}+\diag{\hat{\mat{A}}\ones{n}}}{i}.
\end{align*}

A similar extension applies for the noise sensitivity, defined as
\begin{align*}
\noisesenexp[i]{\ell}{p,q,r} = \sum_{j\in V} \frac{\partial H_{ip}^{(\ell)}}{\partial X_{jq}}\frac{\partial H_{ip}^{(\ell)}}{\partial X_{jr}},
\end{align*}
where the sum is over identical source nodes $j\in V$ instead of over node pairs $j,m\in V$ as in Eqs. \eqref{eq:sigsen_def_thm} and \eqref{eq:recall_globsen_def}; so the same bound on the partial derivatives (Eq. \eqref{eq:jacobian_bound}) gives rise to a corresponding sum in which the self-bottlenecking $\localselfconn{s,t}{\hat{\mat{A}}+\diag{\hat{\mat{A}}\ones{n}}}{i}$, defined in Eq. \eqref{eq:p_order_local_homophily}, replaces the class-bottlenecking score. Thus, we have
\begin{align*}
\abs{\noisesenexp[i]{\ell}{p,q,r}} \le \sum_{s=0}^{\ell}\sum_{t=0}^{\ell} \binom{\ell}{s}\,\binom{\ell}{t}\,\alpha_{1}^{2\ell-s-t}\,(\alpha_{2}\beta)^{s+t} \localselfconn{s,t}{\hat{\mat{A}}+\diag{\hat{\mat{A}}\ones{n}}}{i}.
\end{align*}

In the case of isotropic MPNN models, where we assume $\norm{\nabla_1 M_r}=0$, the contribution from the diagonal term vanishes. Consequently, the matrix $\mat{K}$ simplifies to
\begin{align*}
\mat{K} = \alpha_{1}\,\eye{n} + \alpha_{2}\beta\,\hat{\mat{A}},
\end{align*}
and following the exact same steps as before gives:
\begin{align*}
&\abs{\sigsenexp[i]{\ell}{p,q,r}} \le \sum_{s=0}^{\ell}\sum_{t=0}^{\ell} \binom{\ell}{s}\,\binom{\ell}{t}\,\alpha_{1}^{2\ell-s-t}\,(\alpha_{2}\beta)^{s+t} \localhomo{s,t}{\hat{\mat{A}}}{i},\\
&\abs{\globsenexp[i]{\ell}{p,q,r}} \le \sum_{s=0}^{\ell}\sum_{t=0}^{\ell} \binom{\ell}{s}\,\binom{\ell}{t}\,\alpha_{1}^{2\ell-s-t}\,(\alpha_{2}\beta)^{s+t} \localconn{s,t}{\hat{\mat{A}}}{i},\\
&\abs{\noisesenexp[i]{\ell}{p,q,r}} \le \sum_{s=0}^{\ell}\sum_{t=0}^{\ell} \binom{\ell}{s}\,\binom{\ell}{t}\,\alpha_{1}^{2\ell-s-t}\,(\alpha_{2}\beta)^{s+t} \localselfconn{s,t}{\hat{\mat{A}}}{i}.
\end{align*}

\end{proof}

\avgsensbound*

\begin{proof}
We begin by recalling from Theorem \ref{theorem:general_signal_sensntivity_bound} that for an isotropic MPNN the node-level signal sensitivity at node \(i\) is bounded by:
\begin{equation}\label{eq:step1}
\abs{\sigsenexp[i]{\ell}{p,q,r}} \leq \sum_{s,t=0}^{\ell} \binom{\ell}{s}\binom{\ell}{t}\,\alpha_1^{2\ell-s-t} (\alpha_2\beta)^{s+t}\,\localhomo{s,t}{\hat{\mat{A}}}{i}.
\end{equation}
Averaging both sides of Eq. \eqref{eq:step1} over all nodes \(i \in V:=[n]\) and using Jenson's inequality, yields
\begin{equation}\label{eq:step2}
\abs{\overbar{\sigsenexp[p]{\ell}{q,r}}} \leq \overbar{\abs{\sigsenexp[p]{\ell}{q,r}}} \leq \frac{1}{n} \sum_{i\in V} \sum_{s,t=0}^{\ell} \binom{\ell}{s}\binom{\ell}{t}\,\alpha_1^{2\ell-s-t} (\alpha_2\beta)^{s+t}\,\localhomo{s,t}{\hat{\mat{A}}}{i}.
\end{equation}
Since the sums over \(i\) and over the indices \(s,t\) can be interchanged, we rewrite Eq. \eqref{eq:step2} as
\begin{equation}\label{eq:step3}
\abs{\overbar{\sigsenexp[p]{\ell}{q,r}}} \leq \sum_{s,t=0}^{\ell} \binom{\ell}{s}\binom{\ell}{t}\,\alpha_1^{2\ell-s-t} (\alpha_2\beta)^{s+t} \left( \frac{1}{n} \sum_{i\in V} \localhomo{s,t}{\hat{\mat{A}}}{i} \right).
\end{equation}
Next, we note that by the definition of class-bottlenecking score,
\begin{align*}
\localhomo{s,t}{\hat{\mat{A}}}{i} \coloneqq \sum_{j,k\in V} \left[\hat{\mat{A}}^s\right]_{ij}\left[\hat{\mat{A}}^t\right]_{ik}\delta_{y_j y_k},
\end{align*}
so that averaging over \(i\) gives:
\begin{equation}\label{eq:step4}
\frac{1}{n} \sum_{i\in V} \localhomo{s,t}{\hat{\mat{A}}}{i} = \frac{1}{n}\sum_{j,k\in V}\left[\hat{\mat{A}}^{s+t}\right]_{jk}\delta_{y_j y_k} = \rhomo{s+t}{\hat{\mat{A}}}.
\end{equation}
Substituting Eq. \eqref{eq:step4} into Eq. \eqref{eq:step3} yields
\begin{align*}
\abs{\overbar{\sigsenexp[p]{\ell}{q,r}}} \leq \sum_{s,t=0}^{\ell} \binom{\ell}{s}\binom{\ell}{t}\,\alpha_1^{2\ell-s-t} (\alpha_2\beta)^{s+t}\,\rhomo{s+t}{\hat{\mat{A}}}.
\end{align*}
We now introduce the new index \(u=s+t\); for each fixed \(k\) the pairs \((s,t)\) contribute:
\begin{align*}
\sum_{\substack{s,t\geq0 \\ s+t=u}} \binom{\ell}{s}\binom{\ell}{t}=\binom{2\ell}{u},
\end{align*}
where the equality arises from Vandermonde's identity. This indexing allows rewriting the bound as:
\begin{equation}\label{eq:final_bound}
\abs{\overbar{\sigsenexp[p]{\ell}{q,r}}} \leq \sum_{u=0}^{2\ell} \binom{2\ell}{u}\,\alpha_1^{2\ell-u} (\alpha_2\beta)^u\,\rhomo{u}{\hat{\mat{A}}}.
\end{equation}
Eq. \eqref{eq:final_bound} is precisely the first inequality in Eq. \eqref{eq:avg_sensitivity_bound}.

The derivations for the average global and noise sensitivities follow by analogous arguments. In these cases the class-bottlenecking score \(\localhomo{s,t}{\hat{\mat{A}}}{i}\) is replaced respectively by the total-bottlenecking \(\localconn{s,t}{\hat{\mat{A}}}{i}\) and the self-bottlenecking \(\localselfconn{s,t}{\hat{\mat{A}}}{i}\) scores.
\end{proof}

\underreaching*
\begin{proof}
  The proof follows by considering the first-order asymptotic approximation of Eq. (26) in \cite{loomba2021geodesic}.
\end{proof}
\oversquashing*
\begin{proof}

  By definition, the symmetric normalised adjacency matrix is given by $ \hat{\mat{A}}_{\mathrm{sym}} := \mat{D}^{-\frac{1}{2}} \mat{A} \mat{D}^{-\frac{1}{2}}$ where $\mat{D}$ is the diagonal matrix of node degrees. The LHS of Eq. \eqref{eq:oversquashing} can then be written as:
  \begin{equation}\label{eq:condexpect_adjpower}
  \begin{split}\condexpect{\left[\hat{\mat{A}}^r_\mathrm{sym}\right]_{ij}}{\lambda_{ij} = r} &= \condexpect{\frac{1}{\sqrt{D_{ii}D_{jj}}}\sum_{k_1,k_2,\dots,k_{r-1}=1}^n\frac{A_{ik_1}A_{k_1k_2}\dots A_{k_{r-1}j}}{D_{k_1k_1}D_{k_2k_2}\dots D_{k_{r-1}k_{r-1}}}}{\lambda_{ij}=r}\\
  &=\sum_{\substack{k_1,k_2,\dots,k_{r-1}=1\\i\ne k_1\ne k_2 \dots \ne k_{r-1} \ne j}}^n\condexpect{\frac{1}{\sqrt{D_{ii}D_{jj}}}\frac{A_{ik_1}A_{k_1k_2}\dots A_{k_{r-1}j}}{D_{k_1k_1}D_{k_2k_2}\dots D_{k_{r-1}k_{r-1}}}}{\lambda_{ij}=r},
  \end{split}
  \end{equation}
  where we use the linearity of expectation and the fact that if the shortest path distance from $i$ to $j$ is $r$ then a walk of length $r$ from $i$ to $j$ via nodes $k_1, k_2,\dots k_{r-1}$ must be a path, i.e. $i\ne k_1\ne k_2\dots \ne k_{r-1}\ne j$. For brevity we will define $k_0\coloneqq i,k_r\coloneqq j$ and refer to the sequence $\{k_l\}_{l=0}^r$ as the length-$r$ path of interest. Given the definition of the adjacency matrix, we can write the conditional expectation on the RHS of Eq. \eqref{eq:condexpect_adjpower} as:
  \begin{align}\label{eq:condexpect_adjpower2}
    \condexpect{\left(\sqrt{D_{ii}D_{jj}}\prod_{l=1}^{r-1}D_{ll}\right)^{-1}}{\prod_{l=0}^{r-1}A_{k_lk_{l+1}}=1,\lambda_{ij}=r}\condprob{\prod_{l=0}^{r-1}A_{k_lk_{l+1}}=1}{\lambda_{ij}=r}.
  \end{align}
  
  Consider the first factor in Eq. \eqref{eq:condexpect_adjpower2}. Knowing that $\prod_{l=0}^{r-1}A_{k_lk_{l+1}}=1$ tell us that there \emph{must exist} edges between nodes $k_l$ and $k_{l+1}$. Knowing further that $\lambda_{ij}=r$ tell us that the path $\{k_l\}_{l=0}^r$ is a shortest path, i.e. there \emph{cannot exist} paths shorter than length $m$ between nodes $k_l$ and $k_{l+m}$. Asymptotically, the probability of paths shorter than length $m$ (for any finite $m$) \emph{not existing} between any two nodes in a sparse graph is already $1-\smallo{1}$ (see Lemma \ref{lemma:underreaching} or \cite{loomba2021geodesic}), i.e. the latter asymptotically does not inform the expectation of our quantity of interest. Furthermore, since edges are added between every node pair (conditionally) independently they affect---and can \emph{only} affect---the degree of the nodes to which the edges are attached. This, alongside the fact that every node in the path $\{k_l\}_{l=0}^r$ is unique, permits us to asymptotically approximate the first factor in Eq. \eqref{eq:condexpect_adjpower2} as:
  \begin{equation*}
    \condexpect{D_{ii}^{-\frac{1}{2}}}{A_{ik_1}}\condexpect{D_{jj}^{-\frac{1}{2}}}{A_{k_{r-1}j}}\prod_{l=1}^{r-1}\condexpect{D_{k_lk_l}^{-1}}{A_{k_{l-1}k_l}A_{k_lk_{l+1}}}.
  \end{equation*}
  Asymptotically, ignoring a single or two nodes has a vanishing effect on the degree of another node in a sparse graph with (conditionally) independent edges. In other words, knowing about the existence of a single or two edges attached to a given node merely shifts its degree distribution by one or two, respectively:
  \begin{equation*}
  \begin{split}
    \condexpect{D_{ii}^{-\frac{1}{2}}}{A_{ik_1}} &\approx \expect{\left(D_{ii}+1\right)^{-\frac{1}{2}}},\\
    \condexpect{D_{jj}^{-\frac{1}{2}}}{A_{k_{l-1}j}} &\approx \expect{\left(D_{jj}+1\right)^{-\frac{1}{2}}},\\
    \condexpect{D_{k_lk_l}^{-1}}{A_{k_{l-1}k_l}A_{k_lk_{l+1}}} &\approx \expect{\left(D_{k_lk_l}+2\right)^{-1}}.
  \end{split}
  \end{equation*}
  Asymptotically, the degree of a given node in a sparse graph with (conditionally) independent edges is Poisson distributed whose rate is given by its mean degree \cite{loomba2021geodesic}. This allows us to apply the results in Eqs. \eqref{eq:poisson_moments_xplus2} and \eqref{eq:poisson_moments_sqrtxplus1} in Proposition \ref{theorem:poisson_transformation} to write the first factor of Eq. \eqref{eq:condexpect_adjpower2} as:
  \begin{equation}\label{eq:condexpect_adjpower2_factor1}
  \begin{split}
    &\condexpect{\left(\sqrt{D_{ii}D_{jj}}\prod_{l=1}^{r-1}D_{ll}\right)^{-1}}{\prod_{l=0}^{r-1}A_{k_lk_{l+1}}=1,\lambda_{ij}=r}\lessapprox\left(\avg{D_{ii}}\avg{D_{jj}}\right)^{-\frac{1}{2}}\\&\times\prod_{l=1}^{r-1}\left(\avg{D_{k_lk_l}}^{-1} - \avg{D_{k_lk_l}}^{-2}\left(1-e^{-\avg{D_{k_lk_l}}}\right)\right),
    \end{split}
  \end{equation}
  and the bound is tight for large node mean degrees. Consider the second factor in Eq. \eqref{eq:condexpect_adjpower2} that can be rewritten as:
  \begin{align*}
    \condprob{\lambda_{ij}=r}{\prod_{l=0}^{r-1}A_{k_lk_{l+1}}=1}\frac{\prob{\prod_{l=0}^{r-1}A_{k_lk_{l+1}}=1}}{\prob{\lambda_{ij}=r}}.
  \end{align*}
  Knowing that $\prod_{l=0}^{r-1}A_{k_lk_{l+1}}=1$, i.e. there exists a path of length $r$ between $i$ and $j$, tell us that the shortest path between $i$ and $j$ cannot be longer than $r$. Asymptotically, it tells us nothing about whether there exists a path shorter than length $r$ between them. Since, \emph{a priori}, the probability of the shortest path being less than length $r$ is asymptotically vanishing (see Lemma \ref{lemma:underreaching} or \cite{loomba2021geodesic}), this implies that $ \condprob{\lambda_{ij}=r}{\prod_{l=0}^{r-1}A_{k_lk_{l+1}}=1}=1-\smallo{1}$. Finally, due to conditional independence of edges, and considering the first-order approximation of Eq. (25) in \cite{loomba2021geodesic}, allows us to write the second factor of Eq. \eqref{eq:condexpect_adjpower2} as:
  \begin{align}\label{eq:condexpect_adjpower2_factor2}
    \condprob{\prod_{l=0}^{r-1}A_{k_lk_{l+1}}=1}{\lambda_{ij}=r}\approx \frac{\prod_{l=0}^{r-1}\expect{A_{k_lk_{l+1}}}}{\left[\expect{\mat{A}}^r\right]_{ij}}.
  \end{align}
  Putting Eqs. \eqref{eq:condexpect_adjpower2_factor1} and \eqref{eq:condexpect_adjpower2_factor2} in Eq. \eqref{eq:condexpect_adjpower} yields:
  \begin{subequations}
  \begin{align}\label{eq:condexpect_adjpower3}     \condexpect{\left[\hat{\mat{A}}^r_\mathrm{sym}\right]_{ij}}{\lambda_{ij} = r}&\lessapprox\frac{\left(\avg{D_{ii}}\avg{D_{jj}}\right)^{-\frac{1}{2}}}{\left[\expect{\mat{A}}^r\right]_{ij}}\sum_{\substack{k_1,k_2,\dots,k_{r-1}=1\\i\ne k_1\ne k_2 \dots \ne k_{r-1} \ne j}}^n S\left(i,j,\{k_l\}_{l=1}^{r-1}\right), \textrm{ where}\\\label{eq:condexpect_adjpower3_S} 
    S\left(i,j,\{k_l\}_{l=1}^{r-1}\right)&\coloneqq\expect{A_{ik_1}}\prod_{l=1}^{r-1}\left(\avg{D_{k_lk_l}}^{-1} - \avg{D_{k_lk_l}}^{-2}\left(1-e^{-\avg{D_{k_lk_l}}}\right)\right)\expect{A_{k_lk_{l+1}}}.
  \end{align}
  \end{subequations}
  Consider the term on the RHS of Eq. \eqref{eq:condexpect_adjpower3_S}. Due to the sparsity assumption $\expect{\mat{A}}=\bigOold{n^{-1}}$ we have $S\left(i,j,\{k_l\}_{l=1}^{r-1}\right)=\bigOold{n^{-r}}$. We separately consider what happens when $S(i,j,\{k_l\}_{l=1}^{r-1})$ is summed over different kinds of index combinations $\{k_l\}_{l=1}^{r-1}$.
  
  First, consider unique index combinations $\{k_l\}_{l=1}^{r-1}$ of size $r-1$ from $[n]\setminus\{i,j\}$, as in the RHS of Eq. \eqref{eq:condexpect_adjpower3} since $\{k_l\}_{l=0}^{r}$ encodes a shortest path. There are $\frac{(n-2)!}{(n-r-1)!}=\bigOold{n^{r-1}}$ such index combinations which yields a total contribution of order $\bigOold{n^{-1}}$ to the RHS of Eq. \eqref{eq:condexpect_adjpower3}.

  Next, consider unique index combinations $\{k_l\}_{l=1}^{r-1}$ of size $r-1$ from $[n]$, such that exactly one of the $r-1$ indices is either $i$ or $j$, which \emph{do not} contribute to the RHS of Eq. \eqref{eq:condexpect_adjpower3}. There are $2(r-1)\frac{(n-2)!}{(n-r)!}=\bigOold{n^{r-2}}$ such index combinations which yields a total contribution of $\bigOold{n^{-2}}$.

  Now, consider unique index combinations $\{k_l\}_{l=1}^{r-1}$ of size $r-1$ from $[n]$, such that exactly one of the $r-1$ indices is $i$ and exactly another one is $j$, which \emph{do not} contribute to the RHS of Eq. \eqref{eq:condexpect_adjpower3}. There are $(r-1)(r-2)\frac{(n-2)!}{(n-r+1)!}=\bigOold{n^{r-3}}$ such index combinations which yields a total contribution of $\bigOold{n^{-3}}$.
  
  Finally, consider non-unique index combinations $\{k_l\}_{l=1}^{r-1}$ of size $r-1$ from $n$, such that there are $1\le m<r-1$ unique indices in the sequence $\{k_l\}_{l=1}^{r-1}$ repeated $t_1, t_2, \dots, t_m$ number of times such that $\forall l\in[m]: t_l\ge 1$ and $\sum_{l=1}^mt_l=r-1$, which \emph{do not} contribute to the RHS of Eq. \eqref{eq:condexpect_adjpower3}. There can be $\frac{(r-1)!}{t_1!t_2!\dots t_m!}\frac{n!}{(n-m)!} = \bigOold{n^{m}}$ such index combinations which yields a total contribution of $\bigOold{n^{-r+m}}$. Since $1\le m<r-1$, considering a sum over all possible values of $m$ yields a total contribution of all non-unique index combinations as $\bigOold{n^{-2}}$.
  
  This exhausts all possible index combinations, which leads us to conclude that asymptotically only the unique index combinations contribute relatively non-vanishingly. In other words, replacing the sum over \emph{unique} index combinations by a sum over \emph{all} index combinations makes a vanishing difference to the RHS of Eq. \eqref{eq:condexpect_adjpower3}, allowing us to rewrite it as a product of matrices which yields the RHS of Eq. \eqref{eq:condexpect_adjpower}.
\end{proof}
\underreachingoversquashing*
\begin{proof}
As in Lemma \ref{lemma:underreaching}, by considering the first-order asymptotic approximation of Eq. (25) in \cite{loomba2021geodesic}, we have that for an undirected and simple graph $G$ with adjacency matrix $\mat{A}$ sampled from a general random graph family with conditionally independent edges and expected adjacency matrix $\expect{\mat{A}}$, under the sparsity conditions, the cumulative distribution function of the shortest path length $\lambda_{ij}$ between nodes $i$ and $j$ is approximately:
\begin{align}
\prob{\lambda_{ij}\leq r} \approx 1-\mathrm{exp}\left(-\left[ \sum_{s=1}^r \expect{\mat{A}}^s \right]_{ij}\right) \approx \left[ \sum_{s=1}^r \expect{\mat{A}}^s \right]_{ij},
\end{align}
where $\approx$ here means a first-order approximation up to order $o\left(\frac{1}{n}\right)$.

Therefore, the probability that the shortest path length between $i$ and $j$ is exactly $r$ is given by:
\begin{align}
\prob{\lambda_{ij} = r} = \prob{\lambda_{ij}\leq r} - \prob{\lambda_{ij}\leq r-1} \approx \left[ \expect{\mat{A}}^r \right]_{ij},
\end{align}
which provides the first part of the theorem.

For the second part, using Lemma \ref{lemma:oversquashing}, and assuming large expected degrees but still much smaller than the number of nodes, i.e. $\avg{d}$ is large whilst $\avg{d} = o(n)$, the boundary oversquashing between nodes $i$ and $j$ from Eq. \eqref{eq:oversquashing} is asymptotically bounded by:

\begin{align}\label{eq:assymptotic_boundary_oversquashing}
\condexpect{\left[ \hat{\mat{A}}^r_\mathrm{sym} \right]_{ij} }{ \lambda_{ij} = r } \lessapprox \frac{\left[ \avg{\mat{D}}^{-\frac{1}{2}} \expect{\mat{A}} \left( \avg{\mat{D}}^{-1} \expect{\mat{A}} \right)^{r-1} \avg{\mat{D}}^{-\frac{1}{2}} \right]_{ij} + O\left( \frac{1}{n\avg{d}} \right)}{\left[ \expect{\mat{A}}^r \right]_{ij}},
\end{align}

as a result of combining any higher-order degree terms into $O\left( \frac{1}{n\avg{d}}\right)$. Rewriting the numerator of Eq. \eqref{eq:assymptotic_boundary_oversquashing} as $\left[ \avg{\mat{D}}^{-\frac{1}{2}} \expect{\mat{A}} \left( \avg{\mat{D}}^{-1} \expect{\mat{A}} \right)^{r-1} \avg{\mat{D}}^{-\frac{1}{2}} \right]_{ij} = \left[ \left( \avg{\mat{D}}^{-\frac{1}{2}} \expect{\mat{A}} \avg{\mat{D}}^{-\frac{1}{2}} \right)^r \right]_{ij}$, we obtain:

\begin{align}\label{eq:better_assymptotic_boundary_oversquashing}
\condexpect{\left[ \hat{\mat{A}}^r_\mathrm{sym} \right]_{ij} }{ \lambda_{ij} = r } \lessapprox \frac{\left[ \left( \avg{\mat{D}}^{-\frac{1}{2}} \expect{\mat{A}} \avg{\mat{D}}^{-\frac{1}{2}} \right)^r \right]_{ij} + O\left( \frac{1}{n\avg{d}} \right)}{\left[ \expect{\mat{A}}^r \right]_{ij}}.
\end{align}

As $\left[\expect{\mat{A}}^r \right]_{ij}= O\left( \frac{\avg{d}^{r}}{n} \right)$, then Eq. \eqref{eq:better_assymptotic_boundary_oversquashing} becomes:

\begin{align*}
& \frac{\left[ \left( \avg{\mat{D}}^{-\frac{1}{2}} \expect{\mat{A}} \avg{\mat{D}}^{-\frac{1}{2}} \right)^r \right]_{ij} + O\left( \frac{1}{n\avg{d}} \right)}{\left[ \expect{\mat{A}}^r \right]_{ij}} = \frac{\left[ \left( \avg{\mat{D}}^{-\frac{1}{2}} \expect{\mat{A}} \avg{\mat{D}}^{-\frac{1}{2}} \right)^r \right]_{ij}}{\left[ \expect{\mat{A}}^r \right]_{ij}} + O\left( \frac{1}{\avg{d}^{r+1}} \right).
\end{align*}

Now, let us consider the case where $\lambda_{ij} = t < r$. To analyse this scenario, we can decompose the conditional expectation as a sum over all possible walks of length $r$ from node $i$ to node $j$:

\begin{align}\label{eq:ungrouped_terms}
& \condexpect{\left[\hat{\mat{A}}^r_\mathrm{sym}\right]_{ij}}{\lambda_{ij} = t} = \sum_{w \in \mathcal{W}_{ij}^r} \prob{w \in W_G \mid \lambda_{ij} = t} \condexpect{\prod_{(u,v) \in w} \frac{1}{\sqrt{D_{u} D_{v}}}}{w \in W_G, \lambda_{ij} = t}
\end{align}

where $\mathcal{W}_{ij}^r$ represents the set of all possible walks of length $r$ from $i$ to $j$, $w$ is the sequence of edges $(u,v)$ in the walk, and $W_G$ is the set of all walks in a given graph $G$. We can further refine this sum by grouping walks according to their number of unique edges, denoting each walk as $w_s$, of length $r$ (possibly repeated) edges, but where $s = |w_s|$ is its number of unique edges. Using this grouping walks, the sum in Eq. \eqref{eq:ungrouped_terms} can be written as:

\begin{align}\label{eq:grouped_terms}
 \sum_{s=t}^r \sum_{\substack{w_s \in \mathcal{W}_{ij}^r\\|w_s|=s}} \prob{w_s \in W_G \mid \lambda_{ij} = t} \condexpect{\prod_{(u,v) \in w_s} \frac{1}{\sqrt{D_{u} D_{v}}}}{w_s \in W_G, \lambda_{ij} = t}
\end{align}

To approximate Eq. \eqref{eq:grouped_terms}, we separately consider the two factors that appear in each term of the sum: firstly the normalisation factors $\condexpect{\prod_{(u,v) \in w_s} \frac{1}{\sqrt{D_{u} D_{v}}}}{w_s \in W_G, \lambda_{ij} = t}$ and then the probability of occurrence of paths $\prob{w_s \in W_G \mid \lambda_{ij} = t} $. 

The normalisation factors $\condexpect{\prod_{(u,v) \in w_s} \frac{1}{\sqrt{D_{u} D_{v}}}}{w_s \in W_G, \lambda_{ij} =t}$ can be written as a product of conditional expectations over the unique nodes in the walk $w_s$, just as in the proof of Lemma \ref{lemma:oversquashing}, in total contributing a factor that scales with $O\left( \frac{1}{\avg{d}^{r}} \right)$, which can be shown by firstly writing the product as:

\begin{align}
  & \condexpect{\prod_{(u,v) \in w_s} \frac{1}{\sqrt{D_{u} D_{v}}}}{w_s \in W_G, \lambda_{ij} = t} = \condexpect{\frac{1}{\sqrt{D_{i} D_{j}}}\prod_{u=1,p_1+\dots+p_s=r-1}^{s-1} \frac{1}{D_u^{p_{u}}}}{w_s \in W_G, \lambda_{ij} = t},
\end{align}
where the sequence of integers $(p_u)$ represents the number of crossings of the particular walk $w_s$ through each node $u$ along the walk (excluding the endpoints of the walk $i$ and $j$). The total number of such crossings in a walk of length $r$ is $r-1$, hence their sum is $p_1+\dots+p_s=r-1$. Here, knowledge that the walk $w_s$ passes through any node $u$ means that the degree of node $u$ must be increased by at least 1 over the non-conditional node degree (depending on the nature of the walk, knowledge of the walk could increase the node's degree even more). Therefore:
\begin{align}
  & \condexpect{\prod_{(u,v) \in w_s} \frac{1}{\sqrt{D_{u} D_{v}}}}{w_s \in W_G, \lambda_{ij} = t}\\
  & \leq \expect{\frac{1}{\sqrt{D_{i}+1}}} \expect{\frac{1}{\sqrt{D_{j}+1}}}\prod_{u=1,p_1+\dots+p_s=r-1}^{s-1} \expect{\frac{1}{(1+D_u)^{p_{u}}}},
\end{align}
where, asymptotically, the degree of a given node in a sparse graph with (conditionally) independent edges is Poisson distributed whose rate is given by its mean degree. Using Eq. \eqref{eq:poisson_moments_sqrtxplus1} from Preposition \ref{theorem:poisson_transformation} for the factors outside the product, and the general result in Preposition \ref{theorem:general_poisson_transformation} for the factors inside the product, we can bound the individual expectations in the product as:
\begin{align}
  & \leq \frac{1}{\sqrt{\avg{D}_u \avg{D}_v}} \prod_{u=1,p_1+\dots+p_s=r-1}^{s-1} \frac{1}{\avg{D}_u^{p_u}} =  \bigOold{\frac{1}{\avg{d}^r}}
\end{align}

Now, following the same argument given for Eq. \eqref{eq:condexpect_adjpower2_factor2} in Lemma \ref{lemma:oversquashing}, for a simple path $w_t$ between nodes $i$ and $j$ using exactly $t$ unique edges, i.e. a walk with no cycles or backtracking of length $t$, knowledge that the shortest path between $i$ and $j$ is of length $t$ tells us that the shortest path between them cannot be longer than $t$. Asymptotically, it tells us nothing about whether there exists a path shorter than length $t$ between nodes $i$ and $j$. Since, \emph{a priori}, the probability of the shortest path being less than length $t$ is asymptotically vanishing (see Lemma \ref{lemma:underreaching} or \cite{loomba2021geodesic}), this implies that $ \condprob{\lambda_{ij}=t}{w_t \in W_G}=1-\smallo{1}$. Finally, due to conditional independence of edges, and considering the first-order approximation of Eq. (25) in \cite{loomba2021geodesic}, the probability of the existence of such a path is given by:

\begin{align}
   \prob{w_t \in W_G \mid \lambda_{ij} = t} & = \condprob{\prod_{\substack{l=0: (k_l,k_{l+1})\in w_t}}^{t-1}A_{k_lk_{l+1}}=1}{\lambda_{ij}=t} \\ 
   & \approx \frac{\prod_{\substack{l=0: (k_l,k_{l+1})\in w_t}}^{t-1}\expect{A_{k_lk_{l+1}}}}{\left[\expect{\mat{A}}^t\right]_{ij}} = \frac{\left(\frac{\avg{d}}{n}\right)^{{t}}}{\frac{\avg{d}^{t}}{n}}= O\left(\frac{1}{n^{t-1}}\right),
\end{align}

where $(k_l,k_{l+1})$ are the unique edges in a given path $w_t$. The number of such paths is at most $n^{t-1}$, as each walk can visit at most $n$ intermediate nodes, between nodes $i$ and $j$, $t-1$ times.

Walks $w_s$ between nodes $i$ and $j$ that use $s$ unique edges where $t < s \leq r$, can exist but these contribute a vanishing amount to the sum, as again by Loomba and Jones \cite{loomba2021geodesic}, knowledge that there exists a path of length $s>t$ between nodes $i$ and $j$ asymptotically tells us nothing about whether the shortest path is of length $t$ between them, and vice versa; thus:

\begin{align}
   \prob{w_s \in W_G \mid \lambda_{ij} = t} & = \condprob{\prod_{\substack{l=0: (k_l,k_{l+1})\in w_s}}^{s-1}A_{k_lk_{l+1}}=1}{\lambda_{ij}=t} \\
   & \approx \prod_{\substack{l=0: (k_l,k_{l+1})\in w_s}}^{s-1}\expect{A_{k_lk_{l+1}}} = \bigOold{\frac{\avg{d}^s}{n^s}}
\end{align}

where $(k_l,k_{l+1})$ are the unique edges in a given walk $w_s$. There are at most $n^{s-1}$ such walks, as each walk can visit at most $n$ intermediate nodes, between nodes $i$ and $j$, $s-1$ times. 

Combining all these terms, we obtain our desired result, by expressing the full conditional expectation as:
\begin{align}
\condexpect{\left[ \hat{\mat{A}}^r_{\mathrm{sym}} \right]_{ij}}{\lambda_{ij} = t} & \leq 
\underbrace{n^{t-1}}_{\text{number of paths}} \times 
\underbrace{O(1/n^{t-1})}_{\text{path probability}} \times 
\underbrace{O(1/\avg{d}^{r})}_{\text{normalisation}} \\
&+ \sum_{s=t}^{r}\underbrace{n^{s-1}}_{\text{number of walks}} \times 
\underbrace{O(\avg{d}^{s}/n^{s})}_{\text{walk probability}} \times 
\underbrace{O(1/\avg{d}^{r})}_{\text{normalisation}} \\
& \approx O\left( \frac{1}{\avg{d}^r} \right)+ \bigOold{\frac{1}{n}}
\end{align}

\end{proof}

\sbmhomophily*

\begin{proof}
  For brevity throughout the proof, we drop the subscript $\mathrm{sym}$ and use $\hat{\mat{A}}$ to refer to $\hat{\mat{A}}_\mathrm{sym}$. Given the block membership $\hat{y}_i,\hat{y}_j$ of nodes $i\ne j$, we have $\left[\expect{\mat{A}}^r\right]_{ij}=\left[\mat{B}(\mat{\Pi B})^{r-1}\right]_{ij}/n$. First, consider Eq. \eqref{eq:underover} with $r=1$, i.e. $\expect{\hat{\mat{A}}}$ which is given by:
  \begin{align}
  \label{eq:expect_adj1_final}
    \expect{\hat{A}_{ij}} = \condexpect{\hat{A}_{ij}}{\lambda_{ij}=0}\prob{\lambda_{ij}=0}+\condexpect{\hat{A}_{ij}}{\lambda_{ij}=1}\prob{\lambda_{ij}=1}\lessapprox n^{-1}D_{\hat{y}_i\hat{y}_i}^{-\frac{1}{2}}B_{\hat{y}_i\hat{y}_j} D_{\hat{y}_j\hat{y}_j}^{-\frac{1}{2}},
  \end{align}
  where (a) for $\lambda_{ij}=0 \implies i=j$ we use the fact that there are \emph{no} self-loops i.e. $A_{ij}=0\implies \hat{A}_{ij}=0$, and (b) for $\lambda_{ij}=1\implies i\ne j$ we use Lemma \ref{lemma:underreaching} and Lemma \ref{lemma:oversquashing} with $r=1$, and the bound gets tighter for larger class-wise mean degrees.

  Next, consider Eq. \eqref{eq:underover} with $r=2$, i.e. $\expect{\hat{\mat{A}}^2}$ which is given by:
  \begin{equation}
  \label{eq:expect_adj2}
    \expect{\left[\mat{\hat{A}}^2\right]_{ij}} = \sum_{s=0}^2\condexpect{\left[\mat{\hat{A}}^2\right]_{ij}}{\lambda_{ij}=s}\prob{\lambda_{ij}=s}.
  \end{equation}
  For $\lambda_{ij}=0\implies i=j$, using $d_i$ to denote the degree of node $i$, we get 
  \begin{equation}
    \label{eq:expect_adj2_lambda0}
    \begin{split}
    \expect{\left[\hat{\mat{A}}^2\right]_{ii}}&=\expect{\sum_{j}(d_{i}d_j)^{-1}A_{ij}}=\sum_j\expect{(d_id_j)^{-1}A_{ij}}\\&=\sum_j\condexpect{(d_id_j)^{-1}}{A_{ij}=1}\prob{A_{ij}=1}\\&\approx\sum_j\expect{(d_i+1)^{-1}}\expect{(d_j+1)^{-1}}\prob{A_{ij}=1}\lessapprox\sum_j \expect{d_i}^{-1}\expect{d_j}^{-1}\prob{A_{ij}=1} \\&= D_{\hat{y}_i\hat{y}_i}^{-1} [\mat{B}]_{\hat{y}_i,:} \mat{D}^{-1}\vect{\pi},
    \end{split}    
  \end{equation}
  where the second equality makes use of the linearity of expectation, the asymptotic approximation is due to an identical argument as in the proof for Lemma \ref{lemma:oversquashing} for sparse networks, the bound is due to Eq. \eqref{eq:poisson_moments_xplus1} in Proposition \ref{theorem:poisson_transformation} which becomes tighter for larger class-wise mean degrees, and $[\mat{x}]_{u,:}$ indicates the $\nth{u}$ row-vector of a matrix $\mat{x}$. For $\lambda_{ij}=1\implies i\ne j$ we get:
  \begin{equation}
    \label{eq:expect_adj2_lambda1}
    \begin{split}
      \condexpect{\left[\hat{\mat{A}}^2\right]_{ij}}{\lambda_{ij}=1} &= \condexpect{\left[\hat{\mat{A}}^2\right]_{ij}}{A_{ij}=1} = \condexpect{\sum_l(d_id_j)^{-\frac{1}{2}}d_l^{-1}A_{il}A_{lj}}{A_{ij}=1}\\
      &= \sum_l\condexpect{(d_id_j)^{-\frac{1}{2}}d_l^{-1}}{A_{il}A_{lj}A_{ij}=1}\condprob{A_{il}=1,A_{lj}=1}{A_{ij}=1}\\&=\sum_l\condexpect{(d_id_j)^{-\frac{1}{2}}d_l^{-1}}{A_{il}A_{lj}A_{ij}=1}\prob{A_{il}=1}\prob{A_{lj}=1},
    \end{split}
  \end{equation}
  where the third equality makes use of the linearity of expectation, and the fourth equality uses the assumption of conditionally independent edges. We emphasise that, due to sparsity, the RHS of Eq. \eqref{eq:expect_adj2_lambda1} is of the order $\bigOold{n^{-1}}$. For $\lambda_{ij}=2\implies i\ne j$ we get, using Eq. \eqref{eq:oversquashing} from Lemma \ref{lemma:oversquashing}:
  \begin{equation}
    \label{eq:expect_adj2_lambda2}
     \condexpect{\left[\hat{\mat{A}}^2\right]_{ij}}{\lambda_{ij}=2} \lessapprox \frac{\left(D_{\hat{y}_i\hat{y}_i}D_{\hat{y}_j\hat{y}_j}\right)^{-\frac{1}{2}}\left[\mat{B}\left\{\mat{D}^{-1}-\mat{D}^{-2}\left(\eye{k}-e^{-\mat{D}}\right)\right\}\mat{\Pi B}\right]_{\hat{y}_i\hat{y}_j}}{\left[\mat{B\Pi B}\right]_{\hat{y}_i\hat{y}_j}},
  \end{equation}
  and the bound gets tighter for larger degrees. The RHS of Eq. \eqref{eq:expect_adj2_lambda2} is of the order $\bigOmega{1}$. That is, asymptotically, Eq. \eqref{eq:expect_adj2_lambda1} contributes vanishingly to Eq. \eqref{eq:expect_adj2} when compared to Eq. \eqref{eq:expect_adj2_lambda2}. It then follows from Eqs. \eqref{eq:expect_adj2}, \eqref{eq:expect_adj2_lambda0}, and \eqref{eq:expect_adj2_lambda2} that asymptotically:
  \begin{align}
    \label{eq:expect_adj2_final}
    \expect{\left[\mat{\hat{A}}^2\right]_{ij}}\lessapprox D_{\hat{y}_i\hat{y}_i}^{-1} [\mat{B}]_{\hat{y}_i,:} \mat{D}^{-1}\vect{\pi}\delta_{ij}+\frac{\left(D_{\hat{y}_i\hat{y}_i}D_{\hat{y}_j\hat{y}_j}\right)^{-\frac{1}{2}}\left[\mat{B}\left\{\mat{D}^{-1}-\mat{D}^{-2}\left(\eye{k}-e^{-\mat{D}}\right)\right\}\mat{\Pi B}\right]_{\hat{y}_i\hat{y}_j}}{n}(1-\delta_{ij}),
  \end{align}
  which is a tighter bound for larger class-wise mean degrees. 
  
  From Eqs. \eqref{eq:expect_adj1_final} and \eqref{eq:expect_adj2_final}, we can rewrite the expected first and second powers of the normalised adjacency matrix using indicator functions to sum over all possible class combinations:

  \begin{align}\label{eq:first_moment_normalised_adj_sbm}
    \expect{\hat{A}_{ij}} \lessapprox \sum_{u,v=1}^{k} \delta_{\hat{y}_i u}\delta_{\hat{y}_j v} \left(n^{-1}D_{uu}^{-\frac{1}{2}}B_{uv} D_{vv}^{-\frac{1}{2}}\right),
  \end{align}
  and: 
  \begin{align}\label{eq:second_moment_normalised_adj_sbm}
    \expect{\left[\mat{\hat{A}}^2\right]_{ij}}\lessapprox \sum_{u,v=1}^{k}\delta_{\hat{y}_i u}\delta_{\hat{y}_j v}\left(D_{uu}^{-1} [\mat{B}]_{u,:} \mat{D}^{-1}\vect{\pi}\delta_{ij}+\frac{\left(D_{uu}D_{vv}\right)^{-\frac{1}{2}}\left[\mat{B}\left\{\mat{D}^{-1}-\mat{D}^{-2}\left(\eye{k}-e^{-\mat{D}}\right)\right\}\mat{\Pi B}\right]_{uv}}{n}(1-\delta_{ij})\right)
  \end{align}

  Recall the definition of $r$-order homophily in Eq. \eqref{eq:bottleneck_homophily_connection}; we can expand the Kronecker delta function as $\delta_{y_i y_j}=\sum_{w=1}^k\delta_{y_i w}\delta_{y_j w}$, and after taking the expectation we have the following form for the expected $r$-order homophily:

   \begin{align*}
    \expect{\rhomo{r}{\hat{\mat{A}}}} = \frac{1}{n} \sum_{i,j =1}^n \expect{\left[{\hat{\mat{A}}^r}\right]_{ij}}\delta_{y_i y_j} = \frac{1}{n} \sum_{w=1}^k\sum_{i,j =1}^n \expect{\left[{\hat{\mat{A}}^r}\right]_{ij}}\delta_{y_i w}\delta_{y_j w}.
  \end{align*}

  For $r=1$, using Eq. \eqref{eq:first_moment_normalised_adj_sbm} for the expected normalised adjacency matrix, the expected first-order homophily is:

  \begin{align*}
    & \expect{\rhomo{1}{\hat{\mat{A}}}} = \frac{1}{n} \sum_{w=1}^k\sum_{i,j =1}^n \expect{\left[{\hat{\mat{A}}}\right]_{ij}}\delta_{y_i w}\delta_{y_j w} \lessapprox \frac{1}{n}\sum_{w=1}^k\sum_{i,j =1}^n\sum_{u,v=1}^{k} \delta_{y_i w}\delta_{y_j w} \delta_{\hat{y}_i u}\delta_{\hat{y}_j v}\left(D_{uu}^{-\frac{1}{2}}B_{uv} D_{vv}^{-\frac{1}{2}}\right) \\
    & = \sum_{w=1}^k\sum_{u,v=1}^{k} \frac{1}{n}\left(\sum_{i=1}^n\delta_{y_i w} \delta_{\hat{y}_i u}\right)\left(D_{uu}^{-\frac{1}{2}}B_{uv} D_{vv}^{-\frac{1}{2}}\right) \frac{1}{n}\left(\sum_{j=1}^n\delta_{y_j w}\delta_{\hat{y}_j v}\right)\\
    & = \sum_{w=1}^k\sum_{u,v=1}^{k} \mat{C}_{wu}\left(D_{uu}^{-\frac{1}{2}}B_{uv} D_{vv}^{-\frac{1}{2}}\right) \mat{C}_{wv} = \tr{\mat{C}^T\mat{D}^{-\frac{1}{2}} \mat{B}\mat{D}^{-\frac{1}{2}} \mat{C}}
  \end{align*}
  
  where we introduced new Kronecker delta functions $\delta_{\hat{y}_i u}\delta_{\hat{y}_j v}$ to sum over all possible class combinations, and in the penultimate equality we used the definition of the confusion matrix $\mat{C}$ in Eq. \eqref{eq:confusion_mat_def}. Similarly for the second-order homophily, we have:
  
  \begin{align*}
   \expect{\rhomo{2}{\hat{\mat{A}}}} &= \frac{1}{n} \sum_{w=1}^k\sum_{i,j =1}^n \expect{\left[{\hat{\mat{A}}^2}\right]_{ij}}\delta_{y_i w}\delta_{y_j w} \\
  & \lessapprox \frac{1}{n}\sum_{w=1}^k\sum_{i,j =1}^n\sum_{u,v=1}^{k} \delta_{y_i w}\delta_{y_j w} \delta_{\hat{y}_i u}\delta_{\hat{y}_j v}\Big(D_{uu}^{-1} [\mat{B}]_{u,:} \mat{D}^{-1}\vect{\pi}\delta_{ij} \\
  & \quad + \frac{\left(D_{uu}D_{vv}\right)^{-\frac{1}{2}}\left[\mat{B}\left\{\mat{D}^{-1}-\mat{D}^{-2}\left(\eye{k}-e^{-\mat{D}}\right)\right\}\mat{\Pi B}\right]_{uv}}{n}(1-\delta_{ij})\Big).
  \end{align*}
  
  For the first term (with $\delta_{ij}$):
  
  \begin{align*}
  \frac{1}{n}\sum_{w=1}^k\sum_{i=1}^n\sum_{u=1}^{k} \delta_{y_i w}\delta_{\hat{y}_i u}D_{uu}^{-1} [\mat{B}]_{u,:} \mat{D}^{-1}\vect{\pi}
   = \sum_{w=1}^k\sum_{u=1}^{k} \frac{\mat{C}_{wu}}{n}D_{uu}^{-1} [\mat{B}]_{u,:} \mat{D}^{-1}\vect{\pi} 
   = \vect{\pi}^T\mat{D}^{-1}\mat{B}\mat{D}^{-1} \vect{\pi}.
  \end{align*}
  
  For the second term (with $1-\delta_{ij}$):
  
  \begin{align*}
  & \frac{1}{n^2}\sum_{w=1}^k\sum_{i\neq j}\sum_{u,v=1}^{k} \delta_{y_i w}\delta_{y_j w} \delta_{\hat{y}_i u}\delta_{\hat{y}_j v}\left(D_{uu}D_{vv}\right)^{-\frac{1}{2}}\left[\mat{B}\left\{\mat{D}^{-1}-\mat{D}^{-2}\left(\eye{k}-e^{-\mat{D}}\right)\right\}\mat{\Pi B}\right]_{uv} \\
  & = \sum_{w=1}^k\sum_{u,v=1}^{k} \mat{C}_{wu}\mat{C}_{wv}\left(D_{uu}D_{vv}\right)^{-\frac{1}{2}}\left[\mat{B}\left\{\mat{D}^{-1}-\mat{D}^{-2}\left(\eye{k}-e^{-\mat{D}}\right)\right\}\mat{\Pi B}\right]_{uv} \\
  & = \trace \left(\mat{C}^T \mat{D}^{-1}\mat{B}\left\{\mat{D}^{-1}-\mat{D}^{-2}\left(\eye{k}-e^{-\mat{D}}\right)\right\}\mat{\Pi B M}\right).
  \end{align*}
  
  Combining both terms gives us the final result:
  
  \begin{align*}
  \expect{\rhomo{2}{\hat{\mat{A}}_\mathrm{sym}}} &\lessapprox \vect{\pi}^T\mat{D}^{-1}\mat{B}\mat{D}^{-1} \vect{\pi} + \trace \left(\mat{C}^T \mat{D}^{-1}\mat{B}\left\{\mat{D}^{-1}-\mat{D}^{-2}\left(\eye{k}-e^{-\mat{D}}\right)\right\}\mat{\Pi B M}\right).
  \end{align*}
  
\end{proof}

\sbmhigherorderhomophily*

\begin{proof}
Let's start by expanding the expected $\ell$-order homophily using its definition:

\begin{align*}
  \expect{\rhomo{\ell}{\hat{\mat{A}}_\mathrm{sym}}} &= \frac{1}{n} \sum_{i,j \in V} \expect{\left[{\hat{\mat{A}}_\mathrm{sym}^{\ell}}\right]_{ij}}\delta_{y_i y_j}.
\end{align*}

Using the underreaching-oversquashing decomposition from Eq. \eqref{eq:underover}, we can write this as:

\begin{align*}
  \expect{\rhomo{\ell}{\hat{\mat{A}}_\mathrm{sym}}} &= \frac{1}{n} \sum_{i,j \in V} \sum_{r=1}^{\ell} \condexpect{\left[\hat{\mat{A}}_\mathrm{sym}^{\ell}\right]_{ij}}{\lambda_{ij} = r} \cdot \prob{\lambda_{ij} = r} \cdot \delta_{y_i y_j}.
\end{align*}

From Theorem \ref{theorem:underreaching_oversquashing}, for sparse graphs we can approximate:

\begin{align*}
  \condexpect{\left[\hat{\mat{A}}_\mathrm{sym}^{\ell}\right]_{ij}}{\lambda_{ij}=\ell}\prob{\lambda_{ij} = \ell}& \approx \left[ \left(\avg{\mat{D}}^{-\frac{1}{2}} \expect{\mat{A}}\avg{\mat{D}}^{-\frac{1}{2}} \right)^{\ell}\right]_{ij} = \frac{1}{n}[\mat{\Pi}^{-\frac{1}{2}}\hat{\mat{B}}^{\ell}\mat{\Pi}^{-\frac{1}{2}}]_{\hat{y}_i\hat{y}_j} + O\left(\frac{1}{\avg{d}}\right),
\end{align*}

and:
\begin{align}\label{eq:oversquashing_b}
  \condexpect{\left[\hat{\mat{A}}_\mathrm{sym}^{\ell}\right]_{ij}}{\lambda_{ij}=r}\prob{\lambda_{ij} = r}& \approx O\left(\frac{1}{\avg{d}}\right), \quad \mathrm{for } \quad r < \ell,
\end{align}

where $\hat{\mat{B}}:=\mat{D}^{-\frac{1}{2}}\mat{\Pi}^{\frac{1}{2}}\mat{B}\mat{\Pi}^{\frac{1}{2}}\mat{D}^{-\frac{1}{2}}$. Substituting these approximations into the defintion of $r$-order homophily in Eq. \eqref{eq:bottleneck_homophily_connection}, and taking the expectation gives:

\begin{align*}
  \expect{\rhomo{\ell}{\hat{\mat{A}}_\mathrm{sym}}} &\approx \frac{1}{n} \sum_{i,j \in V} \frac{1}{n}[\mat{\Pi}^{-\frac{1}{2}}\hat{\mat{B}}^{\ell}\mat{\Pi}^{-\frac{1}{2}}]_{\hat{y}_i\hat{y}_j} \delta_{y_i y_j} 
  = \frac{1}{n^2} \sum_{u,v,w=1}^k \sum_{i,j \in V} \delta_{y_i u}\delta_{\hat{y}_i v}[\mat{\Pi}^{-\frac{1}{2}}\hat{\mat{B}}^{\ell}\mat{\Pi}^{-\frac{1}{2}}]_{vw}\delta_{\hat{y}_j w}\delta_{y_j u} \\
  &= \sum_{u,v,w=1}^k C_{uv}[\mat{\Pi}^{-\frac{1}{2}}\hat{\mat{B}}^{\ell}\mat{\Pi}^{-\frac{1}{2}}]_{vw}\mat{C}_{uw} 
  = \trace \left(\mat{C}^T\mat{\Pi}^{-\frac{1}{2}}\hat{\mat{B}}^{\ell}\mat{\Pi}^{-\frac{1}{2}} \mat{C}\right),
\end{align*}

where in the second line we introduced new Kronecker delta functions $\delta_{\hat{y}_i u}\delta_{\hat{y}_j v}$ to sum over all possible class combinations, and in the third line we used the definition of the confusion matrix $\mat{C}$ in Eq. \eqref{eq:confusion_mat_def}. Similarly for total connectivity, we have the same expression, just summed over all pairs of nodes instead of pairs of nodes from the same class:

\begin{align*}
  \expect{\rconn{\ell}{\hat{\mat{A}}_\mathrm{sym}}} &\approx \frac{1}{n} \sum_{i,j \in V} \frac{1}{n}[\mat{\Pi}^{-\frac{1}{2}}\hat{\mat{B}}^{\ell}\mat{\Pi}^{-\frac{1}{2}}]_{\hat{y}_i\hat{y}_j}
  =\frac{1}{n^2} \sum_{u,v=1}^{k} n_u n_v [\mat{\Pi}^{-\frac{1}{2}}\hat{\mat{B}}^{\ell}\mat{\Pi}^{-\frac{1}{2}}]_{\hat{y}_i\hat{y}_j}\\
  &=\sum_{u,v=1}^{k} [\mat{\Pi}^{\frac{1}{2}}\hat{\mat{B}}^{\ell}\mat{\Pi}^{\frac{1}{2}}]_{\hat{y}_i\hat{y}_j} = \ones{k}^T\mat{\Pi}^{\frac{1}{2}}\hat{\mat{B}}^{\ell}\mat{\Pi}^{\frac{1}{2}}\ones{k}.
\end{align*}

Finally, for self-connectivity, we only need to consider the diagonal terms $\left[\hat{\mat{A}}_\mathrm{sym}^{\ell}\right]_{ii}$, for which $\lambda_{ii}=0<r$, and thus by Eq. \eqref{eq:oversquashing_b} has an expectation of $\expect{\left[\hat{\mat{A}}_\mathrm{sym}^{\ell}\right]_{ii}} \approx O\left(\frac{1}{\avg{d}^{\ell}}\right)$, and therefore:

\begin{align*}
  \expect{\rselfconn{\ell}{\hat{\mat{A}}_\mathrm{sym}}} & = \expect{\frac{1}{n}\sum_{i\in V}\left[\hat{\mat{A}}_\mathrm{sym}^{\ell}\right]_{ii}} = \frac{1}{n}\sum_{i\in V}\expect{\left[\hat{\mat{A}}_\mathrm{sym}^{\ell}\right]_{ii}} \approx O\left(\frac{1}{\avg{d}^{\ell}}\right).
\end{align*}
\end{proof}

\begin{restatable}{lemma}{plantedpartitionconnectivity}
\label{lemma:planted_partition_connectivity}
Consider a planted partition stochastic block model with $n$ nodes, $k>1$ equi-sized communities, and node class labels $\{y_i\}_{i\in[n]}$, where the expected adjacency matrix is given by: 
\begin{equation*}
\expect{A_{ij}} = \frac{B_{y_i y_j}}{n},
\end{equation*}
and the block probability matrix is defined as
\begin{equation}\label{eq:planted_B}
\mat{B} = k d \begin{bmatrix}
h & \dfrac{1-h}{k-1} & \cdots & \dfrac{1-h}{k-1}\\
\dfrac{1-h}{k-1} & h & \cdots & \dfrac{1-h}{k-1}\\
\vdots & \vdots & \ddots & \vdots\\
\dfrac{1-h}{k-1} & \dfrac{1-h}{k-1} & \cdots & h
\end{bmatrix},
\end{equation}
where $d>0$ denotes the expected mean degree and $0\le h\le 1$ is the expected edge homophily. Under the conditions of Theorem~\ref{theorem:sbm_higher_order_homophily} and for sufficiently large $d$, the expected $\ell$-order homophily, total connectivity, and self-connectivity computed using the symmetric normalised adjacency operator $\hat{\mat{A}}_\mathrm{sym}$ are approximated by:
\begin{align}
\expect{\rhomo{\ell}{\hat{\mat{A}}_\mathrm{sym}}} &\approx \frac{1}{k}\left[ 1 + (k-1)\left(\frac{k h - 1}{\,k-1}\right)^\ell \right] + O\!\left(\frac{1}{d}\right), \label{eq:homophily_approx}\\
\expect{\rconn{\ell}{\hat{\mat{A}}_\mathrm{sym}}} &\approx 1 + O\!\left(\frac{1}{d}\right), \label{eq:total_conn_approx}\\
\expect{\rselfconn{\ell}{\hat{\mat{A}}_\mathrm{sym}}} &\approx O\!\left(\frac{1}{d^{\ell}}\right). \label{eq:self_conn_approx}
\end{align}
\end{restatable}

\begin{proof}
We begin by examining the block structure of the planted partition model. Since the communities are equi-sized, the class probability matrix is given by $\mat{\Pi}=\frac{1}{k}\eye{k}$. Moreover, the expected degree of each node is $d$, so that the degree matrix is $\mat{D}=d\,\eye{k}$. Consequently, we have $\mat{\Pi}^{\frac{1}{2}}=\frac{1}{\sqrt{k}}\eye{k}$ and $\mat{D}^{-\frac{1}{2}}=\frac{1}{\sqrt{d}}\eye{k}.$
Under these conditions, the normalised block matrix defined in Theorem~\ref{theorem:sbm_higher_order_homophily} reduces to:
\begin{equation}\label{eq:hatB_def}
\hat{\mat{B}} = \mat{D}^{-\frac{1}{2}}\mat{\Pi}^{\frac{1}{2}}\mat{B}\mat{\Pi}^{\frac{1}{2}}\mat{D}^{-\frac{1}{2}} = \frac{1}{k d}\mat{B}.
\end{equation}
Substituting the expression for $\mat{B}$ from Eq.~\eqref{eq:planted_B} into Eq.~\eqref{eq:hatB_def}, we observe that the diagonal entries of $\hat{\mat{B}}$ are $\hat{\mat{B}}_{ii} = \frac{1}{k d}\cdot (k d\,h)= h,$
while for $i\neq j$ the off-diagonal entries are $\hat{\mat{B}}_{ij} = \frac{1}{k d}\cdot \Bigl(k d\,\frac{1-h}{k-1}\Bigr) = \frac{1-h}{k-1}.$
Thus, the matrix $\hat{\mat{B}}$ is a $k\times k$ matrix with constant diagonal entries equal to $h$ and constant off-diagonal entries equal to $\frac{1-h}{k-1}$. To determine the eigenvalues of $\hat{\mat{B}}$, we note that any $k\times k$ matrix with constant diagonal entry $a$ and constant off-diagonal entry $b$ has one eigenvalue: 
\begin{equation}\label{eq:eigenvalues}
\lambda_1 = a+(k-1)b,
\end{equation}
corresponding to eigenvector $\ones{k}$, and $k-1$ repeated eigenvalues given by:
\begin{equation}\label{eq:eigenvalues_repeated}
\lambda_2=\lambda_3=\cdots=\lambda_k = a - b,
\end{equation}
corresponding to eigenvectors $\mat{e}_1-\mat{e}_j$, for $j \in\{ 2,\dots,k\}$. Setting $a=h$ and $b=\frac{1-h}{k-1}$ in Eqs.~\eqref{eq:eigenvalues} and \eqref{eq:eigenvalues_repeated}: $\lambda_1 =1$
and $\lambda_j =  \frac{k h-1}{k-1}$ for $j\in\{2,\dots,k\}$.
We now derive the approximation for the expected $\ell$-order homophily. According to Theorem~\ref{theorem:sbm_higher_order_homophily}, the $\ell$-order homophily can be approximated as:
\begin{equation}\label{eq:homophily_trace}
\expect{\rhomo{\ell}{\hat{\mat{A}}_\mathrm{sym}}} \approx \trace\,\Bigl(\mat{\Pi}^{\frac{1}{2}}\hat{\mat{B}}^{\ell}\mat{\Pi}^{\frac{1}{2}}\Bigr)+ O\!\left(\frac{1}{d}\right)=\frac{1}{k}\trace\,\Bigl(\hat{\mat{B}}^{\ell}\Bigr)+ O\!\left(\frac{1}{d}\right).
\end{equation}
Because the trace of a matrix equals the sum of its eigenvalues, and the eigenvalues of matrix powers are the powers of the eigenvalues, we have $\trace\,\Bigl(\hat{\mat{B}}^{\ell}\Bigr) = 1 + (k-1)\left(\frac{k h-1}{k-1}\right)^\ell.$
Substituting in Eq. \eqref{eq:homophily_trace}, we get the approximation in Eq.~\eqref{eq:homophily_approx}.

We now consider the total connectivity. The expected total connectivity is given in Theorem~\ref{theorem:sbm_higher_order_homophily} as:
\begin{equation}\label{eq:total_conn_def}
\expect{\rconn{\ell}{\hat{\mat{A}}_\mathrm{sym}}} \approx \vect{1}_k^T\mat{\Pi}^{\frac{1}{2}}\hat{\mat{B}}^{\ell}\mat{\Pi}^{\frac{1}{2}}\vect{1}_k + O\!\left(\frac{1}{d}\right) = \frac{1}{k}\,\vect{1}_k^T\hat{\mat{B}}^{\ell}\vect{1}_k + O\!\left(\frac{1}{d}\right).
\end{equation}
Because the matrix $\hat{\mat{B}}$ has an eigenvector $\vect{1}_k$ with eigenvalue $1$, $
\hat{\mat{B}}^{\ell}\vect{1}_k = \vect{1}_k\implies
\vect{1}_k^T\hat{\mat{B}}^{\ell}\vect{1}_k = \vect{1}_k^T\vect{1}_k = k.$
Substituting in Eq.~\eqref{eq:total_conn_def} leads to the expression in Eq.~\eqref{eq:total_conn_approx}.
Lastly, the expected self-connectivity in Eq. \eqref{eq:self_conn_approx} is given immediately by setting $\avg{d}=d$ in the expression from Theorem~\ref{theorem:sbm_higher_order_homophily}.
\end{proof}

\optimalconnectivity*

\begin{proof}
As $\mat{D}:=\mathrm{Diag}\left(\mat{B}\vect{\pi}\right)$, the maximal eigenvalue of $\hat{\mat{B}}:=\mat{D}^{-\frac{1}{2}}\mat{\Pi}^{\frac{1}{2}}\mat{B}\mat{\Pi}^{\frac{1}{2}}\mat{D}^{-\frac{1}{2}}$ is 1. Subject to this constraint, we wish to maximise $\trace \left(\hat{\mat{C}}^T \hat{\mat{B}}^{\ell}\hat{\mat{C}}\right)$. By using the cyclic property of the trace, and expanding $\hat{\mat{B}}$ in the eigenbasis of $\hat{\mat{C}} \hat{\mat{C}}^T = \hat{\mat{Q}}\mat{\Lambda}\hat{\mat{Q}}^T$, we get:

\begin{align*}
  \trace \left(\hat{\mat{C}}^T \hat{\mat{B}}^{\ell}\hat{\mat{C}}\right)=\trace \left(\hat{\mat{C}} \hat{\mat{C}}^T \hat{\mat{B}}^{\ell}\right) = \trace \left(\hat{\mat{Q}}\mat{\Lambda}\hat{\mat{Q}}^T \hat{\mat{B}}^{\ell}\right) = \trace \left(\mat{\Lambda}\hat{\mat{Q}}^T \hat{\mat{B}}^{\ell}\hat{\mat{Q}}\right) = \sum_j{\lambda_j \left[\hat{\mat{Q}}^T \hat{\mat{B}}^{\ell}\hat{\mat{Q}}\right]_{jj}},
\end{align*}

where $\lambda_j$ are the eigenvalues of $\hat{\mat{C}} \hat{\mat{C}}^T$ which are all non-negative, $\hat{\mat{Q}}$ is the (unitary) matrix of eigenvectors of $\hat{\mat{C}}\hat{\mat{C}}^T$. Furthermore, as the eigenvalues of $\hat{\mat{B}}$ are all less than or equal to $1$, then $\left[\hat{\mat{Q}}^T \hat{\mat{B}}^{\ell}\hat{\mat{Q}}\right]_{jj} \leq \beta_\mathrm{max}^2=1$ as $\hat{\mat{Q}}$ is a unitary matrix. Therefore, we can bound the sum as:

\begin{align*}
  \sum_j{\lambda_j \left[\hat{\mat{Q}}^T \hat{\mat{B}}^{\ell}\hat{\mat{Q}}\right]_{jj}} \leq \sum_j\lambda_j = \trace \left(\hat{\mat{C}} \hat{\mat{C}}^T\right).
\end{align*}

Assuming $\hat{\mat{C}} \hat{\mat{C}}^T$ is full-rank, this maximum is achieved only if $\left[\hat{\mat{Q}}^T \hat{\mat{B}}^{\ell}\hat{\mat{Q}}\right]_{jj}=1$ for all $j\implies\trace \left(\hat{\mat{Q}}^T \hat{\mat{B}}^{\ell}\hat{\mat{Q}}\right)=k$. As the trace is the sum of the eigenvalues, which are all bounded by 1, $\trace \left(\hat{\mat{Q}}^T \hat{\mat{B}}^{\ell}\hat{\mat{Q}}\right)=k$ only if the eigenvalues of $\hat{\mat{Q}}^T \hat{\mat{B}}^{\ell}\hat{\mat{Q}}$---which coincide with those of $\hat{\mat{B}}^{\ell}$ since $\hat{\mat{Q}}$ is unitary---are all equal to $1$. Given that the eigenvalues of $\hat{\mat{B}}^{\ell}$ are all equal to 1, and $\hat{\mat{B}}^{\ell}$ is symmetric, it must be the identity matrix $\mat{I}_k$, as the spectral decomposition gives $\hat{\mat{B}}^{\ell} = \hat{\mat{P}} \mat{I}_k\hat{\mat{P}}^T = \mat{I}_k$ for a unitary matrix $\hat{\mat{P}}$. If $\hat{\mat{B}}^\ell = \mat{I}_k$ with odd $\ell$, as $\hat{\mat{B}}$ is real and symmetric, it must have real eigenvalues satisfying $\lambda^\ell = 1\implies\lambda=1\implies\hat{\mat{B}} = \mat{I}_k$.

For even $\ell$, $\lambda^\ell = 1\implies\lambda=1$ or $-1$, and so 
$\hat{\mat{B}}$ does not necessarily have to be $\mat{I}_k$. For even $\ell=2\ell'$, the solution is instead given by the finite set of non-negative, symmetric, orthonormal $k\times k$ matrices, equivalent to the set of symmetric permutation matrices: $\hat{\mat{B}}=\mat{P}_k$. To prove that the only solutions to $\hat{\mat{B}}^{2\ell'} =\mat{I}_k$ are the $k\times k$ symmetric permutation matrices, we begin by considering $\hat{\mat{B}}^{2\ell'} =\mat{I}_k$ as a system of equations. Firstly, looking at the off-diagonal entries gives:

\begin{align*}
 \left[\hat{\mat{B}}^{2\ell'}\right]_{ij} = \sum_{m_1,\dots,m_{2\ell'-1}=1}^k \hat{\mat{B}}_{im_1} \cdots\hat{\mat{B}}_{m_{2\ell'-1} j} = 0,
\end{align*}

and as the terms $\hat{\mat{B}}_{im_1} \cdots\hat{\mat{B}}_{m_{2\ell'-1} j}$ are all non-negative, they must all be equal to zero for all combinations of $m_1,\dots,m_{2\ell'-1}$.
Taking the particular alternating combination $m_1= p, m_2 = i, m_3=p,\dots,m_{2\ell'-1}=p$, for any choice of $p\in\{1,\dots,k\}$, and using the symmetry of $\hat{\mat{B}}$, we have that $(\hat{\mat{B}}_{ip})^{2\ell'-1}\hat{\mat{B}}_{jp} = 0.$
Therefore, for each column $p$, and for all $i,j\neq i$, at least one of $\hat{\mat{B}}_{ip}=0$ or $\hat{\mat{B}}_{jp}=0$ must be true. It follows that column $p$ must have at most one non-zero entry---if not, then there exist $i,j$ such that $\hat{\mat{B}}_{ip}>0$ and $\hat{\mat{B}}_{jp}>0$ leading to a contradiction. By symmetry, any given row must also have at most one non-zero entry.

Now consider the diagonal entries of $\hat{\mat{B}}^{2\ell'}$:

\begin{align*}
  \left[\hat{\mat{B}}^{2\ell'}\right]_{ii} = \sum_{m_1,\dots,m_{2\ell'-1}=1}^k \hat{\mat{B}}_{im_1} \cdots\hat{\mat{B}}_{m_{2\ell'-1} i}=1.
\end{align*}

As established above, in any given row of $\hat{\mat{B}}$ only a single column entry can be non-zero. Therefore, the above sum must contain only one non-zero term corresponding to a particular sequence of $m_1,\dots,m_{2\ell'-1}$ for which $\hat{\mat{B}}_{im_1} \cdots\hat{\mat{B}}_{m_{2\ell'-1} i}=1,$ and as each element of $\hat{\mat{B}}$ is bounded from above by $1$ each factor must be exactly 1, i.e. $\hat{\mat{B}}_{im_{1}} = 1,\dots,\hat{\mat{B}}_{m_{2\ell'-1}i} = 1$. In other words, $\hat{\mat{B}}$ can only be a matrix where each column (and by symmetry each row) has exactly one non-zero entry, equal to 1, which is the definition of permutation matrices that indeed satisfy $\hat{\mat{B}}^{2\ell'} =\mat{I}_k$. Therefore the general solution for $\hat{\mat{B}}$ is the set of symmetric permutation matrices $\mat{P}_k$.

To find $\mat{B}$ from a solution of $\hat{\mat{B}}$, we first note that $\vect{d}^{\frac{1}{2}}:=\mathrm{diag}\left(\mat{D}^{\frac{1}{2}}\right)$ is the eigenvector of $\mat{\Pi}^{-\frac{1}{2}}\hat{\mat{B}}\mat{\Pi}^{\frac{1}{2}}$ corresponding to the leading eigenvalue of 1, as:

\begin{align*}
  & \mat{\Pi}^{-\frac{1}{2}}\hat{\mat{B}}\mat{\Pi}^{\frac{1}{2}} \vect{d}^{\frac{1}{2}} =\mat{D}^{-\frac{1}{2}}\mat{B}\mat{\Pi}\mat{D}^{-\frac{1}{2}}\vect{d}^{\frac{1}{2}} = \mat{D}^{-\frac{1}{2}}\mat{B}\mat{\Pi}\ones{k}=\mat{D}^{-\frac{1}{2}}\mat{B}\vect{\pi}=\mat{D}^{-\frac{1}{2}}\vect{d}=\vect{d}^{\frac{1}{2}},
\end{align*}

where $\ones{k}$ is the length-$k$ vector of ones and $\vect{d}:=\diag{\mat{D}}=\mat{B}\mat{\pi}$. Here, we abuse the notation $\diag{M}$ to refer to the vector formed by the diagonal entries of matrix $M$. Since $\vect{d}^{\frac{1}{2}}$ has non-negative entries, by Perron--Frobenius Theorem it must be the leading eigenvector of $\mat{\Pi}^{-\frac{1}{2}}\hat{\mat{B}}\mat{\Pi}^{\frac{1}{2}}$. But given the solution $\hat{\mat{B}}=\mat{P}_k$, the leading eigenvector of $\mat{\Pi}^{-\frac{1}{2}}\hat{\mat{B}}\mat{\Pi}^{\frac{1}{2}}$ is $\mathrm{diag}\left(\mat{\Pi}^{-1/2}\right)$. Thus, we can choose $\vect{d}^{\frac{1}{2}}$---which controls the mean degree of each class---to be a scalar multiple of $\mathrm{diag}\left(\mat{\Pi}^{-1/2}\right)$, while choosing the scale to tune the overall mean degree $\langle d\rangle = \trace\left(\mat{\Pi \mat{D}}\right)$:

\begin{align*}
  \mat{D}^{\frac{1}{2}}=\sqrt{\frac{\langle d\rangle}{k}}\mat{\Pi}^{-1/2}.
\end{align*}

Finally, we can calculate the general optimal solution $\mat{B}$, when $\hat{\mat{C}} \hat{\mat{C}}^T$ is full rank, as:

\begin{align*}
  \mat{B}= \mat{\Pi}^{-1/2}\mat{D}^{\frac{1}{2}}\hat{\mat{B}}\mat{D}^{\frac{1}{2}} \mat{\Pi}^{-1/2} = \frac{\langle d\rangle}{k} \mat{\Pi}^{-1} {\mat{P}_k}\mat{\Pi}^{-1},
\end{align*}

for any choice of symmetric permutation matrix $\mat{P}_k$, and sufficiently large $\langle d\rangle$.
\end{proof}

\subsection*{Supplementary}
In this section we state some technical results and provide their proofs.

\begin{proposition}[Expectation of transformation of Poisson distributed random variable]\label{theorem:poisson_transformation}
  Let $X\sim\poisson{\lambda}$ be a Poisson distributed random variable with rate parameter $\lambda>0$, then:
  \begin{subequations}\label{eq:poisson_moments}
    \begin{align}\label{eq:poisson_moments_xplus1}
    &\expect{\frac{1}{X+1}} = \frac{1-e^{-\lambda}}{\lambda},\\\label{eq:poisson_moments_xplus2}
    &\expect{\frac{1}{X+2}} = \frac{\lambda-1+e^{-\lambda}}{\lambda^2},\\\label{eq:poisson_moments_sqrtxplus1}
    \sqrt{\frac{1}{\lambda}-\frac{1}{2\lambda^2}} < \ &\expect{\frac{1}{\sqrt{X+1}}} < \frac{1}{\sqrt{\lambda}}.
    \end{align}
  \end{subequations}
\end{proposition}
\begin{proof}
  Consider the LHS of Eq. \eqref{eq:poisson_moments_xplus1}:
  \begin{align*}
    \expect{\frac{1}{X+1}} = \sum_{k=0}^{\infty} \frac{\mathbb{P}(X=k)}{k+1} = \sum_{k=0}^{\infty} \frac{e^{-\lambda}\lambda^k}{(k+1)!}=\frac{e^{-\lambda}}{\lambda}\sum_{k=0}^{\infty} \frac{\lambda^{k+1}}{(k+1)!}=\frac{1-e^{-\lambda}}{\lambda},
  \end{align*}
  where we use the fact that $X$ is Poisson distributed and the series expansion of the exponential function.

  Similarly, consider the LHS of Eq. \eqref{eq:poisson_moments_xplus2}:
  \begin{align*}
    \expect{\frac{1}{X+2}} &= \sum_{k=0}^{\infty} \frac{\mathbb{P}(X=k)}{k+2} = \sum_{k=0}^{\infty} \frac{e^{-\lambda}\lambda^k(k+1)}{(k+2)!}=e^{-\lambda}\frac{d}{d\lambda}\sum_{k=0}^{\infty} \frac{\lambda^{k+1}}{(k+2)!}\\&=e^{-\lambda}\frac{d}{d\lambda}\frac{1}{\lambda}\sum_{k=0}^{\infty} \frac{\lambda^{k+2}}{(k+2)!}=e^{-\lambda}\frac{d}{d\lambda}\frac{e^\lambda-1-\lambda}{\lambda}=\frac{\lambda-1+e^{-\lambda}}{\lambda^2}.
  \end{align*}

  Next, consider the upper bound in Eq. \eqref{eq:poisson_moments_sqrtxplus1}. Due to concavity of the square root, Jensen's inequality yields:
  \begin{align*}
    \expect{\frac{1}{\sqrt{X+1}}}\le \sqrt{\expect{\frac{1}{X+1}}}=\sqrt{\frac{1-e^{-\lambda}}{\lambda}}<\frac{1}{\sqrt{\lambda}},
  \end{align*}
  for $\lambda>0$, and using Eq. \eqref{eq:poisson_moments_xplus1}.
  
  Finally, consider another random variable $Y$ independent and identically distributed (IID) as $X$, i.e. with the rate parameter $\lambda$. Then the AM--GM inequality for $X+1$ and $Y+1$ implies:
  \begin{align*}
    \sqrt{(X+1)(Y+1)}\le{\frac{X+Y+2}{2}} \implies \expect{\frac{1}{\sqrt{(X+1)(Y+1)}}}\ge 2\expect{\frac{1}{X+Y+2}}.
  \end{align*}
  Since $X$ and $Y$ are IID Poisson, $\inde{X+1}{Y+1}$ and $X+Y\sim\poisson{\lambda}$, which when used above alongside Eq. \eqref{eq:poisson_moments_xplus2} yields:
  \begin{align*}
    \expect{\frac{1}{\sqrt{X+1}}}\expect{\frac{1}{\sqrt{Y+1}}}\ge \frac{2\lambda-1+e^{-2\lambda}}{2\lambda^2}\implies \expect{\frac{1}{\sqrt{X+1}}}^2>\frac{1}{\lambda}-\frac{1}{2\lambda^2},
  \end{align*}
  for $\lambda>0$, which yields the lower bound in Eq. \eqref{eq:poisson_moments_sqrtxplus1}.
\end{proof}

\begin{proposition}[Expectation of inverse powers of shifted Poisson]\label{theorem:general_poisson_transformation}
Let $X \sim \text{Poisson}(\lambda)$ with $\lambda > 0$, and let $k$ be a positive integer. Then as $\lambda \to \infty$,
\begin{align*}
\expect{\frac{1}{(X+1)^k}} \le \frac{1}{\lambda^k} + O\left(\frac{1}{\lambda^{k+1}}\right).
\end{align*}
\end{proposition}

\begin{proof}
Let $f(x) = \frac{1}{(1+x)^k}$. Since $X \ge 0$, $f(X)$ is well-defined. We expand $f(X)$ around $\lambda$ using Taylor's theorem:
\begin{align*}
f(X) = f(\lambda) + f'(\lambda)(X - \lambda) + \frac{f''(c)}{2}(X - \lambda)^2,
\end{align*}
for some $c$ between $X$ and $\lambda$. Taking expectations:
\begin{align*}
\expect{f(X)} = f(\lambda) + \frac{1}{2} \expect{f''(c)(X - \lambda)^2}.
\end{align*}

Since $\expect{X} = \lambda$, the linear term vanishes. Now,
\begin{align*}
f(\lambda) = \frac{1}{(1+\lambda)^k} = \frac{1}{\lambda^k} \frac{1}{(1+\lambda^{-1})^{k}} \approx \frac{1}{\lambda^k} \left(1+O\left(\frac{1}{\lambda}\right)\right)^{k} = \frac{1}{\lambda^k} + O\left(\frac{1}{\lambda^{k+1}}\right),
\end{align*}

where the approximation comes from the geometric sum formula, which holds for large $\lambda$. Next, we bound the remainder term
\begin{align*}
R := \frac{1}{2} \expect{f''(c)(X - \lambda)^2}.
\end{align*}
Note that $f''(x) = k(k+1)(1+x)^{-k-2} > 0$ and decreasing in $x$. For large $\lambda$, $c \ge \lambda/2$ with high probability, so:
\begin{align*}
f''(c) \le k(k+1)(1+\lambda/2)^{-k-2}, \quad \text{and } \var{X} = \lambda.
\end{align*}
Therefore,
\begin{align*}
R \le \frac{k(k+1)}{2(1+\lambda/2)^{k+2}} \lambda = O\left(\frac{1}{\lambda^{k+1}}\right).
\end{align*}

To prove that any contribution from the event where $X < \lambda/2$ is negligible, we apply a Chernoff bound for the Poisson variable $X$. In particular, for any $a \le \lambda$, the Chernoff bound \cite{Chernoff1952} for a Poisson variable gives:
\begin{align*}
P(X \le a) \le \left(\frac{a}{\lambda}\right)^{-a} e^{a-\lambda}.
\end{align*}
Taking $a = \lambda/2$, we obtain
\begin{align*}
P\left(X \le \lambda/2\right) &\le \left(\frac{\lambda/2}{\lambda}\right)^{-\lambda/2} e^{\lambda/2-\lambda} = \left(\frac{2}{e}\right)^{\lambda/2}.
\end{align*}
Since $\left(\frac{2}{e}\right)^{\lambda/2}$ decays exponentially in $\lambda$, the probability of the event $X < \lambda/2$ is exponentially small. Thus, any contribution to $\expect{f(X)}$ coming from the region where $X < \lambda/2$ is negligible compared to the main asymptotic terms, and does not affect the overall order $O(1/\lambda^{k+1})$.

Putting everything together:
\begin{align*}
\expect{\frac{1}{(X+1)^k}} = f(\lambda) + R \le \frac{1}{\lambda^k} + O\left(\frac{1}{\lambda^{k+1}}\right).
\end{align*}
\end{proof}

\section{Appendix C: Hyperparameters}\label{sec:appendix_C}

\begin{table}[htbp]
\centering
\caption{Optimised hyperparameters for the base GCN across synthetic SBM datasets.}
\label{tab:hp_base_synth}
\begin{tabular}{lccccc}
\toprule
Dataset ($h$) & Hidden units & Depth & Dropout & Learning Rate & Weight Decay \\
\midrule
0.35 & 128 & 1 & 5.78e-01 & 1.18e-04 & 4.79e-03 \\
0.40 & 128 & 1 & 1.26e-01 & 9.18e-05 & 1.24e-03 \\
0.45 &  64 & 1 & 1.52e-01 & 1.81e-04 & 1.64e-02 \\
0.50 &  64 & 1 & 3.59e-02 & 1.15e-04 & 2.89e-03 \\
0.55 &  32 & 1 & 6.80e-01 & 9.19e-05 & 9.42e-03 \\
0.60 &  64 & 1 & 2.90e-01 & 2.04e-04 & 3.19e-02 \\
0.65 &  32 & 1 & 6.40e-02 & 4.16e-04 & 3.70e-02 \\
0.70 & 128 & 1 & 4.58e-02 & 4.93e-05 & 6.73e-03 \\
\bottomrule
\end{tabular}
\end{table}

\begin{table}[htbp]
\centering
\caption{Optimised hyperparameters for the base GCN across real-world datasets.}
\label{tab:hp_base_real}
\begin{tabular}{lccccc}
\toprule
Dataset    & Hidden Units& Depth & Dropout & Learning Rate      & Weight Decay      \\
\midrule
\textsc{Wisconsin}  & 16     & 1     & 5.26e-02 & 5.10e-02 & 4.33e-04 \\
\textsc{Texas}      & 128    & 1     & 2.97e-02 & 2.77e-03 & 1.01e-02 \\
\textsc{Cornell}    & 128    & 1     & 6.77e-01 & 8.76e-02 & 2.99e-04 \\
\textsc{Cora}       & 64     & 1     & 5.45e-01 & 1.10e-03 & 3.10e-04 \\
\textsc{Citeseer}   & 128    & 1     & 4.04e-01 & 8.56e-04 & 2.22e-04 \\
\textsc{Squirrel}   & 128    & 1     & 1.63e-01 & 9.95e-02 & 1.16e-05 \\
\textsc{Chameleon}  & 32     & 1     & 4.31e-01 & 7.66e-02 & 1.04e-05 \\
\textsc{Actor}      & 32     & 1     & 2.39e-01 & 9.01e-04 & 5.74e-04 \\
\bottomrule
\end{tabular}
\end{table}


\begin{table}[htbp]
\centering
\caption{Optimised BRIDGE hyperparameters across synthetic SBM datasets.}
\label{tab:hp_rewire_synth}
\setlength{\tabcolsep}{3pt}
\begin{tabular}{lcccccccc}
\toprule
Dataset ($h$) & Iter $M$ & Permutation (for $\mat{P}_k$) & $\avg{d}$ & Hidden & Depth & Dropout & Learning Rate      & Weight Decay      \\
\midrule
0.35 & 17 & (1, 2)   & 13.9 &  64 & 1 & 2.13e-01 & 8.64e-02 & 2.56e-06 \\
0.40 & 22 & (2, 1)   & 14.3 & 128 & 1 & 4.89e-01 & 9.81e-02 & 2.38e-06 \\
0.45 & 48 & (2, 1) & 23.0 & 128 & 1 & 5.77e-02 & 3.15e-02 & 6.24e-06 \\
0.50 & 46 & (2, 1) & 11.9 &  64 & 1 & 5.26e-01 & 8.30e-02 & 5.28e-06 \\
0.55 & 12 & (2, 1) & 12.5 &  32 & 1 & 6.28e-02 & 6.39e-02 & 2.10e-05 \\
0.60 & 28 & (2, 1) & 11.9 &  32 & 1 & 2.01e-01 & 7.51e-02 & 1.66e-06 \\
0.65 & 42 & (1, 2)   & 11.7 &  16 & 1 & 3.32e-01 & 9.74e-02 & 1.62e-05 \\
0.70 & 22 & (1, 2)   & 12.5 &  64 & 1 & 5.24e-01 & 8.43e-02 & 5.95e-06 \\
\bottomrule
\end{tabular}
\end{table}

\begin{table}[htbp]
\centering
\caption{Optimised BRIDGE hyperparameters across real-world datasets.}
\label{tab:hp_rewire_real}
\setlength{\tabcolsep}{3pt}
\begin{tabular}{lcccccccc}
\toprule
Dataset    & Iter $M$ & Permutation (for $\mat{P}_k$) & $\avg{d}$ & Hidden Units & Depth & Dropout & Learning Rate      & Weight Decay      \\
\midrule
\textsc{Wisconsin}  & 95   & (1, 4), (2, 5)         & 11.9          & 32     & 1     & 3.84e-01 & 3.11e-04 & 9.36e-05 \\
\textsc{Texas}      & 33   & (2, 3)         & 10.1          & 16     & 1     & 4.90e-01 & 1.04e-04 & 3.64e-06 \\
\textsc{Cornell}    & 81   & (3, 5)         & 10.8          & 32     & 1     & 3.38e-01 & 1.37e-04 & 5.57e-05 \\
\textsc{Cora}       & 89   & (1, 7), (2, 4), (5, 6)       & 51.3          & 32     & 1     & 4.98e-01 & 2.10e-03 & 2.34e-05 \\
\textsc{Citeseer}   & 46   & (1, 2), (3, 6)        & 35.6           & 128    & 3     & 5.61e-01 & 6.99e-04 & 1.05e-06 \\
\textsc{Squirrel}   & 91   & (1, 4), (3, 5)       & 65.9          & 32     & 2     & 5.39e-01 & 1.69e-03 & 1.34e-06 \\
\textsc{Chameleon}  & 26   & (2, 4), (3, 5)       & 14.0          & 64     & 3     & 4.83e-01 & 7.21e-05 & 1.25e-06 \\
\textsc{Actor}      & 12   & (1, 2), (3, 4)        & 10.2           & 64     & 1     & 3.96e-01 & 1.46e-04 & 2.08e-06 \\
\bottomrule
\end{tabular}
\end{table}


\begin{table}[htbp]
\centering
\caption{Optimised SDRF hyperparameters across synthetic SBM datasets.}
\label{tab:hp_sdrf_synth}
\begin{tabular}{lcccccccc}
\toprule
Dataset ($h$) & $\tau$ & Iterations & $C_{\mathrm{plus}}$ & Hidden & Depth & Dropout & Learning Rate & Weight Decay \\
\midrule
0.35 & 2.51e+02 & 190 & 7.90e+00 & 32 & 1 & 9.00e-02 & 2.66e-02 & 1.16e-05 \\
0.40 & 8.77e+01 & 332 & 2.75e+01 & 64 & 1 & 4.69e-01 & 2.78e-03 & 6.64e-04 \\
0.45 & 1.23e+02 & 157 & 4.04e+01 & 64 & 1 & 1.28e-01 & 1.15e-02 & 1.33e-06 \\
0.50 & 9.08e-01 & 332 & 1.80e+01 & 64 & 1 & 4.68e-02 & 3.21e-03 & 5.12e-05 \\
0.55 & 3.57e+02 & 95 & 4.48e+00 & 64 & 1 & 9.22e-02 & 7.86e-03 & 7.76e-04 \\
0.60 & 3.54e+02 & 176 & 2.78e+01 & 32 & 1 & 1.33e-01 & 9.16e-02 & 5.52e-04 \\
0.65 & 1.61e+02 & 20 & 4.15e+01 & 64 & 1 & 2.22e-02 & 7.52e-02 & 1.69e-05 \\
0.70 & 3.94e+02 & 83 & 2.07e+01 & 32 & 1 & 7.67e-02 & 1.17e-03 & 4.38e-06 \\
\bottomrule
\end{tabular}
\end{table}

\begin{table}[htbp]
\centering
\caption{Optimised SDRF hyperparameters across real-world datasets.}
\label{tab:hp_sdrf_real}
\begin{tabular}{lcccccccc}
\toprule
Dataset & Iter $M$ & $\tau$ & $C^{+}$ & Hidden Units & Depth & Dropout & Learning Rate & Weight Decay \\
\midrule
\textsc{Wisconsin} & 33 & 332.84 &  0.99 &  16 & 1 & 4.85e-01 & 9.60e-05 & 1.05e-04 \\
\textsc{Texas}     & 93 & 353.47 & 48.64 & 128 & 1 & 1.02e-01 & 1.61e-04 & 7.15e-04 \\
\textsc{Cornell}   & 81 &  46.13 & 41.49 & 128 & 1 & 3.83e-02 & 5.44e-04 & 9.04e-06 \\
\textsc{Cora}      & 12 & 184.28 & 18.72 & 128 & 1 & 4.03e-01 & 3.05e-04 & 1.00e-06 \\
\textsc{Citeseer}  & 58 & 417.26 & 39.78 &  64 & 1 & 1.71e-01 & 8.41e-03 & 4.06e-04 \\
\textsc{Squirrel}  & 45 &  94.36 & 41.56 &  16 & 1 & 6.05e-02 & 2.34e-02 & 5.12e-05 \\
\textsc{Chameleon} & 64 & 261.00 & 14.99 & 128 & 1 & 1.85e-03 & 5.96e-02 & 1.32e-06 \\
\textsc{Actor}     & 27 & 446.23 & 31.21 & 128 & 1 & 3.42e-01 & 2.89e-04 & 7.31e-04 \\
\textsc{Pubmed}    & 84 & 268.48 & 32.83 &  64 & 1 & 4.58e-01 & 6.83e-03 & 5.35e-05 \\
\bottomrule
\end{tabular}
\end{table}


\begin{table}[htbp]
\centering
\caption{Optimised DIGL hyperparameters across synthetic SBM datasets.}
\label{tab:hp_digl_synth}
\begin{tabular}{lcccccccc}
\toprule
Dataset ($h$) & $\alpha$ & $\epsilon$ & Hidden & Depth & Dropout & Learning Rate & Weight Decay \\
\midrule
0.35 & 1.58e-01 & 8.96e-03 & 128 & 1 & 6.35e-01 & 5.56e-04 & 4.82e-05 \\
0.40 & 1.77e-01 & 8.93e-03 & 128 & 1 & 2.22e-02 & 3.35e-05 & 1.65e-06 \\
0.45 & 1.95e-01 & 1.54e-02 & 128 & 1 & 2.38e-01 & 3.43e-05 & 1.60e-05 \\
0.50 & 9.38e-02 & 9.00e-03 & 128 & 1 & 4.20e-01 & 1.27e-04 & 2.04e-05 \\
0.55 & 8.11e-02 & 4.81e-03 & 128 & 1 & 6.60e-01 & 4.47e-03 & 5.70e-05 \\
0.60 & 8.86e-02 & 8.93e-03 & 128 & 1 & 1.37e-01 & 5.38e-02 & 2.13e-06 \\
0.65 & 2.25e-01 & 1.38e-02 & 128 & 1 & 4.63e-01 & 6.38e-05 & 1.98e-04 \\
0.70 & 1.58e-01 & 9.71e-03 & 128 & 1 & 3.57e-01 & 2.83e-03 & 1.27e-04 \\
\bottomrule
\end{tabular}
\end{table}

\begin{table}[htbp]
\centering
\caption{Optimised DIGL hyperparameters across real-world datasets.}
\label{tab:hp_digl_real}
\begin{tabular}{lcccccccc}
\toprule
Dataset & $\alpha$ & $\epsilon$ & Hidden Units & Depth & Dropout & Learning Rate & Weight Decay \\
\midrule
\textsc{Wisconsin} & 1.16e-01 & 3.08e-04 & 128 & 2 & 6.65e-01 & 1.48e-05 & 1.46e-06 \\
\textsc{Texas}     & 2.30e-01 & 5.88e-04 &  32 & 3 & 5.85e-01 & 1.17e-05 & 1.19e-05 \\
\textsc{Cornell}   & 2.00e-01 & 1.13e-05 &  16 & 2 & 6.14e-01 & 3.18e-02 & 6.67e-04 \\
\textsc{Cora}      & 2.51e-01 & 6.59e-04 & 128 & 1 & 2.62e-02 & 2.63e-03 & 1.54e-06 \\
\textsc{Citeseer}  & 2.65e-01 & 2.70e-04 &  32 & 1 & 9.34e-02 & 1.28e-03 & 1.87e-04 \\
\textsc{Squirrel}  & 2.30e-01 & 3.16e-04 & 128 & 1 & 6.31e-01 & 7.40e-02 & 3.53e-06 \\
\textsc{Chameleon} & 2.65e-01 & 8.86e-04 &  16 & 1 & 6.60e-01 & 5.42e-02 & 1.14e-06 \\
\textsc{Actor}     & 5.59e-02 & 3.21e-04 &  64 & 1 & 6.99e-01 & 7.21e-05 & 4.22e-05 \\
\textsc{Pubmed}    & 2.78e-01 & 2.49e-04 &  32 & 1 & 6.74e-01 & 4.68e-02 & 1.41e-06 \\
\bottomrule
\end{tabular}
\end{table}

\end{document}